%% file: main.tex
\documentclass[sigconf]{acmart}
\AtBeginDocument{%
  }

\copyrightyear{2025}
\acmYear{2025}
\setcopyright{cc}
\setcctype{by}
\acmConference[KDD '25]{Proceedings of the 31st ACM SIGKDD Conference on Knowledge Discovery and Data Mining V.2}{August 3--7, 2025}{Toronto, ON, Canada}
\acmBooktitle{Proceedings of the 31st ACM SIGKDD Conference on Knowledge Discovery and Data Mining V.2 (KDD '25), August 3--7, 2025, Toronto, ON, Canada}
\acmDOI{10.1145/3711896.3737144}
\acmISBN{979-8-4007-1454-2/2025/08}

\input{math_commands.tex}

\usepackage{microtype}
\usepackage{graphicx}
\usepackage{subcaption}
\usepackage{booktabs} 
\usepackage{enumitem}
\usepackage{color}

\usepackage{hyperref}

\usepackage{algorithm}
\usepackage[noend]{algpseudocode}

\usepackage{xspace}
\usepackage{amsmath}

\usepackage{amssymb}
\usepackage{bbm} 
\usepackage{mathtools}
\usepackage{amsthm}
\usepackage{colortbl}
\usepackage[capitalize,noabbrev]{cleveref}
\usepackage{multirow}
\theoremstyle{plain}
\newtheorem{theorem}{Theorem}[section]
\newtheorem{prop}[theorem]{Proposition}
\newtheorem{lemma}[theorem]{Lemma}
\newtheorem{claim}[theorem]{Claim}

\theoremstyle{definition}
\newtheorem{definition}[theorem]{Definition}

\theoremstyle{remark}
\newtheorem{remark}[theorem]{Remark}

\newcommand{\yy}[1]{\textcolor{black}{#1}}

\newcommand{\makecolor}{\cellcolor{orange!20}}

\newcommand{\reb}[1]{\textcolor{black}{#1}}
\newcommand{\sia}{\texttt{+SIA}\xspace}
\newcommand{\ciga}{\texttt{+CIGA}\xspace}
\newcommand{\ssl}{\texttt{+SSL}\xspace}
\newcommand{\ac}{\texttt{+AugCyc}\xspace}

\newcommand{\na}{\texttt{+NA}\xspace}
\newcommand{\ssr}{\texttt{+SSR}\xspace}
\newcommand{\smp}{\texttt{SMP}\xspace}
\newcommand{\rpgin}{\texttt{+RPGNN}\xspace}

\usepackage[textsize=tiny]{todonotes}
\setitemize{noitemsep,topsep=0pt,parsep=0pt,partopsep=0pt}

\begin{document}

\title{Tackling Size Generalization of Graph Neural Networks on Biological Data from a Spectral Perspective}

\author{Gaotang Li}
\authornote{Work started at UM and was completed during an internship at Dartmouth College.}
\affiliation{
    \institution{University of Illinois Urbana-Champaign}
    \city{Urbana}
    \state{IL}
    \country{USA}
}
\email{gaotang3@illinois.edu}

\author{ Danai Koutra}
\affiliation{%
  \institution{University of Michigan, Ann Arbor}
  \city{Ann Arbor}
  \state{MI}
  \country{USA}
}
\email{dkoutra@umich.edu}

\author{ Yujun Yan}
\affiliation{%
  \institution{Dartmouth College}
  \city{Hanover}
  \state{NH}
  \country{USA}
}
\email{yujun.yan@dartmouth.edu}

\begin{abstract}

We address the key challenge of size-induced distribution shifts in graph neural networks (GNNs) and their impact on the generalization of GNNs to larger graphs. 
Existing literature operates under diverse assumptions about distribution shifts, resulting in varying conclusions about the generalizability of GNNs.
In contrast to prior work, we adopt a data-driven approach to identify and characterize the types of size-induced distribution shifts and explore their impact on GNN performance from a spectral standpoint, a perspective that has been largely underexplored. 

{Leveraging the significant variance in graph sizes in real biological datasets, we analyze biological graphs and find that spectral differences—driven by subgraph patterns (e.g., average cycle length)—strongly correlate with GNN performance on larger, unseen graphs. Based on these insights, we propose three model-agnostic strategies to enhance GNNs' awareness of critical subgraph patterns, identifying size-intensive attention as the most effective approach. Extensive experiments with six GNN architectures and seven model-agnostic strategies across five datasets show that our size-intensive attention strategy significantly improves graph classification on test graphs 2 to 10 times larger than the training graphs, boosting F1 scores by up to 8\% over strong baselines.
}
\end{abstract}




\maketitle

\vspace{-0.1cm}
\section{Introduction}
\input{sections/introduction}
\section{Notations and Preliminaries}
\input{sections/preliminary}
\section{Spectral Analysis {of Size-induced Distribution Shifts}}
\input{sections/spectral_analysis}

\section{Methodology: Strategies for Enhancing Cycle Awareness in GNNs}
\input{sections/method}

\section{Experiments}
\input{sections/experiment}

\input{sections/synthetic_exp}

\color{black}

\section{Related Work} 
\input{sections/related_work}
\section{Conclusion}
\input{sections/conclusion}

\section*{Acknowledgements}
{\small We thank the anonymous reviewers for their constructive feedback. This material is based upon work supported by the National Science Foundation under IIS 2212143, and CAREER Grant No.~IIS 1845491. Any opinions, findings, and conclusions or recommendations expressed in this material are those of the authors and do not necessarily reflect the views of the National Science Foundation or other funding parties.
}

\bibliographystyle{ACM-Reference-Format}
\balance
\bibliography{ref}

\appendix 
\input{sections/appen_proof}
\input{sections/appen_datasets}
\input{sections/appen_data_preprocessing}
\input{sections/appen_training_details} 
\input{sections/appen_additional_experiment_table}
\input{sections/appen_attention}

\input{sections/appen_cycle_align_break}

\input{sections/appen_additional_plots}
\input{sections/appen_additional_scaled_eigenvector_plot}
\input{sections/appen_additional_plot_degree}
\input{sections/appen_d_pattern}

\end{document}

%% file: math_commands.tex
\usepackage{amsmath,amsfonts,bm}

\def\eqref#1{equation~\ref{#1}}

\def\1{\bm{1}}

\def\rvk{{\mathbf{k}}}

\def\rvw{{\mathbf{w}}}
\def\rvx{{\mathbf{x}}}
\def\rvy{{\mathbf{y}}}

\def\rmA{{\mathbf{A}}}

\def\rmD{{\mathbf{D}}}

\def\rmI{{\mathbf{I}}}

\def\rmL{{\mathbf{L}}}

\def\rmT{{\mathbf{T}}}
\def\rmU{{\mathbf{U}}}

\def\rmW{{\mathbf{W}}}
\def\rmX{{\mathbf{X}}}

\def\rmlambda{{\mathbf{\Lambda}}}

\DeclareMathAlphabet{\mathsfit}{\encodingdefault}{\sfdefault}{m}{sl}
\SetMathAlphabet{\mathsfit}{bold}{\encodingdefault}{\sfdefault}{bx}{n}

\def\gE{{\mathcal{E}}}

\def\gG{{\mathcal{G}}}

\def\gL{{\mathcal{L}}}

\def\gN{{\mathcal{N}}}

\def\gV{{\mathcal{V}}}

\def\sR{{\mathbb{R}}}

\newcommand{\R}{\mathbb{R}}

\newcommand{\xbold}{\mathbf{x}}
\newcommand{\Xbold}{\mathbf{X}}
\newcommand{\Cbold}{\mathbf{C}}

\DeclareMathOperator*{\argmin}{arg\,min}

%% file: sections/introduction.tex
\label{sec:intro}
{
Graph neural networks (GNNs) \cite{kipf2016semi,velickovic2018graph, graphsage, zhang2018end, yan2019groupinn, lee2018graph,LiDZKY23_sparsification, liuexploring} have gained popularity in graph classification due to their strong performance. While most GNNs handle graphs of varying sizes, their ability to generalize to larger, unseen graphs (size generalizability) remains underexplored. This is crucial in various domains.
For instance, in graph algorithmic reasoning~\cite{yan2020neural,velivckovicneural}, GNNs need to learn complex algorithms from small examples and generalize that reasoning to larger graphs, as obtaining exact solutions for larger graphs is challenging. In biology, datasets exhibit a wide range of graph sizes, spanning from small molecules to large compounds. 
}

{Most existing literature on the size generalization of GNNs relies on varying assumptions, leading to different conclusions. For example, a theoretical study \cite{levie2021transferability} assumes that graphs of different sizes discretize the same underlying space (e.g., grid graphs from physical simulations), claiming that spectral GNNs exhibit robust transferability across different graph sizes. Several other studies \cite{levie2021transferability, xu2021neural, sanchez2020learning} also support strong generalizability of GNNs across varying sizes. Conversely, a recent study \cite{yehudai2021local} assumes that local graph structure changes with graph size (e.g., social graphs), arguing that GNN performance declines for larger graphs due to variations in degree patterns in multi-hop neighborhoods. Other works \cite{bevilacqua2021size, buffellisizeshiftreg, huang2024enhancing} also align with this conclusion.}

{Unlike previous studies that introduce new assumptions, we adopt a \textit{data-driven} approach to characterize the types of size-induced distribution shifts and assess their impact on GNNs.}
{Harnessing the diversity of publicly available biological datasets and the vast range of graph sizes they contain, 
we explore the size generalizability of GNNs within the task of biological graph classification.} 
In our analysis, we take a spectral perspective, which offers a fresh view on understanding GNNs that has been underexplored in terms of size generalization. Specifically, our spectral analysis identifies connections between 
the spectral
differences induced by varying graph sizes and the differences in subgraph patterns, particularly cycles. We empirically find that
breaking cycles in {all} graphs amplifies the spectral difference between smaller and larger graphs, whereas extending cycle lengths in smaller graphs to align with those in larger graphs reduces this difference. Furthermore, we observe that conventional GNNs struggle to generalize effectively without explicit cycle information, leading to performance degradation on larger graphs. While prior research has identified GNNs' limitations in counting and recognizing cycles~\cite{chen2020can} within in-distribution generalization, {we provide a thorough and systematic exploration} of their impact in \textit{out-of-distribution} scenarios, specifically focusing on size generalizability. 

{To verify our insight that incorporating cycle information enhances GNN size generalization, we introduce three strategies that can be applied to \textit{any} GNN model. Our first strategy, Size-Insensitive Attention (\texttt{SIA}), encodes cycle information into feature representations and integrates it within an attention mechanism to guide learning. Our second strategy employs a self-supervised auxiliary task to enrich node representations with cycle-related information, while our third strategy leverages augmentation to mitigate cycle discrepancies between small and large graphs. Empirical analysis on diverse biological datasets confirms that incorporating cycle information significantly improves GNN generalization to larger graphs, with \texttt{SIA} emerging as the most effective approach, achieving substantial performance gains over state-of-the-art models.}
In sum, our paper makes the following contributions:
\begin{itemize}
    \item \textbf{New Observations.} 
    We {characterize the types of} size-induced distribution shifts {in real biological datasets},  
    and provide insights into patterns that can help GNNs become size-agnostic and generalize better to larger graphs (\S~\ref{sec: analysis}).
    \item \textbf{Spectral Analysis.} 
    {In contrast to prior work, we leverage spectral analysis to deepen our understanding of the size generalizability of GNNs (\S~\ref{sec: analysis})} {and provide a complementary perspective to the typical exploration of the problem}.   
    \item \textbf{Effective Model-agnostic Strategies.} 
    {
{We propose three strategies to make GNNs aware of size-related subgraph patterns (e.g., average cycle length), including a novel size-insensitive attention mechanism (\texttt{SIA}, \S~\ref{sec:method}). Experiments show that \texttt{SIA} improves GNN generalizability on test graphs 2-10$\times$ larger than training graphs, boosting $\text{F}_1$ scores by up to 8\% over strong baselines (\S~\ref{sec:exp}).}
}
\end{itemize}

%% file: sections/preliminary.tex
\label{sec:prelim}
{In this section, we begin by introducing the notations and definitions used throughout the paper. Next, we provide 
an introduction to the fundamentals of GNNs.} 

\vspace{-0.2cm}
\subsection{Notations \& Definitions}

Let $\gG$($\gV$, $\gE)$ be an undirected and unweighted graph with $N$ nodes, where $\gV$ denotes the node set, and $\gE$ denotes the edge set. {The neighborhood of a node $v_i$ is defined as the set of nodes that connect to $v_i$: $\gN_i=\{v_j|(v_j, v_i) \in \gE\}.$
The graph is represented by its adjacency matrix $\rmA\in\mathbb{R}^{N\times N}$, and it  has
a degree matrix $\rmD$, where the $i$th diagonal element $d_i$ corresponds to the degree of node $v_i$. } {We use $\mathbf{I}$ to denote the $N$-dimensional identity matrix.} {We use the term ``spectrum'' to denote the set of eigenvalues of a matrix.}

\vspace{0.1cm}
\noindent \textbf{Cycle basis.} 
{An important concept we use to study cycles is \textit{cycle basis}~\citep{syslo1979cycle}. A cycle basis is defined as the smallest set of cycles where any cycle in the graph can be expressed as a sum of cycles from this basis, similar to the concept of a basis in vector spaces. Here, the summation of cycles is defined as “exclusive OR” of the edges. We represent the cycle basis for a graph as $\mathcal{C}$ and refer to the $j$th cycle in this cycle basis as $\mathcal{C}_j$. The cycle basis can be found using the algorithm CACM 491~\cite{paton1969algorithm}.}

\vspace{-0.2cm}
\subsection{Graph Learning Task}
\label{sec:task}
{Size generalizability naturally comes up in the graph classification task}, where each node $v_i$ is associated with a feature vector $\rvx_i^{(0)}$, and the feature matrix $\rmX^{(0)}$ is constructed by arranging the node feature vectors as rows. When using a GNN for the graph classification task, we further denote the node representation matrix at the $l$-th layer as $\rmX^{(l)}$, and the representation of node $v_i$ as $\rvx_i^{(l)}$.

\vspace{0.1cm}
\noindent \textbf{Supervised Graph Classification.} Each graph $\gG_i$ is associated with a label {$y_{\gG_i}$} sampled from a label set $\hat{\gL}$. Given a subset of labeled graphs (from a label set $\hat{\gL}$), the goal is to learn a mapping $f^{\gG}:(\rmA, \rmX^{(0)})_i \mapsto$ {$y_{\gG_i}$} between each graph $\gG_i$ and its ground truth label {$y_{\gG_i}$} $ \in \hat{\gL}$. The graph classification loss is given by $L$ = $\frac{1}{|\gG_{\text{train}}|}\sum_{\gG_i \in \gG_{\text{train}}}$\texttt{CrossEntropy} ({$\rvx_{\gG_i}, y_{\gG_i}$}), where $\gG_{\text{train}}$ is the training graph set and $\rvx^{\gG_i}$ is the representation of graph $\gG_i$.

\vspace{0.1cm}
\noindent {\textbf{Evaluation of Size Generalizability.} 
Following prior work~\cite{buffellisizeshiftreg, yehudai2021local}, {given a dataset,} we evaluate the size generalizability of GNNs by testing their classification performance on graphs whose sizes are  larger than those in the train set.} 
{We provide more empirical setup details in \Cref{subsec:setup}.} 

\subsection{Graph Neural Networks} 
\label{subsec:gnns}
\vspace{-0.1cm}

GNNs can be designed from either the spatial perspective or the spectral perspective. Despite the difference in the design perspectives, a recent work~\cite{balcilar2021analyzing} has shown that spectral GNNs and spatial GNNs are related and that spectral analysis of GNNs' behavior can provide a complementary point of view to understand GNNs. 

{Most spatial GNNs~\cite{kipf2016semi, xu2018powerful, velickovic2018graph, graphsage} use the message passing framework~\cite{gilmer2017neural}, which consists of three steps: neighborhood propagation, message combination and global pooling.}
Spectral GNNs~\cite{bo2021beyond, defferrard2016convolutional, levie2018cayleynets} utilize the spectral properties of a propagation matrix $\rmT$ to perform the graph classification. The matrix $\rmT$ is usually a function of the adjacency matrix $\rmA$, such as the normalized adjacency matrix $\rmT=(\rmD+\rmI)^{-1/2}(\rmA+\rmI)(\rmD+\rmI)^{-1/2}$, or the normalized graph Laplacian matrix $\hat{\rmL}$. Since we consider an undirected graph with a real and symmetric adjacency matrix, the matrix $\rmT$ is also real and symmetric. Then, the eigendecomposition on the propagation matrix {is defined as} $\rmT=\rmU\rmlambda\rmU^{T}$, where $\rmU$ is an orthogonal matrix whose columns $\rmU_i$ are orthonormal and are the eigenvectors of $\rmT$, and $\rmlambda$ is a matrix whose diagonal elements are the eigenvalues of $\rmT$, {sorted from large to small by their absolute values}. The set of eigenvectors $\{\rmU_i\}$ form the orthonormal basis of $\sR^{n}$. The goal of a spectral GNN is to learn a proper spectral filter: $f(\rmlambda)=c_0\rmI+c_1\rmlambda+\cdots+c_i\rmlambda^{i}+\cdots$, where $c_i$ are the learnable coefficients. The convolution at each layer 
is equivalent to: $\rmX^{(l+1)}=\sigma(\rmU f(\rmlambda)\rmU^{T}\rmX^{(l)}\rmW^{(l)})$, where $\rmW^{(l)}$ is a learnable weight matrix, and $\sigma{(\cdot)}$ is a nonlinear function (e.g., \texttt{ReLU}). {The graph representation is obtained from the node representations at the last convolution layer: $\rvx^{\gG}$ = \texttt{Pooling} ($\{\rvx_i^\text{(Last)}\}$), where the \texttt{Pooling} function is performed on the set of all the node representations, and it can be \texttt{Global\_mean} or \texttt{Global\_max} or other more complex pooling functions~\cite{ying2018hierarchical, knyazev2019understanding}.}

%% file: sections/spectral_analysis.tex
\label{sec: analysis}

\begin{figure*}
\begin{minipage}{\textwidth}
  \begin{minipage}[b]{0.6\textwidth}
    \centering
    \includegraphics[width=.98\textwidth]{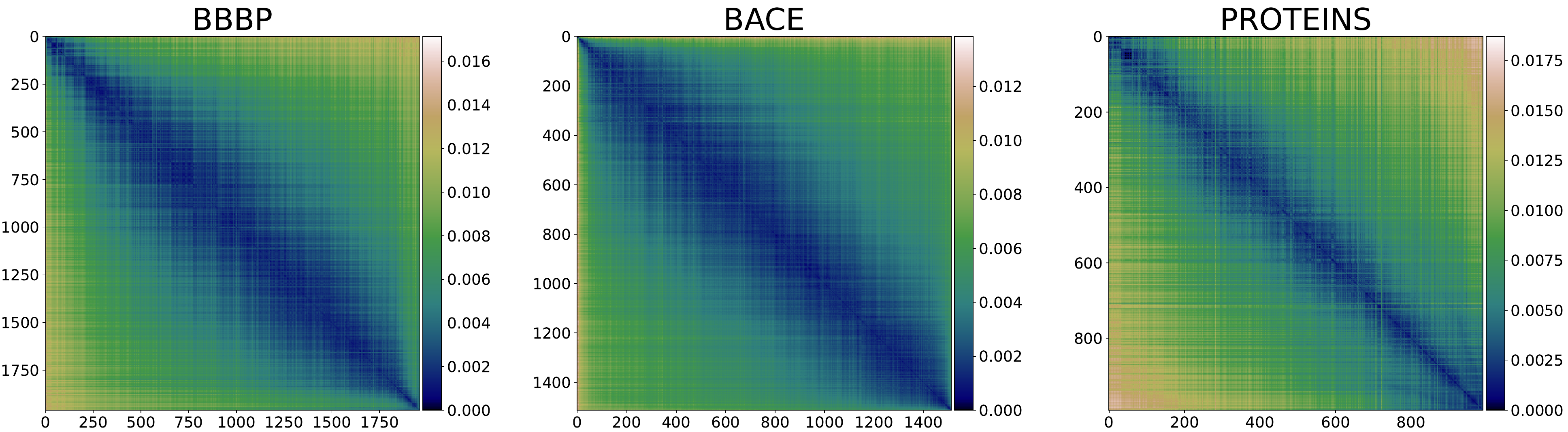}
    \vspace{-0.2cm}
    \captionof{figure}{{[Qualitative Analysis]} Pairwise graph distance of eigenvalue distributions: Graphs are sorted from small to large, and the (i,j)-th pixel in the plot represents the Wasserstein distance between the i-th and j-th graphs. Dark blue represents a small distance (high similarity), while light red represents a large distance (low similarity). We observe that \textbf{eigenvalue distributions show a strong correlation with the graph size}.}
    \label{fig:eigenvalue_dsitribution}
  \end{minipage}
  \hfill
  \begin{minipage}[b]{0.38\textwidth}
    \centering
    \resizebox{\linewidth}{!}{
    \begin{tabular}{@{}lccc@{}}
    \toprule
    & \textbf{Different Size} & \textbf{Similar Size} & \textbf{Relative Diff} \\ \midrule
    \textbf{BBBP}        & 0.00566       & 0.00184       & 208\%       \\
    \textbf{BACE}        & 0.00411       & 0.00149       & 177\%       \\
    \textbf{PROTEINS}    & 0.00765       & 0.00261       & 193\%       \\
    \textbf{NCI1}        & 0.00566       & 0.00215       & 164\%       \\
    \textbf{NCI109}      & 0.00563       & 0.00215       & 162\%       \\ \bottomrule
    \end{tabular}
    }
    \captionof{table}{[Quantitative Analysis] Average Wasserstein distance between graphs of `similar sizes' and graphs of `different sizes' based on  \textbf{eigenvalue} distributions, respectively. The relative difference is computed by the difference of the Wasserstein distance \yy{(col. 1-col. 2)} normalized by the Wasserstein distance of similar graphs.}
    \label{tab:eigen_dist}
    \end{minipage}
  \end{minipage}
\end{figure*}
{In this section, we present our data-driven spectral analysis and findings on subgraph patterns that explain distribution shifts between small and large graphs. 
Specifically, we examine the differences in eigenvalue distributions of the normalized adjacency matrices $\rmT$, which are typically used as the propagation matrices in spectral GNNs (\S~\ref{subsec:gnns}).  First, in \S~\ref{subsec: eigen_bio}, we empirically discover that the eigenvalue distribution of the normalized adjacency matrix depends on the graph size. Then, in  \S~\ref{sec:findings}, we further explore the subgraph patterns responsible for the spectral discrepancies between small and large graphs, revealing two key findings:}

\begin{itemize}[noitemsep,nolistsep,topsep=0pt]
\item Breaking cycles in graphs amplifies the spectral difference between smaller and larger graphs.
\item Extending cycle lengths in smaller graphs to match larger ones reduces the spectral difference.
\end{itemize}
{In \S~\ref{subsec: GNN_spec_gen}, we provide theoretical insights into the relationship between eigenvalue distribution and graph size, elucidating how this dependency may limit the generalization capabilities of GNNs across various graph sizes.}

\subsection{\scalebox{.99}[1.0]{Size-based Spectral Differences in Real Data}}
\label{subsec: eigen_bio}
{{In this subsection, we} investigate how the eigenvalue distribution of the normalized adjacency matrix $\rmT$ varies with the graph size in real-world data. {This property is important as we will show later in \S~\ref{subsec: GNN_spec_gen} that} the spectral discrepancy between small and large graphs affects the size generalizability of GNNs.}

\vspace{0.1cm}
\noindent {\textbf{Datasets.} In our data-driven analysis, we use five pre-processed biological datasets (BBBP, BACE, NCI1, NCI109, and PROTEINS) from the Open Graph Benchmark~\citep{hu2020open} and TuDataset~\citep{Morris+2020}.  
We give details about them  in~\cref{appen:dataset_details}. }

\begin{figure*}[t]
    \centering
    \includegraphics[width=0.75\textwidth]{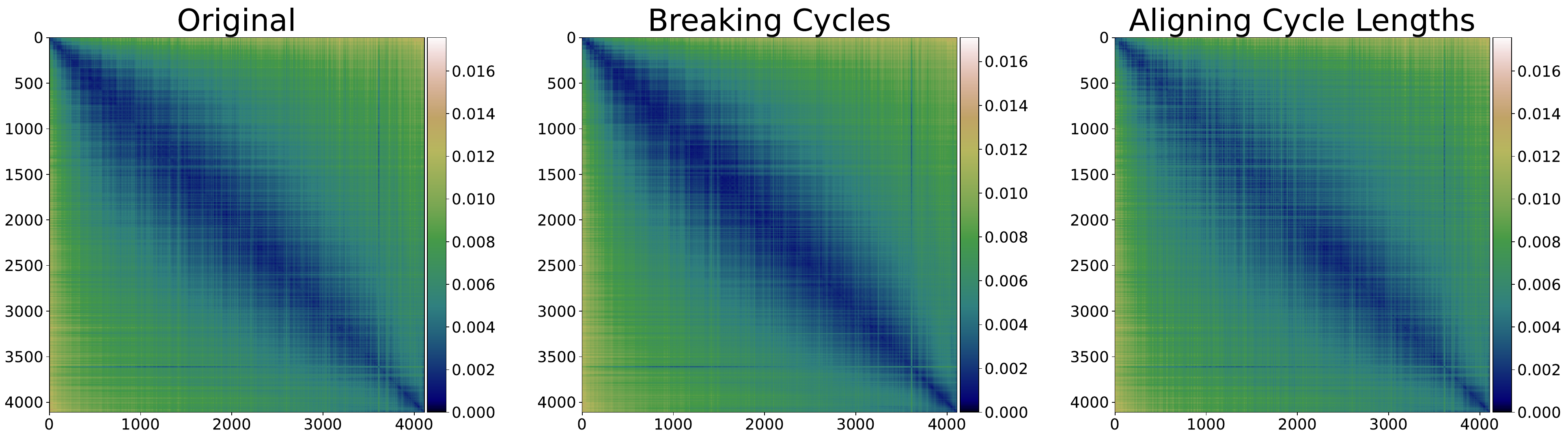}
    \vspace{-0.3cm}
  \caption{{[Qualitative Analysis]} Pairwise graph distance measured by the Wasserstein distance of eigenvalue distributions after breaking cycles and aligning cycle lengths on the NCI109 dataset. Breaking cycles \textbf{amplifies the correlation} between eigenvalue distribution and graph size, while aligning cycle lengths  \textbf{reduces the correlation}. {Quantitative analysis is given in~ \Cref{tab:eigen_difference_break_align}.}}
  \label{fig:cycle_purturbation_eigen}
  \vspace{-0.3cm}
\end{figure*}

\begin{table}[t]
\centering
\caption{{[Quantitative Analysis]} Changes in relative spectral difference after breaking cycles or aligning cycle lengths. {Relative difference is computed as in \Cref{tab:eigen_dist}. We use $\uparrow$ ($\downarrow$) to denote the increase (decrease) in the relative difference compared to not taking the corresponding action.} Breaking cycles results in a larger relative difference, while aligning cycle lengths reduces the relative difference.} 
\label{tab:eigen_difference_break_align}
\small 
\footnotesize
\centering
\begin{tabular}{@{}lc|c@{}}
\toprule
 \multicolumn{1}{l|}{\textbf{Datasets}} & \multicolumn{1}{|c}{\textbf{Breaking cycles}} & \multicolumn{1}{|c}{\textbf{Aligning cycle lengths}} \\ \midrule
\multicolumn{1}{l|}{\textbf{BBBP}}      & $\uparrow$ 50\%                                            & $\downarrow$ 41\%                                            \\
\multicolumn{1}{l|}{\textbf{BACE}}       & $\uparrow$ \;\;6\%                                               & $\downarrow$ 41\%                                            \\
\multicolumn{1}{l|}{\textbf{PROTEINS}}   & $\uparrow$ 31\%                                              & $\downarrow$ 41\%                                            \\ 
\multicolumn{1}{l|}{\textbf{NCI1}}    & $\uparrow$ 53\%                                              & $\downarrow$ 31\%                                            \\
\multicolumn{1}{l|}{\textbf{NCI109}}    & $\uparrow$ 52\%                                               & $\downarrow$ 31\%                                            \\
\bottomrule
\end{tabular}
\vspace{-0.3cm}
\end{table}

\vspace{0.1cm}
\noindent \textbf{Setup.}   We represent the graphs by their empirical distributions of the eigenvalues of the 
normalized adjacency matrix as suggested in~\cite{kipf2016semi}, $\rmT=(\rmD+\rmI)^{-1/2}(\rmA+\rmI)(\rmD+\rmI)^{-1/2}$, and compute their pairwise distances by using the Wasserstein distance~\cite{villani2009optimal}. 
The eigenvalues do not scale with the graph size, and are bounded between [-1,1].  
\Cref{fig:eigenvalue_dsitribution} illustrates the pairwise distances of the graphs arranged in ascending order of size: 
Dark blue represents a small distance (high similarity) while light red represents a large distance (low similarity). We note that using the normalized Laplacian matrix to represent each graph leads to similar observations.

\vspace{0.1cm}
\noindent \textbf{Results.} As can be seen in the three subplots in \cref{fig:eigenvalue_dsitribution}, there is a wide blue band along the diagonal, which indicates that graphs of similar size have more similar eigenvalue distributions than graphs of different sizes. This suggests a strong correlation between the eigenvalue distributions and the graph size. To verify the observation quantitatively, we compute the distance of graphs with `similar size' and graphs of `different sizes' in \Cref{tab:eigen_dist}. 
For each graph, we consider the 20 most `similar graphs' in terms of size, and treat the remaining graphs as graphs of `different sizes'.
The table shows that the Wasserstein distances of eigenvalue distributions between the graphs of different sizes are significantly larger than the distances between graphs of similar size. 
{These empirical results, which highlight the strong dependence of eigenvalue distributions on graph size, are crucial. As we will demonstrate in \S~\ref{subsec: GNN_spec_gen}, this dependency may hinder GNNs from generalizing to larger graph sizes.}

\subsection{Key Findings: Size-based Differences in Subgraph Patterns}
\label{sec:findings}
{Now we {investigate} the subgraph patterns that {can help} explain the spectral differences between small and large graphs.
{Specifically,} we examine how cycle properties differ in small and large graphs and how these differences are revealed in the spectrum. Our analysis aims to answer two questions: (Q1)~How does the existence of cycles in the graphs influence the spectral differences? (Q2)~How do the variations in cycle lengths contribute to differences in the spectrum? To answer these questions, we investigate the properties of the 
cycles in the cycle basis of each graph.}

\subsubsection{(Q1) Existence of Cycles \& Spectrum: The Impact of Breaking Cycles}
\label{subsection: breaking_cycles}
{
To understand how the existence of cycles in the graphs influences the spectral differences,  we break all basis cycles with minimal edge removal while maintaining the same number of disconnected components, according to the details and algorithm given in~\Cref{appen:break_align_cycles}.
We analyze the impact of breaking cycles by assessing the corresponding changes in the spectrum. 
By following the convention of \Cref{fig:eigenvalue_dsitribution}, in the center of \cref{fig:cycle_purturbation_eigen}, we plot the pairwise graph distance based on eigenvalue distributions of {graphs with different} sizes after breaking cycles.
The blue band along the diagonal of the plot becomes darker and narrower, suggesting a larger \yy{relative} spectral difference between small and large graphs and a stronger correlation between the spectrum and graph size. To evaluate the effects quantitatively, we further compute the changes in the {relative spectral difference} and present the results in \Cref{tab:eigen_difference_break_align}. These results indicate that failing to consider cycle information can lead to more significant differences in the spectrum between graphs of varying sizes.}

\subsubsection{(Q2) Cycle Length \& Spectrum: Aligning Cycle Lengths} 

In \S~\ref{subsection: breaking_cycles}, we showed that cycle information is {the key to explain the spectral differences between small and large graphs.}
We now further explore what cycle information reduces {this {relative} difference}, 
{which can in turn improve the size generalization of GNNs.}

To facilitate our exploration, we divide each real-world dataset into two subsets: one subset contains small graphs, and the other subset contains graphs of significantly larger size. {Details regarding this dataset split can be found in~\Cref{appen:dataset_details}.} 
Using this dataset split, we observe a significant difference in the cycle lengths for small and large graphs~(\Cref{appen:break_align_cycles}). 
As described in~\Cref{appen:break_align_cycles}, to reduce that difference, we align the average cycle lengths between small and large graphs by 
randomly inserting redundant nodes to increase the cycle lengths in small graphs. 
{We note that, in \Cref{app: random_nodes}, our approach of aligning the cycle lengths is more effective in reducing the spectral disparities than \textit{randomly} adding the same number of nodes and edges.}

The rightmost heatmap in \cref{fig:cycle_purturbation_eigen} shows how the correlation of eigenvalue distributions and graph size changes after aligning cycle lengths. We observe a lighter blue band along the diagonal, which suggests a weaker correlation between the spectrum and graph size. Going beyond the qualitative analysis, \Cref{tab:eigen_difference_break_align} quantitatively presents the changes in the relative spectral difference between small and large graphs. We observe that aligning cycle lengths results in reduced {relative} disparities in the spectrum between graphs of different sizes. 
Importantly, our observation on cycle length variation differs from the assumptions in~\cite{bevilacqua2021size}. While they assume the density invariance of the induced k-sized subgraphs, we observe the absence of large cycles in small graphs, indicating a lack of density invariance for cycles.

\subsubsection{Interpretation from the Chemical Domain.}
\label{subsec:chemical_interpretation}
While our analysis provides meaningful insights from the datasets, we further explore shifts in cycle statistics from a chemical perspective. Cyclic compounds are fundamental in organic chemistry, often forming larger molecules through repeated unit additions. For example, crown ethers (e.g., 12-Crown-4, 15-Crown-5) and cyclodextrins (e.g., $\alpha$-Cyclodextrin, $\beta$-Cyclodextrin) vary in ring size while maintaining functional linkages, influencing their cation-binding affinities~\citep{gokel2016crown,dodziuk2006cyclodextrins}. In more complex molecules, such as those in our BBBP dataset, cyclic structures affect lipophilicity and molecular flexibility~\citep{pliska1996lipophilicity}. In addition, nodes connecting different rings can be important. For instance, Bicyclo[2.2.0]hexane consists of two fused cyclobutane rings sharing adjacent carbon atoms. The fusion intensifies the ring strain, making the molecule less stable and more reactive compared to unstrained hydrocarbons.
\subsubsection{Cycles \& GNN Performance} 
{So far, we have demonstrated that the spectral differences between small and large graphs are primarily due to variations in their cycle information. Here, we intend to show that these variations in cycle information lead to performance gaps in GNNs when applied to small vs.\ large graphs.}

Towards this end, we perform a controlled experiment to identify the most important nodes for a graph classification task. {Specifically, we first train a model using the small datasets {of BBBP} and then apply GNNExplainer~\cite{ying2019gnnexplainer}, an explainable GNN model, to identify important nodes for label predictions.} {For each graph, we compute the proportion of cycle nodes that are found to be important, finding that 65.6\% of cycle nodes in the small dataset are deemed important for label predictions, while 55.4\% of cycle nodes are found important in the large dataset. This suggests that (1) cycle information is crucial for label prediction, and (2) cycle information in the small datasets is better utilized. These findings provide additional insights into the superior performance of the model on smaller graphs and justify our emphasis on cycle nodes.}

\subsection{{Theoretical Connections}: Graph Spectrum \& GNN Size Generalizability}
\label{subsec: GNN_spec_gen}
{In this subsection, we aim to show how our  findings in real data in \S~\ref{subsec: eigen_bio} and \S~\ref{sec:findings} relate to the generalizability of GNNs. We start by deriving the graph representations learned by spectral GNNs and then analyze the factors that lead to their dependence on graph size. Strong dependence implies poor generalizability of GNNs.}

At a high level, the following proposition asserts that the dependence of graph representations on graph size relies on the dependence of eigenvalue distributions on graph size. 
{Moreover, for GNNs to generalize well to larger graphs, the disparity in the spectra between small and large graphs should be small.}

\vspace{-0.1cm}

{\begin{prop}
\label{prop}
\yy{Assuming that each feature (col. of $\mathbf{X}_i^{(l)}$) is normalized, i.e., $\lVert\mathbf{X}_i^{(l)}\rVert_2/\sqrt{N}=C$ for constant $C$, 
the final node representations learned by spectral GNNs, with $\lambda_i$ denoting the eigenvalues of the propagation matrix $\rmT$, depend on the graph size $N$ as follows:}
\begin{equation}
\footnotesize
\begin{split}
\rmX^{(l+1)} &= \sigma([\sum_{j=1,\cdots,N}{f(\lambda_j)\texttt{COSINE}(\rmU_j, \rmX^{(l)}_1)\cdot C \cdot (\sqrt{N}\cdot\rmU_j)}, \\
&\cdots, \sum_{j=1,\cdots,N}{f(\lambda_j)\texttt{COSINE}(\rmU_j, \rmX^{(l)}_D)\cdot C \cdot (\sqrt{N}\cdot\rmU_j)}]\cdot\rmW^{(l)}),
\end{split}
\end{equation}
\end{prop}
}

\yy{where $f(\cdot)$ is some filter function learned by the GNN.}

\begin{proof}
We give the proof in Appendix~\ref{app:proposition}.
\end{proof}

\begin{remark}
    The final graph representation is obtained through a global pooling (\S~\ref{sec:prelim}), where a set function is applied to each feature dimension. 
    To ensure the graph representation to be irrelevant to the graph size $N$, 
    it suffices to decouple the distribution of eigenvalues ($\lambda_j$) and scaled eigenvectors ($\sqrt{N} \cdot \rmU_j$) from $N$.
    Experiments on real biological graphs (Section~\ref{subsec: eigen_bio} and Appendix~\ref{appen:scaled_eigenvector}) confirm that eigenvalue distribution shifts primarily drive size-dependent variations, while the scaled eigenvectors have less effect. 
    Notably, since eigenvalues of normalized Laplacian and adjacency matrices are bounded ([0,2] and [-1,1], respectively), these shifts are not due to value range~\cite{berlingerio2012scalable}. Our findings also extend to spatial GNNs, given their link to spectral GNNs~\cite{muhammet2020spectral, balcilar2020bridging}. This highlights spectral variations as the key factor in size-related shifts, motivating methods to align the spectra via key subgraph patterns.
\end{remark}

%% file: sections/method.tex
\label{sec:method}
{Our findings in Sec.~\ref{sec: analysis} suggest that GNNs with better ability to identify cycles and generalize over cycle lengths may have better size generalizability on biological graphs. However, recent work~\cite{chen2020can} has found that most GNNs are incapable of learning cycle information. Inspired by these, 
we propose three model-agnostic strategies to help {any }GNN learn the cycle information. {{We note that our objective is not to propose a specific advanced model architecture, but rather to substantiate our findings through model-agnostic designs.} } 
\subsection{Strategy {S1:} Size-insensitive Attention (SIA)}
{One way to incorporate cycle information into GNNs is by encoding it in the features and leveraging them 
within the attention mechanism to guide the learning process. Specifically, for each graph $\mathcal{G}$, we obtain its cycle basis $\mathcal{C}$. Then, for each node $v_i \in \mathcal{G}$, we calculate the average length of the cycle basis to which it belongs:}
\begin{equation}
    \ell_i = \begin{cases}
         \frac{\sum_{j=1}^{|\mathcal{C}|}|\mathcal{C}_j| \cdot \mathbbm{1}_{\{v_i \in \mathcal{C}_j\}}} { \sum_{j=1}^{|\mathcal{C}|} \mathbbm{1}_{\{v_i \in \mathcal{C}_j\}}} &\text{ if } v_i \text{ belongs to some cycles} \\ 
         0 &\text{ otherwise,}
    \end{cases}
\end{equation}
{where $\mathbbm{1}_{\{\text{condition}\}}$ is an indicator function that outputs 1 when the condition is met and 0 otherwise.}  
Then we manually construct a two-dimensional feature vector for each node $v_i$ based on its associated cycle information:
\begin{equation}
    \mathbf{c}_i = [\mathbbm{1}_{\{v_i \in \text{cycle}\}}, \ell_i].
\end{equation}

{We use the structural feature matrix $\mathbf{C} = [\mathbf{c}_1; \ldots; \mathbf{c}_N] \in \R^{N\times2}$ for attention.} 
{Since attention weights often diminish with increasing graph size due to the utilization of \texttt{Softmax}, 
we propose scaling the attention weights by the graph size and employing \texttt{Global\_max} as the global pooling operation to mitigate the impact of graph size. Mathematically, our final graph representation is given by: 
\begin{equation}
    \rvk =  \texttt{Softmax}(\Cbold \rvw^\top_{A}) \cdot N, \; \; \;\xbold^{\mathcal{G}} = \texttt{Global\_max} (  \texttt{Diag}(\rvk) \cdot \Xbold^{(\text{Last})}  )
\end{equation}
where $\rvw^\top_{A}$ is a learnable vector, and $\texttt{Diag}(\cdot)$ creates a diagonal matrix using the vector as its elements.
We note that when we train on small graphs and test on large graphs, some structural features may not be seen in the training, such as certain cycle lengths in the large graphs. We rely on the attention mechanism to generalize to those cases.}

\subsection{Strategy S2: Self-supervised Auxiliary Task}
{Our second proposed strategy utilizes a self-supervised (SSL) auxiliary task to enhance the node representations with cycle-related information. The auxiliary task is to predict whether a node belongs to a cycle. We do not utilize cycle lengths as labels because large test graphs may have cycle lengths not present in the training data. Formally, let $X^{(\text{Last})}$ denote the node representations obtained after the last graph convolution.}
The conventional way of learning is to directly apply a pooling operator and then minimize the loss function for label supervision as below:
\begin{equation}
    \mathcal{L}_{\text{label}} = \texttt{CrossEntropy}(\texttt{Linear}(\texttt{Pooling}(\rmX^{(\text{Last})}), y^{\gG})),
\end{equation}
where $y^{\gG}$ is the ground truth label for the graph. {In this approach, we incorporate an additional loss that aims to diffuse cycle-related information into the node representations through supervision:}
\begin{equation}
    \mathcal{L}_{\text{cycle}} = \texttt{CrossEntropy}(\texttt{MLP}(\rmX^{(\text{Last})}), \rvy_{\text{cycle}}), 
\end{equation}
{where $\rvy_{\text{cycle}}$ is an indicator vector} denoting whether a node belongs to a cycle. {To sum up, the total loss is given by}:
\begin{equation}
    \mathcal{L} = \mathcal{L}_{\text{label}} + \lambda \mathcal{L}_{\text{cycle}}, 
\end{equation}
{where $\lambda$ is a hyperparameter tuned via cross-validation.}

\subsection{Strategy S3: Augmentation}
\vspace{-0.1cm}
{Our third strategy, \ac, aims to
reduce the discrepancy between small and large graphs through direct augmentation for the small training graphs. We augment the training graphs by extending the cycle lengths such that the average cycle length and standard deviation align with those in large graphs. To achieve this augmentation, we use the algorithm detailed in~\Cref{appen:break_align_cycles}. {In short, our algorithm selectively increases the cycle length of a subset of small graphs.} 
{The newly added nodes share} the same features from the nodes with the lowest degrees in the same cycle. 
Last, we feed the augmented training graphs to GNNs for graph classification.
}

%% file: sections/experiment.tex
\label{sec:exp}
\vspace{-0.1cm}

\newcommand{\githubrepo}{\url{https://anonymous.4open.science/r/SizeGen}}

In this section, we conduct extensive experiments to evaluate our proposed strategies. \yy{Starting with real datasets, we aim to answer the following research questions:}

 \noindent  $\bullet$ (\textbf{RQ1}) {Is cycle information helpful for size generalization?} Do our strategies incorporating cycle information effectively enhance the size generalizability of GNNs?  
 
 \noindent  $\bullet$  (\textbf{RQ2}) {How {do our strategies compare to other baseline strategies?} 
 
  \noindent  $\bullet$ (\textbf{RQ3}) {How effective is \texttt{SIA} on large datasets?} 
  
  \noindent  $\bullet$  (\textbf{RQ4})~How effective is it on synthetic datasets that simulate real-world molecular graphs with varying cycle distributions.}\footnote{The code is available at \url{https://github.com/GaotangLi/SizeGenBio}.} 

\subsection{Experimental Setup}
\label{subsec:setup}
\textbf{Real Datasets.} {We use the biological datasets in \Cref{subsec: eigen_bio}.}

\vspace{0.15cm}
\noindent\yy{\textbf{Synthetic Datasets.} As discussed in~\Cref{subsec:chemical_interpretation}, cyclic structures are fundamental in chemical molecules. We consider a binary classification problem: the first category consists of cycles and connections without shared bonds, while the second includes fused cycles (two or more cycles sharing bonds) with additional connections. Each graph is connected, and all features are assigned as i.i.d. nine-dimensional Gaussian vectors, aligning with the feature dimensionality of BBBP. Our dataset reflects real-world molecular structures, where cycle fusion significantly alters molecular properties. Fused rings affect physical characteristics, chemical reactivity, and biological activity. For example, benzene ($C_6H_6$) is a highly stable six-membered aromatic ring, whereas naphthalene ($C_{10}H_{8}$), composed of two fused benzene rings, exhibits increased reactivity in electrophilic substitution and higher melting and boiling points. Similarly, while cyclohexane ($C_{6}H_{12}$) is chemically inert, decalin ($C_{10}H_{18}$), formed by fusing two cyclohexane rings, becomes significantly more rigid.}

\vspace{0.15cm}
\noindent \textbf{Data Preprocessing and Important Training Details.} In order to analyze size generalizability, we {generate} four splits for each dataset: train, validation, small\_test, and large\_test, where large\_test contains graphs with significantly larger sizes. We generate the splits as follows. First, we sort the samples in the dataset by their size. Next, we take the train, validation, and small\_test split from the 50\% smallest graphs in the dataset. An intuitive way of getting the large\_split is to take the top-k largest graphs. However, doing so would result in severe label skewness (class imbalances) between the small\_test and the large\_test. 
To avoid such a severe label shift, we select the same number of graphs per class as in the small\_test subset, starting from the largest graph within each class. {This way guarantees that the label distribution between small\_test and large\_test is the same, while ensuring that the graph size in the latter is 2-10 times larger}. Nevertheless, the smallest 50\% samples still have significant class imbalance. To address this issue, we use upsampling during training throughout the experiments, and we use F1 as the metric to measure the model performance. 
{More details about data preprocessing, hyperparameters, and training can be found in \cref{appen:data_preprocessing} and \cref{appen:training_details}.}  

\vspace{0.15cm}
\noindent \textbf{Baselines.} {We consider a total of 43 baselines: a strong expressive model, \smp~\cite{vignac2020building}, which excels at the cycle detection task, and the combination of 6 backbone models with 7 model-agnostic strategies. Specifically, the}
baseline models are: Multilayer Perceptron (MLP), GCN~\citep{kipf2016semi}, GAT~\citep{velickovic2018graph}, GIN~\cite{xu2018powerful}, FAGCN~\cite{bo2021beyond}, and GNNML3~\cite{balcilar2021breaking}. We integrate seven model-agnostic strategies with these GNN backbones: {our three proposed strategies, (1)~\sia, size-insensitive attention method that leverages the attention mechanism to incorporate cycle information (S1); (2)~\ssl, our self-supervised based strategy that encourages the node representation to contain cycle information (S2); (3)~\ac, our augmentation-based strategy that aligns cycle lengths (S3); as well as other existing model-agnostic strategies: (4)} 
\reb{(4)~SizeShiftReg~\cite{buffellisizeshiftreg} (\ssr), a regularization based on the idea of
simulating a shift in the size of the training graphs using coarsening techniques;
(5)~RPGNN~\cite{murphy2019relational} (\rpgin), an expressive model for arbitrary-sized graphs; and (6)\texttt{+CIGAv1} \& (7)\texttt{+CIGAv2}~\cite{chen2022learning}, two versions of a causal model that effectively handles out-of-distribution problems on graphs.
}   

\begin{table*}[t]
\caption{Size generalizability evaluated by the graph classification performance on small and large test graphs. The performance is reported by the average F1 scores and its standard deviation. 
The rightmost column denotes the average improvements compared with the original performance using the same backbone model across five different datasets. The largest average improvement within the same model and small/large category is highlighted in orange. \textbf{All strategies enhance GNNs’ size generalizability,
with \sia drop as the most effective
method.}}
\label{tab:cycle_aware_table}
\resizebox{\linewidth}{!}{ 
\begin{tabular}{@{}lllllllllllll@{}}
\toprule
\multicolumn{1}{c}{\textbf{Datasets}}                & \multicolumn{2}{c}{\textbf{BBBP}}                    & \multicolumn{2}{c}{\textbf{BACE}}                    & \multicolumn{2}{c}{\textbf{PROTEINS}}                & \multicolumn{2}{c}{\textbf{NCI1}}                    & \multicolumn{2}{c}{\textbf{NCI109}} & \multicolumn{2}{c}{\textbf{Avg Improv}} \\ \midrule
\multicolumn{1}{l|}{\textbf{Models}}                 & \textbf{Small} & \multicolumn{1}{l|}{\textbf{Large}} & \textbf{Small} & \multicolumn{1}{l|}{\textbf{Large}} & \textbf{Small} & \multicolumn{1}{l|}{\textbf{Large}} & \textbf{Small} & \multicolumn{1}{l|}{\textbf{Large}} & \textbf{Small}     & \multicolumn{1}{l|}{\textbf{Large}} & \textbf{Small}        & \textbf{Large}         \\ \midrule
\multicolumn{1}{l|}{\textbf{MLP}}                    & 90.36\tiny\(\pm\)0.71     & \multicolumn{1}{l|}{55.61\tiny\(\pm\)3.37}     & 61.06\tiny\(\pm\)5.79     & \multicolumn{1}{l|}{21.06\tiny\(\pm\)7.89}     & 36.15\tiny\(\pm\)2.28     & \multicolumn{1}{l|}{21.55\tiny\(\pm\)1.34}     & 36.43\tiny\(\pm\)3.89     & \multicolumn{1}{l|}{3.36\tiny\(\pm\)2.87}      & 35.87\tiny\(\pm\)4.23         & \multicolumn{1}{l|}{4.65\tiny\(\pm\)3.72} & - & - \\
\multicolumn{1}{l|}{\textbf{MLP+SSL}}                & 90.90\tiny\(\pm\)1.76      & \multicolumn{1}{l|}{62.56\tiny\(\pm\)5.48}     & 58.57\tiny\(\pm\)8.85     & \multicolumn{1}{l|}{23.01\tiny\(\pm\)11.95}    & 35.00\tiny\(\pm\)2.8       & \multicolumn{1}{l|}{20.88\tiny\(\pm\)1.64}     & 34.71\tiny\(\pm\)1.33     & \multicolumn{1}{l|}{2.86\tiny\(\pm\)0.78}      & 37.29\tiny\(\pm\)4.69          & \multicolumn{1}{l|}{6.34\tiny\(\pm\)4.78}  & -0.68 & +1.88  \\
\multicolumn{1}{l|}{\textbf{MLP+AugCyc}}       & 90.72\tiny\(\pm\)2.70     & \multicolumn{1}{l|}{57.86\tiny\(\pm\)4.74}     & 59.88\tiny\(\pm\)7.33     & \multicolumn{1}{l|}{26.50\tiny\(\pm\)14.97}     & 37.29\tiny\(\pm\)0.0      & \multicolumn{1}{l|}{22.22\tiny\(\pm\)0.0}      & 36.98\tiny\(\pm\)2.29     & \multicolumn{1}{l|}{2.64\tiny\(\pm\)2.1}       & 40.59\tiny\(\pm\)3.86         & \multicolumn{1}{l|}{8.41\tiny\(\pm\)3.71}   & +1.12 & +2.28  \\
\multicolumn{1}{l|}{\textbf{MLP+SIA}}                & 90.38\tiny\(\pm\)1.05     & \multicolumn{1}{l|}{62.79\tiny\(\pm\)7.55}     & 60.85\tiny\(\pm\)7.83     & \multicolumn{1}{l|}{21.79\tiny\(\pm\)15.07}      & 40.68\tiny\(\pm\)3.56     & \multicolumn{1}{l|}{33.57\tiny\(\pm\)11.87}    & 35.42\tiny\(\pm\)3.83     & \multicolumn{1}{l|}{3.26\tiny\(\pm\)2.55}      & 39.20\tiny\(\pm\)2.96          & \multicolumn{1}{l|}{12.2\tiny\(\pm\)5.05}    & \makecolor+1.33  & \makecolor+5.48 \\ \midrule
\multicolumn{1}{l|}{\textbf{GCN}}                    & 91.37\tiny\(\pm\)0.59     & \multicolumn{1}{l|}{68.59\tiny\(\pm\)7.47}     & 63.68\tiny\(\pm\)6.63     & \multicolumn{1}{l|}{28.72\tiny\(\pm\)14.26}     & 72.35\tiny\(\pm\)2.58     & \multicolumn{1}{l|}{40.57\tiny\(\pm\)7.67}     & 54.91\tiny\(\pm\)2.37     & \multicolumn{1}{l|}{28.80\tiny\(\pm\)7.57}     & 60.83\tiny\(\pm\)1.92         & \multicolumn{1}{l|}{30.45\tiny\(\pm\)4.34}   & - & - \\
\multicolumn{1}{l|}{\textbf{GCN+SSL}}                & 92.66\tiny\(\pm\)1.21     & \multicolumn{1}{l|}{73.24\tiny\(\pm\)5.71}     & 64.92\tiny\(\pm\)4.44     & \multicolumn{1}{l|}{32.84\tiny\(\pm\)16.08}    & 72.46\tiny\(\pm\)1.58     & \multicolumn{1}{l|}{41.21\tiny\(\pm\)6.66}     & 57.43\tiny\(\pm\)3.23     & \multicolumn{1}{l|}{32.58\tiny\(\pm\)10.08}    & 60.50\tiny\(\pm\)3.09         & \multicolumn{1}{l|}{27.35\tiny\(\pm\)11.42} & +0.97 & +2.01  \\
\multicolumn{1}{l|}{\textbf{GCN+AugCyc}}       & 91.41\tiny\(\pm\)1.33     & \multicolumn{1}{l|}{68.08\tiny\(\pm\)7.65}     & 63.83\tiny\(\pm\)5.44     & \multicolumn{1}{l|}{35.65\tiny\(\pm\)7.70}     & 72.87\tiny\(\pm\)3.68     & \multicolumn{1}{l|}{54.73\tiny\(\pm\)8.24}     & 53.85\tiny\(\pm\)3.71     & \multicolumn{1}{l|}{27.39\tiny\(\pm\)8.33}     & 62.78\tiny\(\pm\)2.98         & \multicolumn{1}{l|}{33.62\tiny\(\pm\)3.58}  & +0.32 & +4.47  \\
\multicolumn{1}{l|}{\textbf{GCN+SIA}}                & 91.32\tiny\(\pm\)0.73     & \multicolumn{1}{l|}{71.66\tiny\(\pm\)6.99}     & 64.35\tiny\(\pm\)9.76     & \multicolumn{1}{l|}{24.24\tiny\(\pm\)17.03}     & 73.84\tiny\(\pm\)3.65     & \multicolumn{1}{l|}{58.74\tiny\(\pm\)9.49}     & 59.78\tiny\(\pm\)1.65     & \multicolumn{1}{l|}{45.70\tiny\(\pm\)6.70}     & 60.32\tiny\(\pm\)2.90         & \multicolumn{1}{l|}{38.78\tiny\(\pm\)4.55} & \makecolor+1.29 & \makecolor+8.40   \\ \midrule
\multicolumn{1}{l|}{\textbf{GAT}}                    & 91.27\tiny\(\pm\)1.43     & \multicolumn{1}{l|}{68.35\tiny\(\pm\)7.02}     & 69.73\tiny\(\pm\)2.05     & \multicolumn{1}{l|}{42.23\tiny\(\pm\)11.18}    & 72.25\tiny\(\pm\)4.25    & \multicolumn{1}{l|}{43.86\tiny\(\pm\)6.82}    & 58.22\tiny\(\pm\)2.86     & \multicolumn{1}{l|}{49.36\tiny\(\pm\)4.12}     & 64.39\tiny\(\pm\)3.29         & \multicolumn{1}{l|}{38.36\tiny\(\pm\)8.93}   &  - & - \\
\multicolumn{1}{l|}{\textbf{GAT+SSL}}                & 91.65\tiny\(\pm\)0.92     & \multicolumn{1}{l|}{74.24\tiny\(\pm\)7.34}     & 71.20\tiny\(\pm\)2.04     & \multicolumn{1}{l|}{40.88\tiny\(\pm\)10.81}    & 74.20\tiny\(\pm\)1.46     & \multicolumn{1}{l|}{49.30\tiny\(\pm\)5.56}      & 59.47\tiny\(\pm\)2.89     & \multicolumn{1}{l|}{51.85\tiny\(\pm\)4.03}     & 66.79\tiny\(\pm\)3.56         & \multicolumn{1}{l|}{42.20\tiny\(\pm\)6.71}  & \makecolor+1.49 & +3.26  \\
\multicolumn{1}{l|}{\textbf{GAT+AugCyc}}       & 92.41\tiny\(\pm\)1.29     & \multicolumn{1}{l|}{69.57\tiny\(\pm\)2.89}     & 68.39\tiny\(\pm\)6.06     & \multicolumn{1}{l|}{40.73\tiny\(\pm\)13.4}     & 74.99\tiny\(\pm\)1.89     & \multicolumn{1}{l|}{59.80\tiny\(\pm\)7.27}     & 56.23\tiny\(\pm\)3.85     & \multicolumn{1}{l|}{49.37\tiny\(\pm\)7.52}     & 64.07\tiny\(\pm\)3.46         & \multicolumn{1}{l|}{45.25\tiny\(\pm\)9.19}  & +0.05 & +4.51 \\
\multicolumn{1}{l|}{\textbf{GAT+SIA}}                & 91.88\tiny\(\pm\)2.12     & \multicolumn{1}{l|}{74.87\tiny\(\pm\)5.62}     & 69.64\tiny\(\pm\)6.79     & \multicolumn{1}{l|}{43.87\tiny\(\pm\)7.98}     & 75.35\tiny\(\pm\)3.28     & \multicolumn{1}{l|}{62.71\tiny\(\pm\)4.98}     & 61.42\tiny\(\pm\)1.07     & \multicolumn{1}{l|}{55.73\tiny\(\pm\)12.98}    & 63.27\tiny\(\pm\)3.15         & \multicolumn{1}{l|}{45.97\tiny\(\pm\)7.74}  & +1.14 & \makecolor+8.20  \\ \midrule
\multicolumn{1}{l|}{\textbf{GIN}}                    & 88.28\tiny\(\pm\)2.39     & \multicolumn{1}{l|}{66.67\tiny\(\pm\)5.55}     & 57.02\tiny\(\pm\)6.48     & \multicolumn{1}{l|}{22.97\tiny\(\pm\)10.26}     & 74.55\tiny\(\pm\)4.27     & \multicolumn{1}{l|}{50.20\tiny\(\pm\)5.36}     & 62.17\tiny\(\pm\)3.86     & \multicolumn{1}{l|}{44.26\tiny\(\pm\)7.03}     & 62.42\tiny\(\pm\)2.77         & \multicolumn{1}{l|}{33.23\tiny\(\pm\)6.77}  & - & -  \\
\multicolumn{1}{l|}{\textbf{GIN+SSL}}                & 91.13\tiny\(\pm\)1.32     & \multicolumn{1}{l|}{68.67\tiny\(\pm\)9.75}     & 56.46\tiny\(\pm\)8.59     & \multicolumn{1}{l|}{23.91\tiny\(\pm\)10.64}    & 75.47\tiny\(\pm\)1.15     & \multicolumn{1}{l|}{48.14\tiny\(\pm\)4.00}     & 61.18\tiny\(\pm\)3.53     & \multicolumn{1}{l|}{46.47\tiny\(\pm\)9.86}     & 63.11\tiny\(\pm\)4.05         & \multicolumn{1}{l|}{35.0\tiny\(\pm\)11.43} & +0.58 &  +0.97  \\
\multicolumn{1}{l|}{\textbf{GIN+AugCyc}}       & 92.56\tiny\(\pm\)1.17     & \multicolumn{1}{l|}{77.69\tiny\(\pm\)5.63}     & 58.30\tiny\(\pm\)5.29     & \multicolumn{1}{l|}{23.89\tiny\(\pm\)13.17}    & 74.56\tiny\(\pm\)2.92     & \multicolumn{1}{l|}{51.02\tiny\(\pm\)8.42}     & 62.70\tiny\(\pm\)0.94     & \multicolumn{1}{l|}{46.76\tiny\(\pm\)5.34}     & 64.56\tiny\(\pm\)5.45         & \multicolumn{1}{l|}{37.16\tiny\(\pm\)5.86}   & +1.65 & \makecolor+3.84 \\
\multicolumn{1}{l|}{\textbf{GIN+SIA}}                & 92.70\tiny\(\pm\)0.45     & \multicolumn{1}{l|}{75.99\tiny\(\pm\)4.74}     & 61.30\tiny\(\pm\)6.77     & \multicolumn{1}{l|}{24.42\tiny\(\pm\)16.37}    & 74.88\tiny\(\pm\)4.24     & \multicolumn{1}{l|}{51.36\tiny\(\pm\)7.76}     & 62.83\tiny\(\pm\)1.07     & \multicolumn{1}{l|}{42.82\tiny\(\pm\)8.92}     & 63.00\tiny\(\pm\)4.24          & \multicolumn{1}{l|}{41.65\tiny\(\pm\)4.19}  & \makecolor+2.05  & +3.78  \\ \midrule
\multicolumn{1}{l|}{\textbf{FAGCN}}                  & 90.58\tiny\(\pm\)1.72     & \multicolumn{1}{l|}{64.93\tiny\(\pm\)7.62}     & 62.96\tiny\(\pm\)2.12     & \multicolumn{1}{l|}{24.65\tiny\(\pm\)11.71}    & 70.03\tiny\(\pm\)5.20      & \multicolumn{1}{l|}{42.34\tiny\(\pm\)6.61}     & 43.51\tiny\(\pm\)4.29     & \multicolumn{1}{l|}{10.16\tiny\(\pm\)7.80}     & 55.78\tiny\(\pm\)3.5          & \multicolumn{1}{l|}{22.65\tiny\(\pm\)12.87}  & - & - \\
\multicolumn{1}{l|}{\textbf{FAGCN+SSL}}              & 91.55\tiny\(\pm\)2.51     & \multicolumn{1}{l|}{67.56\tiny\(\pm\)5.48}     & 64.67\tiny\(\pm\)3.88     & \multicolumn{1}{l|}{35.46\tiny\(\pm\)16.52}    & 66.97\tiny\(\pm\)1.75     & \multicolumn{1}{l|}{48.06\tiny\(\pm\)8.33}     & 46.42\tiny\(\pm\)6.08     & \multicolumn{1}{l|}{12.11\tiny\(\pm\)5.39}      & 56.04\tiny\(\pm\)4.29         & \multicolumn{1}{l|}{23.99\tiny\(\pm\)10.57}  & \makecolor+0.56 & \makecolor+4.49   \\
\multicolumn{1}{l|}{\textbf{FAGCN+AugCyc}}     & 91.30\tiny\(\pm\)2.26     & \multicolumn{1}{l|}{71.44\tiny\(\pm\)6.45}     & 57.68\tiny\(\pm\)3.38     & \multicolumn{1}{l|}{26.41\tiny\(\pm\)23.39}    & 68.85\tiny\(\pm\)16.12    & \multicolumn{1}{l|}{44.39\tiny\(\pm\)16.89}    & 39.48\tiny\(\pm\)4.99     & \multicolumn{1}{l|}{10.98\tiny\(\pm\)5.45}      & 55.30\tiny\(\pm\)3.46         & \multicolumn{1}{l|}{24.59\tiny\(\pm\)9.19}  & -2.05 & +2.62 \\
\multicolumn{1}{l|}{\textbf{FAGCN+SIA}}              & 90.17\tiny\(\pm\)2.83     & \multicolumn{1}{l|}{74.65\tiny\(\pm\)9.13}     & 62.40\tiny\(\pm\)3.36     & \multicolumn{1}{l|}{30.35\tiny\(\pm\)13.48}    & 71.30\tiny\(\pm\)5.79     & \multicolumn{1}{l|}{48.94\tiny\(\pm\)10.62}    & 46.95\tiny\(\pm\)5.71     & \multicolumn{1}{l|}{10.99\tiny\(\pm\)7.50}       & 52.82\tiny\(\pm\)6.28         & \multicolumn{1}{l|}{19.08\tiny\(\pm\)5.32} & +0.16 & +3.86   \\ \midrule
\multicolumn{1}{l|}{\textbf{GNNML3}}                 & 92.01\tiny\(\pm\)1.56     & \multicolumn{1}{l|}{64.18\tiny\(\pm\)6.99}     & 62.31\tiny\(\pm\)4.90     & \multicolumn{1}{l|}{32.94\tiny\(\pm\)12.86}    & 71.59\tiny\(\pm\)3.5      & \multicolumn{1}{l|}{40.74\tiny\(\pm\)15.0}     & 63.73\tiny\(\pm\)4.67     & \multicolumn{1}{l|}{51.75\tiny\(\pm\)9.05}     & 59.39\tiny\(\pm\)3.76         & \multicolumn{1}{l|}{33.80\tiny\(\pm\)11.19} & - & -  \\
\multicolumn{1}{l|}{\textbf{GNNML3+SSL}}             & 92.96\tiny\(\pm\)1.54     & \multicolumn{1}{l|}{64.18\tiny\(\pm\)8.62}     & 65.65\tiny\(\pm\)5.69     & \multicolumn{1}{l|}{31.78\tiny\(\pm\)12.72}     & 74.41\tiny\(\pm\)3.21     & \multicolumn{1}{l|}{56.81\tiny\(\pm\)3.49}     & 63.91\tiny\(\pm\)3.34     & \multicolumn{1}{l|}{48.84\tiny\(\pm\)10.07}    & 61.01\tiny\(\pm\)2.44         & \multicolumn{1}{l|}{35.13\tiny\(\pm\)9.49}   & \makecolor+1.78 & +2.67 \\
\multicolumn{1}{l|}{\textbf{GNNML3+AugCyc}}    & 91.38\tiny\(\pm\)2.92     & \multicolumn{1}{l|}{69.82\tiny\(\pm\)5.51}     & 63.36\tiny\(\pm\)2.78     & \multicolumn{1}{l|}{32.59\tiny\(\pm\)10.32}    & 70.54\tiny\(\pm\)5.00     & \multicolumn{1}{l|}{38.79\tiny\(\pm\)5.21}     & 62.30\tiny\(\pm\)3.27     & \multicolumn{1}{l|}{55.57\tiny\(\pm\)11.73}    & 58.18\tiny\(\pm\)3.17         & \multicolumn{1}{l|}{41.30\tiny\(\pm\)14.75}   & -0.65 & +2.93\\
\multicolumn{1}{l|}{\textbf{GNNML3+SIA}}             & 92.70\tiny\(\pm\)0.81     & \multicolumn{1}{l|}{70.43\tiny\(\pm\)6.36}     & 64.57\tiny\(\pm\)2.72     & \multicolumn{1}{l|}{37.73\tiny\(\pm\)7.68}     & 69.32\tiny\(\pm\)3.79     & \multicolumn{1}{l|}{48.94\tiny\(\pm\)10.62}    & 63.91\tiny\(\pm\)5.81     & \multicolumn{1}{l|}{48.85\tiny\(\pm\)12.11}    & 61.58\tiny\(\pm\)3.98         & \multicolumn{1}{l|}{49.70\tiny\(\pm\)17.85}  & +0.61 & \makecolor+6.45  \\ \bottomrule
\end{tabular}
}

\vspace{-0.2cm}
\end{table*}

\begin{table*}[ht]
\caption{Size generalizability evaluated with other baselines, following the same rule as in~\Cref{tab:cycle_aware_table}. \sia consistently and significantly outperforms other strategies regarding size generalizability.}
\label{tab:baseline_table}
\resizebox{0.95\linewidth}{!}{ 
\begin{tabular}{@{}lllllllllllll@{}}
\toprule
\multicolumn{1}{c}{\textbf{Datasets}}                & \multicolumn{2}{c}{\textbf{BBBP}}                    & \multicolumn{2}{c}{\textbf{BACE}}                    & \multicolumn{2}{c}{\textbf{PROTEINS}}                & \multicolumn{2}{c}{\textbf{NCI1}}                    & \multicolumn{2}{c}{\textbf{NCI109}} & \multicolumn{2}{c}{\textbf{Avg Improv}} \\ \midrule
\multicolumn{1}{l|}{\textbf{Models}}                 & \textbf{Small} & \multicolumn{1}{l|}{\textbf{Large}} & \textbf{Small} & \multicolumn{1}{l|}{\textbf{Large}} & \textbf{Small} & \multicolumn{1}{l|}{\textbf{Large}} & \textbf{Small} & \multicolumn{1}{l|}{\textbf{Large}} & \textbf{Small}     & \multicolumn{1}{l|}{\textbf{Large}} & \textbf{Small}        & \textbf{Large}         \\ \midrule
\multicolumn{1}{l|}{\textbf{MLP}}                    & 90.36\tiny\(\pm\)0.71     & \multicolumn{1}{l|}{55.61\tiny\(\pm\)3.37}     & 61.06\tiny\(\pm\)5.79     & \multicolumn{1}{l|}{21.06\tiny\(\pm\)7.89}     & 36.15\tiny\(\pm\)2.28     & \multicolumn{1}{l|}{21.55\tiny\(\pm\)1.34}     & 36.43\tiny\(\pm\)3.89     & \multicolumn{1}{l|}{3.36\tiny\(\pm\)2.87}      & 35.87\tiny\(\pm\)4.23         & \multicolumn{1}{l|}{4.65\tiny\(\pm\)3.72} & - & - \\

\multicolumn{1}{l|}{\reb{\textbf{MLP+RPGNN}}}                    & \reb{90.44\tiny\(\pm\)1.03}     & \multicolumn{1}{l|}{\reb{55.10\tiny\(\pm\)7.64}}     & \reb{57.80\tiny\(\pm\)8.53}    & \multicolumn{1}{l|}{\reb{21.26\tiny\(\pm\)8.20}}     & \reb{45.60\tiny\(\pm\)2.69}     & \multicolumn{1}{l|}{\reb{20.71\tiny\(\pm\)1.41}}     & \reb{35.34\tiny\(\pm\)4.35}     & \multicolumn{1}{l|}{\reb{13.45\tiny\(\pm\)25.12}}      & \reb{38.60\tiny\(\pm\)2.37}         & \multicolumn{1}{l|}{\reb{10.44\tiny\(\pm\)18.77}} & \reb{\makecolor+1.58} & \reb{+2.95} \\

\multicolumn{1}{l|}{\reb{\textbf{MLP+SSR}}}    & \reb{91.07\tiny\(\pm\)0.67}     & \multicolumn{1}{l|}{\reb{57.02\tiny\(\pm\)9.04}}     & \reb{60.35\tiny\(\pm\)5.36}     & \multicolumn{1}{l|}{\reb{25.42\tiny\(\pm\)4.02}}    & \reb{37.15\tiny\(\pm\)2.48}     & \multicolumn{1}{l|}{\reb{22.77\tiny\(\pm\)4.28}}     & \reb{34.42\tiny\(\pm\)1.89}     & \multicolumn{1}{l|}{\reb{1.27\tiny\(\pm\)0.63}}    & \reb{38.68\tiny\(\pm\)4.70}         & \multicolumn{1}{l|}{\reb{6.96\tiny\(\pm\)4.65}}   & \reb{+0.36} & \reb{+1.44} \\ 

\multicolumn{1}{l|}{\reb{\textbf{MLP+CIGAv1}}}                    & \reb{88.15\tiny\(\pm\)1.05}     & \multicolumn{1}{l|}{\reb{54.14\tiny\(\pm\)6.07}}     & \reb{58.18\tiny\(\pm\)10.81}    & \multicolumn{1}{l|}{\reb{24.99\tiny\(\pm\)14.41}}     & \reb{34.68\tiny\(\pm\)4.23}     & \multicolumn{1}{l|}{\reb{18.23\tiny\(\pm\)2.50}}     & \reb{33.52\tiny\(\pm\)7.85}     & \multicolumn{1}{l|}{\reb{8.93\tiny\(\pm\)7.48}}      & \reb{32.79\tiny\(\pm\)1.20}         & \multicolumn{1}{l|}{\reb{3.19\tiny\(\pm\)0.01}} & \reb{-2.51} & \reb{+0.65} \\

\multicolumn{1}{l|}{\reb{\textbf{MLP+CIGAv2}}}                    & \reb{88.08\tiny\(\pm\)3.84}     & \multicolumn{1}{l|}{\reb{61.70\tiny\(\pm\)13.15}}     & \reb{57.46\tiny\(\pm\)4.71}    & \multicolumn{1}{l|}{\reb{18.43\tiny\(\pm\)16.73}}     & \reb{38.28\tiny\(\pm\)4.23}     & \multicolumn{1}{l|}{\reb{24.87\tiny\(\pm\)5.62}}     & \reb{35.77\tiny\(\pm\)7.55}     & \multicolumn{1}{l|}{\reb{4.22\tiny\(\pm\)6.89}}      & \reb{38.45\tiny\(\pm\)3.68}         & \multicolumn{1}{l|}{\reb{4.55\tiny\(\pm\)1.28}} & \reb{-0.37} & \reb{+1.51} \\

\multicolumn{1}{l|}{\textbf{MLP+SIA}}                & 90.38\tiny\(\pm\)1.05     & \multicolumn{1}{l|}{62.79\tiny\(\pm\)7.55}     & 60.85\tiny\(\pm\)7.83     & \multicolumn{1}{l|}{21.79\tiny\(\pm\)15.07}      & 40.68\tiny\(\pm\)3.56     & \multicolumn{1}{l|}{33.57\tiny\(\pm\)11.87}    & 35.42\tiny\(\pm\)3.83     & \multicolumn{1}{l|}{3.26\tiny\(\pm\)2.55}      & 39.20\tiny\(\pm\)2.96          & \multicolumn{1}{l|}{12.2\tiny\(\pm\)5.05}    & +1.33  & \makecolor+5.48 \\ \midrule
\multicolumn{1}{l|}{\textbf{GCN}}                    & 91.37\tiny\(\pm\)0.59     & \multicolumn{1}{l|}{68.59\tiny\(\pm\)7.47}     & 63.68\tiny\(\pm\)6.63     & \multicolumn{1}{l|}{28.72\tiny\(\pm\)14.26}     & 72.35\tiny\(\pm\)2.58     & \multicolumn{1}{l|}{40.57\tiny\(\pm\)7.67}     & 54.91\tiny\(\pm\)2.37     & \multicolumn{1}{l|}{28.80\tiny\(\pm\)7.57}     & 60.83\tiny\(\pm\)1.92         & \multicolumn{1}{l|}{30.45\tiny\(\pm\)4.34}   & - & - \\

\multicolumn{1}{l|}{\reb{\textbf{GCN+RPGNN}}}                    & \reb{92.27\tiny\(\pm\)0.81}     & \multicolumn{1}{l|}{\reb{68.69\tiny\(\pm\)6.58}}     & \reb{63.70\tiny\(\pm\)1.06}    & \multicolumn{1}{l|}{\reb{33.86\tiny\(\pm\)13.91}}     & \reb{74.74\tiny\(\pm\)3.75}     & \multicolumn{1}{l|}{\reb{24.61\tiny\(\pm\)10.08}}     & \reb{58.88\tiny\(\pm\)2.03}     & \multicolumn{1}{l|}{\reb{34.68\tiny\(\pm\)10.77}}      & \reb{63.10\tiny\(\pm\)1.86}         & \multicolumn{1}{l|}{\reb{39.69\tiny\(\pm\)5,88}} & \reb{+1.91} & \reb{+0.88} \\ 

\multicolumn{1}{l|}{\reb{\textbf{GCN+SSR}}}    & \reb{91.19\tiny\(\pm\)1.14}     & \multicolumn{1}{l|}{\reb{68.15\tiny\(\pm\) 6.38}}     & \reb{66.01\tiny\(\pm\) 2.51}     & \multicolumn{1}{l|}{\reb{31.64\tiny\(\pm\)9.96}}    & \reb{73.51\tiny\(\pm\) 2.91}     & \multicolumn{1}{l|}{\reb{43.33\tiny\(\pm\)5.19}}     & \reb{59.60\tiny\(\pm\)2.62}     & \multicolumn{1}{l|}{\reb{35.01\tiny\(\pm\)7.13}}    & \reb{59.78\tiny\(\pm\)2.71}         & \multicolumn{1}{l|}{\reb{33.11\tiny\(\pm\)6.44}}   & \reb{+1.39} & \reb{+2.82} \\ 

\multicolumn{1}{l|}{\reb{\textbf{GCN+CIGAv1}}}                    & \reb{90.55\tiny\(\pm\)1.32}     & \multicolumn{1}{l|}{\reb{66.55\tiny\(\pm\)5.80}}     & \reb{66.66\tiny\(\pm\)5.72}    & \multicolumn{1}{l|}{\reb{28.51\tiny\(\pm\)8.64}}     & \reb{72.64\tiny\(\pm\)1.81}     & \multicolumn{1}{l|}{\reb{54.67\tiny\(\pm\)6.08}}     & \reb{58.52\tiny\(\pm\)4.88}     & \multicolumn{1}{l|}{\reb{40.82\tiny\(\pm\)11.14}}      & \reb{59.09\tiny\(\pm\)3.50}         & \multicolumn{1}{l|}{\reb{25.82\tiny\(\pm\)7.81}} & \reb{+0.86} & \reb{+3.85} \\

\multicolumn{1}{l|}{\reb{\textbf{GCN+CIGAv2}}}                    & \reb{89.45\tiny\(\pm\)3.60}     & \multicolumn{1}{l|}{\reb{69.71\tiny\(\pm\)8.20}}     & \reb{65.02\tiny\(\pm\)1.80}    & \multicolumn{1}{l|}{\reb{35.42\tiny\(\pm\)12.36}}     & \reb{72.15\tiny\(\pm\)3.86}     & \multicolumn{1}{l|}{\reb{60.12\tiny\(\pm\)6.84}}     & \reb{57.89\tiny\(\pm\)3.74}     & \multicolumn{1}{l|}{\reb{35.42\tiny\(\pm\)10.75}}      & \reb{58.12\tiny\(\pm\)5.37}         & \multicolumn{1}{l|}{\reb{28.51\tiny\(\pm\)10.10}} & \reb{-0.10} & \reb{+6.41} \\

\multicolumn{1}{l|}{\textbf{GCN+SIA}}                & 91.32\tiny\(\pm\)0.73     & \multicolumn{1}{l|}{71.66\tiny\(\pm\)6.99}     & 64.35\tiny\(\pm\)9.76     & \multicolumn{1}{l|}{24.24\tiny\(\pm\)17.03}     & 73.84\tiny\(\pm\)3.65     & \multicolumn{1}{l|}{58.74\tiny\(\pm\)9.49}     & 59.78\tiny\(\pm\)1.65     & \multicolumn{1}{l|}{45.70\tiny\(\pm\)6.70}     & 60.32\tiny\(\pm\)2.90         & \multicolumn{1}{l|}{38.78\tiny\(\pm\)4.55} & \makecolor+1.29 & \makecolor+8.40   \\ \midrule
\multicolumn{1}{l|}{\textbf{GAT}}                    & 91.27\tiny\(\pm\)1.43     & \multicolumn{1}{l|}{68.35\tiny\(\pm\)7.02}     & 69.73\tiny\(\pm\)2.05     & \multicolumn{1}{l|}{42.23\tiny\(\pm\)11.18}    & 72.25\tiny\(\pm\)4.25    & \multicolumn{1}{l|}{43.86\tiny\(\pm\)6.82}    & 58.22\tiny\(\pm\)2.86     & \multicolumn{1}{l|}{49.36\tiny\(\pm\)4.12}     & 64.39\tiny\(\pm\)3.29         & \multicolumn{1}{l|}{38.36\tiny\(\pm\)8.93}   &  - & - \\

\multicolumn{1}{l|}{\reb{\textbf{GAT+RPGNN}}}                    & \reb{91.76\tiny\(\pm\)1.69}     & \multicolumn{1}{l|}{\reb{65.85\tiny\(\pm\)5.37}}     & \reb{69.97\tiny\(\pm\)2.17}    & \multicolumn{1}{l|}{\reb{39.27\tiny\(\pm\)13.50}}     & \reb{72.89\tiny\(\pm\)3.35}     & \multicolumn{1}{l|}{\reb{38.49\tiny\(\pm\)6.14}}     & \reb{59.31\tiny\(\pm\)5.51}     & \multicolumn{1}{l|}{\reb{58.18\tiny\(\pm\)5.76}}      & \reb{65.52\tiny\(\pm\)1.94}         & \multicolumn{1}{l|}{\reb{44.15\tiny\(\pm\)5.76}} & \reb{-0.75} & \reb{+0.76} \\ 

\multicolumn{1}{l|}{\reb{\textbf{GAT+SSR}}}    & \reb{91.98\tiny\(\pm\)0.66}     & \multicolumn{1}{l|}{\reb{74.83\tiny\(\pm\)4.35}}     & \reb{66.03\tiny\(\pm\)3.83}     & \multicolumn{1}{l|}{\reb{41.41\tiny\(\pm\)11.8}}    & \reb{74.72\tiny\(\pm\)3.51}     & \multicolumn{1}{l|}{\reb{44.81\tiny\(\pm\)8.59}}     & \reb{60.68\tiny\(\pm\)1.95}     & \multicolumn{1}{l|}{\reb{49.64\tiny\(\pm\)5.26}}    & \reb{66.73\tiny\(\pm\)1.65}         & \multicolumn{1}{l|}{\reb{41.14\tiny\(\pm\)4.41}}   & \reb{+0.86} & \reb{+1.93} \\ 

\multicolumn{1}{l|}{\reb{\textbf{GAT+CIGAv1}}}                    & \reb{89.53\tiny\(\pm\)1.51}     & \multicolumn{1}{l|}{\reb{67.35\tiny\(\pm\)8.74}}     & \reb{67.18\tiny\(\pm\)5.12}    & \multicolumn{1}{l|}{\reb{39.88\tiny\(\pm\)18.64}}     & \reb{73.28\tiny\(\pm\)4.87}     & \multicolumn{1}{l|}{\reb{48.56\tiny\(\pm\)5.42}}     & \reb{59.52\tiny\(\pm\)3.27}     & \multicolumn{1}{l|}{\reb{54.35\tiny\(\pm\)11.85}}      & \reb{66.82\tiny\(\pm\)2.03}         & \multicolumn{1}{l|}{\reb{50.62\tiny\(\pm\)4.85}} & \reb{+0.09} & \reb{+3.72} \\

\multicolumn{1}{l|}{\reb{\textbf{GAT+CIGAv2}}}                    & \reb{90.92\tiny\(\pm\)1.43}     & \multicolumn{1}{l|}{\reb{72.08\tiny\(\pm\)7.09}}     & \reb{66.57\tiny\(\pm\)4.91}    & \multicolumn{1}{l|}{\reb{40.93\tiny\(\pm\)19.35}}     & \reb{74.68\tiny\(\pm\)7.54}     & \multicolumn{1}{l|}{\reb{60.88\tiny\(\pm\)3.48}}     & \reb{56.88\tiny\(\pm\)24.40}     & \multicolumn{1}{l|}{\reb{54.28\tiny\(\pm\)22.63}}      & \reb{66.78\tiny\(\pm\)3.20}         & \multicolumn{1}{l|}{\reb{52.62\tiny\(\pm\)7.98}} & \reb{-0.01} & \reb{+7.73} \\

\multicolumn{1}{l|}{\textbf{GAT+SIA}}                & 91.88\tiny\(\pm\)2.12     & \multicolumn{1}{l|}{74.87\tiny\(\pm\)5.62}     & 69.64\tiny\(\pm\)6.79     & \multicolumn{1}{l|}{43.87\tiny\(\pm\)7.98}     & 75.35\tiny\(\pm\)3.28     & \multicolumn{1}{l|}{62.71\tiny\(\pm\)4.98}     & 61.42\tiny\(\pm\)1.07     & \multicolumn{1}{l|}{55.73\tiny\(\pm\)12.98}    & 63.27\tiny\(\pm\)3.15         & \multicolumn{1}{l|}{45.97\tiny\(\pm\)7.74}  & \makecolor+1.14 & \makecolor+8.20  \\ \midrule
\multicolumn{1}{l|}{\textbf{GIN}}                    & 88.28\tiny\(\pm\)2.39     & \multicolumn{1}{l|}{66.67\tiny\(\pm\)5.55}     & 57.02\tiny\(\pm\)6.48     & \multicolumn{1}{l|}{22.97\tiny\(\pm\)10.26}     & 74.55\tiny\(\pm\)4.27     & \multicolumn{1}{l|}{50.20\tiny\(\pm\)5.36}     & 62.17\tiny\(\pm\)3.86     & \multicolumn{1}{l|}{44.26\tiny\(\pm\)7.03}     & 62.42\tiny\(\pm\)2.77         & \multicolumn{1}{l|}{33.23\tiny\(\pm\)6.77}  & - & -  \\

\multicolumn{1}{l|}{\reb{\textbf{GIN+RPGNN}}}                    & \reb{89.59\tiny\(\pm\)1.33}     & \multicolumn{1}{l|}{\reb{69.23\tiny\(\pm\)8.05}}     & \reb{57.23\tiny\(\pm\)7.07}    & \multicolumn{1}{l|}{\reb{16.28\tiny\(\pm\)9.31}}     & \reb{71.63\tiny\(\pm\)5.85}     & \multicolumn{1}{l|}{\reb{45.11\tiny\(\pm\)15.98}}     & \reb{61.59\tiny\(\pm\)4.12}     & \multicolumn{1}{l|}{\reb{48.86\tiny\(\pm\)5.88}}      & \reb{62.27\tiny\(\pm\)2.33}         & \multicolumn{1}{l|}{\reb{44.04\tiny\(\pm\)7.06}} & \reb{-0.43} & \reb{+1.24} \\

\multicolumn{1}{l|}{\reb{\textbf{GIN+SSR}}}    & \reb{89.00\tiny\(\pm\)1.77}     & \multicolumn{1}{l|}{\reb{68.84\tiny\(\pm\)6.01}}     & \reb{59.83\tiny\(\pm\)2.34}     & \multicolumn{1}{l|}{\reb{21.28\tiny\(\pm\)19.26}}    & \reb{72.46\tiny\(\pm\)2.86}     & \multicolumn{1}{l|}{\reb{55.46\tiny\(\pm\)15.95}}     & \reb{62.54\tiny\(\pm\)1.30}     & \multicolumn{1}{l|}{\reb{48.73\tiny\(\pm\)8.62}}    & \reb{61.05\tiny\(\pm\)3.37}         & \multicolumn{1}{l|}{\reb{35.63 \tiny\(\pm\)7.29}}   & \reb{+0.09} & \reb{+2.52} \\

\multicolumn{1}{l|}{\reb{\textbf{GIN+CIGAv1}}}                    & \reb{91.01\tiny\(\pm\)1.59}     & \multicolumn{1}{l|}{\reb{74.85\tiny\(\pm\)10.41}}     & \reb{60.58\tiny\(\pm\)7.15}    & \multicolumn{1}{l|}{\reb{23.78\tiny\(\pm\)17.58}}     & \reb{75.32\tiny\(\pm\)16.25}     & \multicolumn{1}{l|}{\reb{54.83\tiny\(\pm\)9.87}}     & \reb{61.94\tiny\(\pm\)1.08}     & \multicolumn{1}{l|}{\reb{45.85\tiny\(\pm\)5.83}}      & \reb{61.55\tiny\(\pm\)12.55}         & \multicolumn{1}{l|}{\reb{35.88\tiny\(\pm\)7.77}} & \reb{+1.19} & \reb{+3.57} \\

\multicolumn{1}{l|}{\reb{\textbf{GIN+CIGAv2}}}                    & \reb{90.66\tiny\(\pm\)1.72}     & \multicolumn{1}{l|}{\reb{75.80\tiny\(\pm\)14.30}}     & \reb{63.02\tiny\(\pm\)1.80}    & \multicolumn{1}{l|}{\reb{22.42\tiny\(\pm\)12.36}}     & \reb{73.25\tiny\(\pm\)5.42}     & \multicolumn{1}{l|}{\reb{53.35\tiny\(\pm\)8.76}}     & \reb{64.42\tiny\(\pm\)5.35}     & \multicolumn{1}{l|}{\reb{45.37\tiny\(\pm\)5.12}}      & \reb{59.52\tiny\(\pm\)4.68}         & \multicolumn{1}{l|}{\reb{38.42\tiny\(\pm\)5.68}} & \reb{+1.29} & \reb{+3.61} \\

\multicolumn{1}{l|}{\textbf{GIN+SIA}}                & 92.70\tiny\(\pm\)0.45     & \multicolumn{1}{l|}{75.99\tiny\(\pm\)4.74}     & 61.30\tiny\(\pm\)6.77     & \multicolumn{1}{l|}{24.42\tiny\(\pm\)16.37}    & 74.88\tiny\(\pm\)4.24     & \multicolumn{1}{l|}{51.36\tiny\(\pm\)7.76}     & 62.83\tiny\(\pm\)1.07     & \multicolumn{1}{l|}{42.82\tiny\(\pm\)8.92}     & 63.00\tiny\(\pm\)4.24          & \multicolumn{1}{l|}{41.65\tiny\(\pm\)4.19}  & \makecolor+2.05  & \makecolor+3.78  \\ \midrule
\multicolumn{1}{l|}{\textbf{FAGCN}}                  & 90.58\tiny\(\pm\)1.72     & \multicolumn{1}{l|}{64.93\tiny\(\pm\)7.62}     & 62.96\tiny\(\pm\)2.12     & \multicolumn{1}{l|}{24.65\tiny\(\pm\)11.71}    & 70.03\tiny\(\pm\)5.20      & \multicolumn{1}{l|}{42.34\tiny\(\pm\)6.61}     & 43.51\tiny\(\pm\)4.29     & \multicolumn{1}{l|}{10.16\tiny\(\pm\)7.80}     & 55.78\tiny\(\pm\)3.5          & \multicolumn{1}{l|}{22.65\tiny\(\pm\)12.87}  & - & - \\

\multicolumn{1}{l|}{\reb{\textbf{FAGCN+RPGNN}}}                    & \reb{90.43\tiny\(\pm\)2.58}     & \multicolumn{1}{l|}{\reb{69.58\tiny\(\pm\)11.71}}     & \reb{61.00\tiny\(\pm\)51.91}    & \multicolumn{1}{l|}{\reb{20.94\tiny\(\pm\)12.62}}     & \reb{68.71\tiny\(\pm\)3.58}     & \multicolumn{1}{l|}{\reb{43.58\tiny\(\pm\)12.21}}     & \reb{44.55\tiny\(\pm\)5.82}     & \multicolumn{1}{l|}{\reb{12.22\tiny\(\pm\)5.95}}      & \reb{57.03\tiny\(\pm\)1.08}         & \multicolumn{1}{l|}{\reb{21.86\tiny\(\pm\)13.32}} & \reb{-0.23} & \reb{+0.69} \\

\multicolumn{1}{l|}{\reb{\textbf{FAGCN+SSR}}}    & \reb{88.95\tiny\(\pm\)2.16}     & \multicolumn{1}{l|}{\reb{70.12\tiny\(\pm\)9.49}}     & \reb{64.12\tiny\(\pm\)2.57}     & \multicolumn{1}{l|}{\reb{22.37\tiny\(\pm\)7.80}}    & \reb{67.92\tiny\(\pm\)3.00}     & \multicolumn{1}{l|}{\reb{42.27\tiny\(\pm\)8.69}}     & \reb{47.47\tiny\(\pm\)2.77}     & \multicolumn{1}{l|}{\reb{14.69\tiny\(\pm\)8.22}}    & \reb{55.66\tiny\(\pm\)3.37}         & \multicolumn{1}{l|}{\reb{22.35\tiny\(\pm\)9.35}}   & \reb{+0.25} & \reb{+1.41} \\

\multicolumn{1}{l|}{\reb{\textbf{FAGCN+CIGAv1}}}                    & \reb{91.08\tiny\(\pm\)1.48}     & \multicolumn{1}{l|}{\reb{69.29\tiny\(\pm\)3.35}}     & \reb{62.66\tiny\(\pm\)5.72}    & \multicolumn{1}{l|}{\reb{28.51\tiny\(\pm\)8.64}}     & \reb{68.58\tiny\(\pm\)6.19}     & \multicolumn{1}{l|}{\reb{57.79\tiny\(\pm\)10.45}}     & \reb{40.93\tiny\(\pm\)10.69}     & \multicolumn{1}{l|}{\reb{9.45\tiny\(\pm\)11.16}}      & \reb{48.07\tiny\(\pm\)5.70}         & \multicolumn{1}{l|}{\reb{15.11\tiny\(\pm\)5.82}} & \reb{+1.19} & \reb{+3.57} \\

\multicolumn{1}{l|}{\reb{\textbf{FAGCN+CIGAv2}}}                    & \reb{92.08\tiny\(\pm\)1.70}     & \multicolumn{1}{l|}{\reb{68.37\tiny\(\pm\)9.67}}     & \reb{60.37\tiny\(\pm\)4.65}    & \multicolumn{1}{l|}{\reb{22.29\tiny\(\pm\)15.00}}     & \reb{69.45\tiny\(\pm\)4.97}     & \multicolumn{1}{l|}{\reb{60.28\tiny\(\pm\)12.24}}     & \reb{44.27\tiny\(\pm\)6.74}     & \multicolumn{1}{l|}{\reb{15.88\tiny\(\pm\)10.25}}      & \reb{50.25\tiny\(\pm\)4.44}         & \multicolumn{1}{l|}{\reb{16.22\tiny\(\pm\)7.74}} & \reb{-1.29} & \reb{+3.66} \\

\multicolumn{1}{l|}{\textbf{FAGCN+SIA}}              & 90.17\tiny\(\pm\)2.83     & \multicolumn{1}{l|}{74.65\tiny\(\pm\)9.13}     & 62.40\tiny\(\pm\)3.36     & \multicolumn{1}{l|}{30.35\tiny\(\pm\)13.48}    & 71.30\tiny\(\pm\)5.79     & \multicolumn{1}{l|}{48.94\tiny\(\pm\)10.62}    & 46.95\tiny\(\pm\)5.71     & \multicolumn{1}{l|}{10.99\tiny\(\pm\)7.50}       & 52.82\tiny\(\pm\)6.28         & \multicolumn{1}{l|}{19.08\tiny\(\pm\)5.32} & \makecolor+0.16 & \makecolor+3.86   \\ \midrule
\multicolumn{1}{l|}{\textbf{GNNML3}}                 & 92.01\tiny\(\pm\)1.56     & \multicolumn{1}{l|}{64.18\tiny\(\pm\)6.99}     & 62.31\tiny\(\pm\)4.90     & \multicolumn{1}{l|}{32.94\tiny\(\pm\)12.86}    & 71.59\tiny\(\pm\)3.5      & \multicolumn{1}{l|}{40.74\tiny\(\pm\)15.0}     & 63.73\tiny\(\pm\)4.67     & \multicolumn{1}{l|}{51.75\tiny\(\pm\)9.05}     & 59.39\tiny\(\pm\)3.76         & \multicolumn{1}{l|}{33.80\tiny\(\pm\)11.19} & - & -  \\

\multicolumn{1}{l|}{\reb{\textbf{GNNML3+RPGNN}}}                    & \reb{92.57\tiny\(\pm\)1.45}     & \multicolumn{1}{l|}{\reb{72.30\tiny\(\pm\)9.54}}     & \reb{61.85\tiny\(\pm\)3.67}    & \multicolumn{1}{l|}{\reb{27.54\tiny\(\pm\)12.03}}     & \reb{70.48\tiny\(\pm\)2.46}     & \multicolumn{1}{l|}{\reb{38.61\tiny\(\pm\)14.61}}     & \reb{64.60\tiny\(\pm\)2.14}     & \multicolumn{1}{l|}{\reb{50.25\tiny\(\pm\)8.65}}      & \reb{60.22\tiny\(\pm\)2.40}         & \multicolumn{1}{l|}{\reb{36.88\tiny\(\pm\)15.93}} & \reb{+0.14} & \reb{+0.43} \\

\multicolumn{1}{l|}{\reb{\textbf{GNNML3+SSR}}}    & \reb{91.86\tiny\(\pm\)1.30}     & \multicolumn{1}{l|}{\reb{69.96\tiny\(\pm\)5.50}}     & \reb{64.95\tiny\(\pm\)2.82}     & \multicolumn{1}{l|}{\reb{26.56\tiny\(\pm\)4.68}}    & \reb{74.33\tiny\(\pm\)1.85}     & \multicolumn{1}{l|}{\reb{49.02\tiny\(\pm\)7.42}}     & \reb{63.19\tiny\(\pm\)4.35}     & \multicolumn{1}{l|}{\reb{54.30\tiny\(\pm\)10.33}}    & \reb{63.27\tiny\(\pm\)2.74}         & \multicolumn{1}{l|}{\reb{45.74\tiny\(\pm\)11.47}}   & \makecolor\reb{+1.41} & \reb{+4.43}\\

\multicolumn{1}{l|}{\reb{\textbf{GNNML3+CIGAv1}}}                    & \reb{91.25\tiny\(\pm\)4.67}     & \multicolumn{1}{l|}{\reb{72.80\tiny\(\pm\)5.88}}     & \reb{64.23\tiny\(\pm\)2.10}    & \multicolumn{1}{l|}{\reb{34.95\tiny\(\pm\)15.92}}     & \reb{73.89\tiny\(\pm\)4.73}     & \multicolumn{1}{l|}{\reb{52.01\tiny\(\pm\)12.59}}     & \reb{61.89\tiny\(\pm\)4.53}     & \multicolumn{1}{l|}{\reb{45.25\tiny\(\pm\)7.80}}      & \reb{60.95\tiny\(\pm\)3.44}         & \multicolumn{1}{l|}{\reb{38.80\tiny\(\pm\)12.36}} & \reb{+0.64} & \reb{+4.08} \\

\multicolumn{1}{l|}{\reb{\textbf{GNNML3+CIGAv2}}}                    & \reb{89.61\tiny\(\pm\)1.46}     & \multicolumn{1}{l|}{\reb{67.17\tiny\(\pm\)5.94}}     & \reb{64.19\tiny\(\pm\)5.36}    & \multicolumn{1}{l|}{\reb{27.28\tiny\(\pm\)12.79}}     & \reb{75.07\tiny\(\pm\)2.64}     & \multicolumn{1}{l|}{\reb{54.50\tiny\(\pm\)9.29}}     & \reb{60.28\tiny\(\pm\)4.56}     & \multicolumn{1}{l|}{\reb{46.25\tiny\(\pm\)10.42}}      & \reb{60.60\tiny\(\pm\)1.16}         & \multicolumn{1}{l|}{\reb{42.19\tiny\(\pm\)12.61}} & \reb{+0.14} & \reb{+2.80} \\

\multicolumn{1}{l|}{\textbf{GNNML3+SIA}}             & 92.70\tiny\(\pm\)0.81     & \multicolumn{1}{l|}{70.43\tiny\(\pm\)6.36}     & 64.57\tiny\(\pm\)2.72     & \multicolumn{1}{l|}{37.73\tiny\(\pm\)7.68}     & 69.32\tiny\(\pm\)3.79     & \multicolumn{1}{l|}{48.94\tiny\(\pm\)10.62}    & 63.91\tiny\(\pm\)5.81     & \multicolumn{1}{l|}{48.85\tiny\(\pm\)12.11}    & 61.58\tiny\(\pm\)3.98         & \multicolumn{1}{l|}{49.70\tiny\(\pm\)17.85}  & +0.61 & \makecolor+6.45  \\

\midrule

\multicolumn{1}{l|}{\reb{\textbf{SMP}}}    & \reb{92.30\tiny\(\pm\)2.62}     & \multicolumn{1}{l|}{\reb{80.25\tiny\(\pm\)5.98}}     & \reb{61.32\tiny\(\pm\)5.69}     & \multicolumn{1}{l|}{\reb{28.71\tiny\(\pm\)6.55}}    & \reb{76.87\tiny\(\pm\)1.90}     & \multicolumn{1}{l|}{\reb{45.69\tiny\(\pm\)15.96}}     & \reb{51.09\tiny\(\pm\)6.39}     & \multicolumn{1}{l|}{\reb{22.98\tiny\(\pm\)16.26}}    & \reb{49.15\tiny\(\pm\)6.92}         & \multicolumn{1}{l|}{\reb{30.73\tiny\(\pm\)11.30}}   & \reb{-} & \reb{-} \\ 
 \bottomrule
\end{tabular}
}
\end{table*}

\subsection{(RQ1) Is Cycle Information Helpful?}
\label{subsec:exp_rq1}

In this section, we aim to evaluate whether cycle information in {our} three {proposed} strategies (\sia, \ssl, and \ac) helps enhance the size generalizability of GNNs. 

\vspace{0.15cm}
\noindent \textbf{Setup.} {We evaluate the size generalizability of GNNs by training them on small graphs and testing their graph classification performance on large graphs. Better size generalizability translates into better performance on large graphs.}
\Cref{tab:cycle_aware_table} showcases the size generalizability results for our three proposed cycle-aware strategies, which we evaluate across five distinct datasets and  six different backbone models. 
Notably, to better compare different strategies, the last column gives the average improvements compared with the original model evaluated across all datasets.

\vspace{0.15cm}
\noindent 
{\textbf{Results.} First, all of the cycle-aware strategies consistently lead to improvements in large test datasets without
sacrificing the performance on small graphs. On average, these enhancements can reach up to 8.4\%, affirming the effectiveness of cycle information in improving GNN size generalizability, as discussed in \Cref{sec: analysis}. 
Second, it is worth noting that cycle lengths provide more valuable information for enhancing size generalizability of GNNs. This is evident from the consistently better performance of the strategies \ac and \sia  compared to the \ssl strategy on large test graphs, which solely predicts whether a node belongs to a cycle in the auxiliary task. 
Third, the simple attention-based strategy, \sia, achieves the best overall performance improvements in all scenarios. 
On average, the \sia strategy enhances both in-distribution and out-of-distribution generalization, while \ssl and \ac excel particularly in out-of-distribution generalization. Additionally, \sia achieves the highest average improvements on large test graphs. 
We attribute this to the challenges GNNs face in effectively learning cycle information, as shown in recent literature ~\cite{chen2020can}.} \reb{Also, in \Cref{appen:ablation_attention}, we conduct an ablation study further demonstrating that the attention mechanism, without considering the cycle information, cannot improve the GNN size generalizability.}

\vspace{0.15cm}
\noindent 
\textbf{Discussion.} 
The three strategies examined in this study primarily highlight the importance of cycle information in achieving size generalization. However, we do not recommend the combined use of all three strategies. \Cref{tab:cycle_aware_table} demonstrates that these strategies differ in their effectiveness for in-distribution generalization, with \ac sometimes leading to performance degradation. Notably, \ssl can be considered a subset of \sia, as \sia incorporates more comprehensive information, including both cycle length and is\_cycle, into its mechanism. Therefore, we recommend using \sia in practice due to its broader scope, while the other two strategies serve to illustrate the practical value of cycle information.

\begin{table*}[ht]
    \centering
    \caption{Performance of baseline and \sia across various model architectures in the synthetic dataset.}
    \vspace{-0.2cm}
    \label{tab:synthetic_performance}
    \renewcommand{\arraystretch}{1.2}
    \resizebox{\linewidth}{!}{ 
    \begin{tabular}{lcccccccccccc}
        \toprule
        Method & MLP-small & MLP-large & GCN-small & GCN-large & GAT-small & GAT-large & GIN-small & GIN-large & FAGCN-small & FAGCN-large & Avg-small & Avg-large \\
        \midrule
        Original & 63.61±4.50 & 59.72±8.50 & 64.75±1.56 & 32.61±8.36 & 65.48±2.17 & 55.14±5.64 & 67.01±1.60 & 44.16±6.12 & 65.32±1.86 & 61.45±1.88 & 65.23 & 50.62 \\
        SIA      & 63.82±4.30 & 62.56±8.48 & 72.88±3.67 & 69.61±9.38 & 64.52±4.56 & 63.72±6.33 & 74.13±4.23 & 67.56±8.67 & 71.80±3.42 & 69.07±3.11 & \makecolor69.43 & \makecolor66.50 \\
        \bottomrule
    \end{tabular}
    }
\end{table*}

\subsection{(RQ2) Comparison with Other Baselines}
\label{subsec:comparison_baseline}
{We now compare our best-performing strategy, \sia, with other approaches and present their respective graph classification performances in \Cref{tab:baseline_table}. Averaged improvements are presented in \Cref{tab:avg_results} in \Cref{appen:average_results}. \reb{We find that \sia consistently achieves the best performance compared with other baseline methods.}   
While \reb{\rpgin, \ssr, and \ciga} also enhance the size generalizability of GNNs, the improvements are less pronounced than those of \sia.  
Additionally, it's worth noting that \texttt{+CIGAv1} shows sensitivity to hyperparameters.
\reb{While not explicitly designed for size generalizability, the expressive model \smp demonstrates strong performance on the BBBP dataset due to its cycle detection ability, validating our empirical insights. However, though \smp excels in cycle detection, overall, \sia consistently outperforms \smp.} In addition, our proposed strategies result in minimal increases in training time compared to the baselines. We present the per-epoch runtime statistics in \Cref{tab:epoch_time_comparison} and the average results in \Cref{tab:avg_results}.

\subsection{(RQ3) Effectiveness on Large Dataset}

We evaluate the effectiveness of our method (\sia) on the HIV dataset from~\citet{hu2020open}. Using the ROC-AUC metric~\citep{hu2020open,chen2022learning}, we find that \sia consistently improves performance across all backbone models, as shown in \Cref{tab:large_benchmark}. On average, +SIA yields a +0.15 improvement on small test sets and a +3.63 improvement on large test sets, demonstrating its efficacy on real benchmarks.

\begin{table}[t]
    \centering
    \renewcommand{\arraystretch}{1.2}
        \caption{Performance comparison of different models with and without SIA on large dataset. -S and -L denote small and large test set performance, respectively.}
    \label{tab:large_benchmark}

            \resizebox{\linewidth}{!}{ 
    \begin{tabular}{lcccccc}
        \toprule
        Models & MLP & GCN & GAT & GIN & FAGCN & GNNML3 \\
        \midrule
        Original-S & 67.95±0.08 & 71.29±0.35 & 73.34±1.17 & 71.04±0.34 & 68.35±0.50 & 74.22±0.61 \\
        Original-L & 38.97±1.07 & 38.35±1.29 & 41.22±1.26 & 43.55±0.44 & 50.92±0.33 & 47.44±0.84 \\
        +SIA-S & 70.64±1.14 & 70.77±1.42 & 72.33±0.32 & 69.62±0.43 & 69.53±0.39 & 74.52±1.55 \\
        +SIA-L & 40.93±0.79 & 43.24±1.44 & 46.72±1.05 & 43.55±0.45 & 52.58±1.45 & 54.68±0.44 \\
        \bottomrule
    \end{tabular}
    }
\end{table}

%% file: sections/synthetic_exp.tex
\subsection{\textbf{(RQ4)} Verification on Synthetic Dataset}
\label{sec:additional_emprical}

Going beyond real datasets, we further evaluate various models on synthetic data (\Cref{subsec:setup}) and verify our findings in controlled settings. The results, shown in \Cref{tab:synthetic_performance}, highlight two key observations. First, the original model struggles to generalize across different graph sizes. Second, the effectiveness of \sia is even more pronounced in this controlled synthetic setting.

Additionally, we extend the aligning/breaking cycle analysis from \Cref{tab:eigen_difference_break_align} to the synthetic dataset. As shown in \Cref{tab:synthetic_cycle_comparison}, while the absolute spectral differences between graphs of different and similar sizes remain largely unchanged, the relative spectral differences exhibit clear variations. These observations reinforce our claims.
\begin{table}[t]
    \centering
    \caption{Effects of breaking and aligning cycles on synthetic datasets, evaluated using the same metrics as in \Cref{tab:eigen_difference_break_align}.}
    \label{tab:synthetic_cycle_comparison}
    \renewcommand{\arraystretch}{1.2}
        \resizebox{\linewidth}{!}{ 
    \begin{tabular}{lccc}
        \toprule
        & \textbf{Original} & \textbf{Breaking Cycles} & \textbf{Aligning Cycles} \\
        \midrule
        Different Size & 0.0051 & 0.0049 & 0.0050 \\
        Similar Size  & 0.0016 & 0.0011 & 0.0021 \\
        Relative Difference & 220\% & 356\% \textbf{(+136\%)} & 134\% \textbf{(-86\%)} \\
        \bottomrule
    \end{tabular}
    }
    
\end{table}

%% file: sections/related_work.tex
\label{sec:related work}

Literature on GNNs' size generalization presents conflicting views. On one hand, empirical studies highlight GNNs' size generalizability in physics simulation~\cite{sanchez2020learning} and algorithmic reasoning~\cite{xu2021neural}. 
Levie et al. \cite{levie2021transferability} theoretically showed that spectral GNNs robustly transfer between graphs with varied sizes when discretizing the same underlying space.

On the other hand, some works suggest that GNNs may require additional {information} to achieve size generalizability. For instance, Yan et al. \cite{yan2020neural} and Velivckovic et al. \cite{velivckovicneural} found that neural networks effectively generalize to larger graphs than those used in training when attention weights are properly supervised. In contrast, Yehudai et al. \cite{yehudai2021local} argued that performance degradation is linked to changes in degree patterns within multi-hop neighborhoods, known as `d-patterns'. However, our findings indicate that d-patterns do not adequately capture variations in cycle statistics, a critical factor for generalizing GNNs to larger biological graphs (\Cref{appen:d_pattern}). Additionally, our analysis shows that degree distribution does not correlate distinctly with graph size in biological datasets (\Cref{appen:deg_distri_fromLocal}).

In general, various methodologies have been proposed to address performance degradation from size shifts, including 
attention with proper thresholding~\cite{knyazev2019understanding}, stochastic anchoring~\cite{trivedi2024accurate}, 
regularization that improves model robustness to size shifts via graph coarsening~\cite{buffellisizeshiftreg},
a causal model to learn invariant representations for approximately better extrapolation between training and test data~\cite{bevilacqua2021size}, 
structural causal models for robust out-of-distribution generalization through invariant subgraph identification~\cite{chen2022learning},
a Wasserstein barycenter matching (WBM) layer to improve convergence rates with respect to graph size~\cite{chu2023wasserstein}, and a disentanglement learning framework to encourage size-invariant learning~\cite{huang2024enhancing}. 
Our study stands out as the first to utilize spectral analysis to characterize the types of size-induced distribution shifts, shedding light on the underlying causes that hinder GNNs from effectively generalizing to large graphs.

\reb{Some expressive models also show promise in addressing size generalization. For example, Murphy et al.~\cite{murphy2019relational} proposed a model-agnostic framework that learns graph representations invariant to graph isomorphism given variable-length inputs, while Clement et al.~\cite{vignac2020building} proposed an expressive GNN that performs well on difficult structural tasks, such as cycle detection and diameter computation. Our study validates that expressive models excelling in cycle-related tasks demonstrate good size generalizability.
} 
{Other related works include
studying the size OOD problem in the task of link prediction~\cite{zhou2022ood}, curating OOD datasets for AI-aided drug discovery~\cite{ji2022drugood}, and
improving uncertainty estimates in GNNs under various distribution shifts, including shifts in graph size~\cite{trivedi2024accurate}.}
 
\nocite{xu2024towards,xu2024slmrec,ning2024information,xu2025iagent}

%% file: sections/conclusion.tex
Our work extensively characterized size-induced distribution shifts and evaluated their impact on GNNs' generalizability to significantly larger test graphs compared to the training set. Through spectral analysis on real-world biological data, we uncovered a strong correlation between graph spectrum and size, which hinders GNNs' size generalization. We further identified the pivotal role of cycle-related information in reducing spectral differences between small and large graphs. Motivated by these findings, we introduced {three cycle-aware strategies} that can be integrated with any GNN architectures. Our empirical findings suggest that cycle information is crucial for size generalization and our proposed method significantly enhances the size generalizability of GNNs.
Overall, this work provides practical insights for enhacing GNN generalization across varying 
graph sizes. 
{An interesting future direction is to extend this work to domains beyond biology.}

%% file: sections/appen_proof.tex
\section{Proposition}
\label{app:proposition}

{In this section, we provide theoretical reasoning for the proposition in \Cref{subsec: GNN_spec_gen}.}

 \sloppy 
 Without loss of generality, the output of a spectral GNN at the $(l+1)$-th layer is given by: $\rmX^{(l+1)}=\sigma(\rmU f(\rmlambda)\rmU^{T}\rmX^{(l)}\rmW^{(l)})$ (Section~\ref{sec:prelim}). To facilitate our analysis, we express $\rmU$, $\rmlambda$ and $\rmX^{(l)}$ in terms of their column vectors: $\rmU=[\rmU_1, \rmU_2, \cdots, \rmU_N]$, $\rmlambda=\texttt{Diag}([\lambda_1, \lambda_2, \cdots, \lambda_N])$, and $\rmX^{(l)}=[\rmX^{(l)}_1, \rmX^{(l)}_2, \cdots, \rmX^{(l)}_{D}]$. Here, $\rmU_i$ is the $i$-th column vector of $\rmU$, $\lambda_i$ is the $i$-th largest eigenvalue, and $\rmX^{(l)}_i$ is the $i$-th column vector of $\rmX^{(l)}$. $\texttt{Diag}(\cdot)$ constructs a diagonal matrix with the given vector as its diagonal elements, and $D$ denotes the feature dimension at the $l$-th layer.
 
\begin{equation}
\resizebox{0.95\hsize}{!}{$
\begin{split}
\rmX^{(l+1)} &=\sigma([\rmU_1, \rmU_2, \cdots, \rmU_N]\cdot \texttt{Diag}([f(\lambda_1), f(\lambda_2), \cdots, f(\lambda_N)])\\&\cdot[\rmU_1, \rmU_2, \cdots, \rmU_N]^T \cdot [\rmX^{(l)}_1, \rmX^{(l)}_2, \cdots, \rmX^{(l)}_{D}]\cdot\rmW^{(l)})\\
&=\sigma([f(\lambda_1)\rmU_1, f(\lambda_2)\rmU_2, \cdots, f(\lambda_N)\rmU_N]\cdot([\rmU_1, \rmU_2, \cdots, \rmU_N])^T\\
&\cdot [\rmX^{(l)}_1, \rmX^{(l)}_2, \cdots, \rmX^{(l)}_{D}]\cdot\rmW^{(l)}).
\end{split}$
}
\label{eq:SDGNN_eq}
\end{equation}

Since $\{\rmU_i:i=1,\cdots,N\}$ form an orthonormal basis of $\sR^{N}$, $\rmX^{(l)}_i$ can be expressed as a linear combination of $\{\rmU_i\}$. Thus, supposing that
$\rmX^{(l)}_i=\sum_{j=1}^{N}{\alpha_j^{i}\rmU_j}$, Equation~\ref{eq:SDGNN_eq} can be rewritten as:

\begin{equation}
\resizebox{0.95\hsize}{!}{$
\begin{split}
\rmX^{(l+1)} &=\sigma([f(\lambda_1)\rmU_1, f(\lambda_2)\rmU_2, \cdots, f(\lambda_N)\rmU_N]\cdot([\rmU_1, \rmU_2, \cdots, \rmU_N])^T\\&\cdot [\sum_{j=1}^{N}{\alpha_j^{1}\rmU_j}, \cdots, \sum_{j=1}^{N}{\alpha_j^{i}\rmU_j}, \cdots]\cdot\rmW^{(l)}) \\
&=\sigma\left([f(\lambda_1)\rmU_1, f(\lambda_2)\rmU_2, \cdots, f(\lambda_N)\rmU_N]\cdot 
\begin{bmatrix}
    \alpha_1^1 & \alpha_1^2 & \cdots & \alpha_1^{D} \\
    \alpha_2^1 & \alpha_2^2 & \cdots & \alpha_2^{D} \\
    \cdots & \cdots & \cdots & \cdots \\
    \alpha_N^1 & \alpha_N^2 & \cdots & \alpha_N^{D}
\end{bmatrix}\cdot\rmW^{(l)}\right) \\
&=\sigma([\sum_{j=1}^{N}{f(\lambda_j)\alpha_j^1\rmU_j}, \sum_{j=1}^{N}{f(\lambda_j)\alpha_j^2\rmU_j}, \cdots, \sum_{j=1}^{N}{f(\lambda_j)\alpha_j^D\rmU_j}]\cdot\rmW^{(l)}).
\end{split}$
}
\label{eq:SDGNN_eq_2}
\end{equation}

We know that $\alpha_j^i=\rmU_j^T\cdot\rmX^{(l)}_i=\texttt{COSINE}(\rmU_j, \rmX^{(l)}_i)\cdot \lVert\rmX^{(l)}_i\rVert_2=\texttt{COSINE}(\rmU_j, \rmX^{(l)}_i)\cdot \frac{\lVert\rmX^{(l)}_i\rVert_2}{\sqrt{N}}\cdot \sqrt{N}$. 
Then Equation~\ref{eq:SDGNN_eq_2} can be rewritten as:

\begin{equation}
\resizebox{0.92\hsize}{!}{$
\begin{split}
\rmX^{(l+1)} &= \sigma([\sum_{j=1,\cdots,N}{f(\lambda_j)\texttt{COSINE}(\rmU_j, \rmX^{(l)}_1)\cdot \frac{\lVert\rmX^{(l)}_1\rVert_2}{\sqrt{N}}\cdot (\sqrt{N}\cdot\rmU_j)}, \\
&\cdots, \sum_{j=1,\cdots,N}{f(\lambda_j)\texttt{COSINE}(\rmU_j, \rmX^{(l)}_D)\cdot \frac{\lVert\rmX^{(l)}_D\rVert_2}{\sqrt{N}}\cdot (\sqrt{N}\cdot\rmU_j)}]\cdot\rmW^{(l)}).
\end{split}$
}
\end{equation}

Here, $\frac{\lVert\rmX^{(l)}_i\rVert_2}{\sqrt{N}} = C$ holds due to feature normalization~\cite{singh2022feature}.

%% file: sections/appen_datasets.tex
\section{Datasets}
\label{appen:dataset_details}

In this section, we provide additional details about the datasets utilized in our study. We utilize {five} pre-processed biological datasets, namely BBBP and BACE from the Open Graph Benchmark~\citep{hu2020open}, and PROTEINS, NCI1, and NCI109 from TuDataset~\citep{Morris+2020}, for graph classification in our experiments.
Each graph in BBBP and BACE represents a molecule, where nodes are atoms, and edges are chemical bonds.
Each node has a 9-dimensional feature vector, which contains its atomic number and chirality, as well as other additional atom features such as formal charge and whether the atom is in the ring or not~\citep{hu2020open}. In the PROTEINS dataset, nodes correspond to amino acids, and an edge connects two nodes if their distance is less than 6 Angstroms. Each node is associated with a 3-dimensional feature vector representing the type of secondary structure elements, such as helix, sheet, or turn. Each graph in NCI1 and NCI109 represents a chemical compound: each node stands for an atom and has a one-hot encoded feature representing the corresponding atom type; each edge represents the chemical bonds between the atoms.
The description of each dataset is summarized as follows:

\begin{itemize}
    \item \textbf{BBBP}: The blood-brain barrier penetration (BBBP) dataset comes from a study on the modeling and the prediction of barrier permeability. It includes binary labels of over 2000 compounds on their permeability properties~\citep{wu2018moleculenet}. 
    \item \textbf{BACE}: The BACE dataset provides quantitative (IC50) and qualitative (binary label) binding results for a set of inhibitors of human $\beta$-secretase 1 (BACE-1). It contains 1522 compounds with their 2D structures and binary labels~\citep{wu2018moleculenet}.
    \item \textbf{PROTEINS}: The PROTEINS dataset comprises the macromolecule graphs of proteins and binary labels of the protein function (being an enzyme or not) for a total of 1113 samples~\citep{Morris+2020}.
    \item \textbf{NCI1} \& \textbf{NCI109}: The NCI1 and NCI109 are two balanced subsets of chemical compounds screened for their activity against non-small cell lung cancer and ovarian cancer cell lines, respectively~\citep{4053093}. NCI1 contains a total of 4110 samples and NCI109 contains a total of 4117 samples.
\end{itemize}

%% file: sections/appen_data_preprocessing.tex
\section{Data Splits and Pre-processing}
\label{appen:data_preprocessing}

In this section, we discuss the techniques employed to address the challenges of class imbalance and label distribution shifts between small and large graphs. {Given our emphasis on size shift, we preprocess the dataset to mitigate the influence from other distribution shifts.} 

{For each dataset, \cref{tab:dataset_class_details} presents the respective class size in the whole dataset, the smallest 50\% subset, and the largest 10\% subset.} This table reveals two main issues. {Firstly, many datasets (BBBP, BACE, and PROTEINS) demonstrate imbalanced label distributions.} Secondly, both the smallest 50\% and the largest 10\% graphs exhibit class imbalance, and there is a notable difference in the label distribution between these subsets. To mitigate these issues, we carefully split our data and apply the upsampling technique.

\begin{table}[ht!]
\caption{Label distributions of BBBP, BACE, PROTEINS, NCI1 and NCI109 in the entire dataset, the smallest 50\% subset, and the largest 10\% subset.}
\label{tab:dataset_class_details}
\centering
\resizebox{\linewidth}{!}{ 
\begin{tabular}{@{}ccccccc@{}}
\toprule
\multicolumn{2}{c}{\textbf{Dataset}}                                                                       & \multicolumn{1}{l}{\textbf{BBBP}} & \multicolumn{1}{l}{\textbf{BACE}} & \multicolumn{1}{l}{\textbf{PROTEINS}} & \multicolumn{1}{l}{\textbf{NCI1}} & \multicolumn{1}{l}{\textbf{NCI109}} \\ \midrule
\multirow{2}{*}{Entire Dataset} & \begin{tabular}[c]{@{}c@{}}Number of \\ Class 0 Samples\end{tabular} & 479                               & 822                               & 663     & 2053 &      2048                       \\
                                    & \begin{tabular}[c]{@{}c@{}}Number of \\ Class 1 Samples\end{tabular} & 1560                              & 691                               & 450            & 2057 & 2079                       \\ \midrule
\multirow{2}{*}{Smallest 50\%}      & \begin{tabular}[c]{@{}c@{}}Number of\\ Class 0 Samples\end{tabular}  & 138                               & 501                               & 232   & 1283 & 1283                               \\
                                    & \begin{tabular}[c]{@{}c@{}}Number of\\ Class 1 Samples\end{tabular}  & 882                               & 256                               & 325         &    772 & 781                       \\ \midrule
\multirow{2}{*}{Largest 10\%}       & \begin{tabular}[c]{@{}c@{}}Number of\\ Class 0 Samples\end{tabular}  & 122                               & 53                                & 101               & 78 & 87                    \\
                                    & \begin{tabular}[c]{@{}c@{}}Number of \\ Class 1 Samples\end{tabular} & 82                                & 99                                & 11     &333 & 326                              \\ \bottomrule
\end{tabular}
}
\end{table}

\textbf{Data splits.} For each dataset, we create four distinct splits: train, validation, small\_test, and large\_test. The large\_test split consists of graphs with significantly larger sizes compared to the other splits.
To generate the train, validation, and small\_test subsets, we initially select the smallest 50\% of graphs from the dataset. The division among these subsets follows a ratio of 0.7:0.15:0.15, respectively. Importantly, we ensure that the data is split for each class using the same ratio, thereby maintaining consistent label distributions across the train, validation, and small\_test subsets. Next, we generate the large\_test subset by selecting the same number of graphs per class as in the small\_test subset. The selection process begins with the largest graph within each class and ensures that the large\_test subset maintains the same class distribution as the small\_test subset.
\cref{tab:dataset_statitics_after_balancing} shows the class size obtained after applying this operation. 

\textbf{Upsampling.} As can be seen in \cref{tab:dataset_class_details}, even after performing the appropriate data split, each dataset still exhibits varying degrees of class imbalance. To avoid training an extremely biased model, we adopt the upsampling technique during the training process. Specifically, we upsample the graphs in class 0 of BBBP at a ratio of 6, the graphs in class 1 of BACE at a ratio of 2, $\frac{2}{3}$ of the graphs in class 0 of PROTEINS at a ratio of 2, and $\frac{2}{3}$ of the graphs in class 1 of NCI1 and NCI109 at a ratio of 2. 

\begin{table}[ht!]
\centering
\caption{Label distributions after proper data splits. both the Both the small\_test and large\_test subsets now have the same class distribution.}
\label{tab:dataset_statitics_after_balancing}
\resizebox{\linewidth}{!}{ 
\begin{tabular}{@{}ccccccc@{}}
\toprule
\multicolumn{2}{c}{\textbf{Dataset}}                                                                & \multicolumn{1}{l}{\textbf{BBBP}} & \multicolumn{1}{l}{\textbf{BACE}} & \multicolumn{1}{l}{\textbf{PROTEINS}} & \multicolumn{1}{l}{\textbf{NCI1}} & \multicolumn{1}{l}{\textbf{NCI109}}\\ \midrule
\multirow{2}{*}{Train}       & \begin{tabular}[c]{@{}c@{}}Number of \\ Class 0 Samples\end{tabular} & 96                                & 350                               & 140  & 898    & 898                              \\
                             & \begin{tabular}[c]{@{}c@{}}Number of \\ Class 1 Samples\end{tabular} & 617                               & 179                               & 210     & 540            & 546                  \\ \midrule
\multirow{2}{*}{Val}         & \begin{tabular}[c]{@{}c@{}}Number of\\ Class 0 Samples\end{tabular}  & 22                                & 76                                & 30   & 193  & 193                                \\
                             & \begin{tabular}[c]{@{}c@{}}Number of\\ Class 1 Samples\end{tabular}  & 133                               & 39                                & 46                                  & 117 & 118 \\ \midrule
\multirow{2}{*}{Small\_test} & \begin{tabular}[c]{@{}c@{}}Number of\\ Class 0 Samples\end{tabular}  & 20                                & 75                                & 30        &192  & 192                          \\
                             & \begin{tabular}[c]{@{}c@{}}Number of \\ Class 1 Samples\end{tabular} & 132                               & 38                                & 45                                   & 115 & 117\\ \midrule
\multirow{2}{*}{Large\_test} & \begin{tabular}[c]{@{}c@{}}Number of\\ Class 0 Samples\end{tabular}  & 20                                & 75                                & 30           & 192   & 192                      \\
                             & \begin{tabular}[c]{@{}c@{}}Number of \\ Class 1 Samples\end{tabular} & 132                               & 38                                & 45                                   & 115 & 117\\ \bottomrule
\end{tabular}
}
\end{table}

%% file: sections/appen_training_details.tex
\section{Training Details}
\label{appen:training_details}

In this section, we provide the details of our training process for the purpose of reproducibility.

To ensure fair comparisons, we maintain consistent hyperparameter settings across all experiments. Specifically, we use a batch size of 30 and a learning rate of 0.001. The Adam optimizer is employed without gradient clipping. For the baseline models, we utilize three graph convolution layers with a global max pooling. The previously described upsampling technique is applied to all methods. {To prevent overfitting, we implement early stopping with a patience period of 50 epochs, and the lowest validation loss as the criterion for model selection.} Each method is executed five times, and we report the average result along with the standard deviation in \Cref{tab:cycle_aware_table} and \Cref{tab:baseline_table}. 

To ensure consistency with the original papers, we set the other hyperparameters of each baseline model to match the values specified in their respective papers. \reb{For \ssr, we use the same hyperparameter settings as in~\cite{buffellisizeshiftreg} and set the regularization coefficient $\lambda=0.1$ and coarsening ratios $C=\{0.8, 0.9\}$ for all datasets. We adhere to the authors' recommendation and employ spectral graph coarsening as the coarsening algorithm. Additionally, we fine-tune the aggregation strategy, selecting from options such as max, mean, and sum, to derive features for the nodes in the coarsened version. For \smp and \rpgin, we use the original code provided by the authors. For \smp, we tune the number of SMP and FastSMP layers from the set \{3,4\}. Based on the validation results, we implement 3 FastSMP layers. For \rpgin, we use the same hyperparameter settings as in \citep{bevilacqua2021size, buffellisizeshiftreg}. For \ciga, we follow the same hyperparameter search range as specified in~\cite{chen2022learning}.} 
Regarding the \ssl method that we proposed, we tune $\lambda$ over $\{1, 1e-1, 1e-2, 1e-3, 1e-4, 1e-5\}$. Based on the validation results, we identify $\lambda=1e-2$ for NCI1, and NCI109, $\lambda=1$ for PROTEINS, BBBP, and $\lambda=1e-2$ for BACE. For our proposed cycle augmentation, we use the set of hyperparameters $n, \mathcal{R}$ from \Cref{appen:break_align_cycles}. For \sia, we have no hyperparameters.

%% file: sections/appen_additional_experiment_table.tex
\section{Additional Tables}
\label{appen:average_results}

To facilitate the interpretation of the experimental results, we present the averaged improvements on both small and large graphs, evaluated across various model backbones and datasets, in \Cref{tab:avg_results}. Our findings suggest that incorporating cycle information enhances size generalization, with \sia demonstrating the most significant improvements. 
In addition, we show the per-epoch runtime in \Cref{tab:epoch_time_comparison}.

\begin{table}[h!]
    \centering
    \caption{Comparison of per-epoch time across different methods, measured in minutes.}
    \vspace{-0.2cm}
    \label{tab:epoch_time_comparison}
    \renewcommand{\arraystretch}{1.2}
    \resizebox{\linewidth}{!}{ 
    \begin{tabular}{lccccc}
        \toprule
        & \textbf{bbbp} & \textbf{bace} & \textbf{proteins} & \textbf{NCI1} & \textbf{NCI109} \\
        \midrule
        Original & 0.33 & 0.21 & 0.13 & 0.55 & 0.53 \\
        \ssr & 1.55 & 0.93 & 0.74 & 2.36 & 2.22 \\
        +CIGAv2 & 3.10 & 1.92 & 1.32 & 4.25 & 4.11 \\
        \ssl (Ours) & 0.33 & 0.21 & 0.12 & 0.57 & 0.54 \\
        \ac (Ours) & 0.34 & 0.21 & 0.13 & 0.54 & 0.53 \\
        \sia (Ours) & 1.32 & 0.78 & 0.55 & 1.98 & 1.94 \\
        \bottomrule
    \end{tabular}
    }
\end{table}

\begin{table}[t]
\caption{Averaged experimental results for model-agnostic methods evaluated across various model backbones and datasets. The results indicate that \sia achieves the most significant overall improvements.}
\label{tab:avg_results}
\resizebox{\linewidth}{!}{
\begin{tabular}{@{}l|ccccccc@{}}
\toprule
Method             & \multicolumn{1}{l}{SIA} & \multicolumn{1}{l}{SSL} & \multicolumn{1}{l}{AUG} & \multicolumn{1}{l}{GRPGNN} & \multicolumn{1}{l}{SSR} & \multicolumn{1}{l}{CIGAv1} & \multicolumn{1}{l}{CIGAv2} \\ \midrule
Avg Improv (Small) & \makecolor 1.09                    & 0.78                    & 0.07                    & 0.37                       & 0.73                    & 0.24                       & -0.05                      \\
Avg Improv (Large) &\makecolor 6.02                    & 2.55                    & 3.44                    & 1.15                       & 2.43                    & 3.24                       & 4.29                       \\ \bottomrule
\end{tabular}
}
\end{table}

%% file: sections/appen_attention.tex
\section{Ablation Studies on Attention Mechanism}
\label{appen:ablation_attention}

\begin{table*}[!ht]
\caption{Size generalizability evaluated with \texttt{+NA}, following the same rule as in~\Cref{tab:cycle_aware_table}. Naive attention is not helpful for size generalization.}

\label{tab:NA_table}
\resizebox{\linewidth}{!}{ 
\begin{tabular}{@{}lllllllllllll@{}}
\toprule
\multicolumn{1}{c}{\textbf{Datasets}}                & \multicolumn{2}{c}{\textbf{BBBP}}                    & \multicolumn{2}{c}{\textbf{BACE}}                    & \multicolumn{2}{c}{\textbf{PROTEINS}}                & \multicolumn{2}{c}{\textbf{NCI1}}                    & \multicolumn{2}{c}{\textbf{NCI109}} & \multicolumn{2}{c}{\textbf{Avg Improv}} \\ \midrule
\multicolumn{1}{l|}{\textbf{Models}}                 & \textbf{Small} & \multicolumn{1}{l|}{\textbf{Large}} & \textbf{Small} & \multicolumn{1}{l|}{\textbf{Large}} & \textbf{Small} & \multicolumn{1}{l|}{\textbf{Large}} & \textbf{Small} & \multicolumn{1}{l|}{\textbf{Large}} & \textbf{Small}     & \multicolumn{1}{l|}{\textbf{Large}} & \textbf{Small}        & \textbf{Large}         \\ \midrule
\multicolumn{1}{l|}{\textbf{MLP}}                    & 90.36\tiny\(\pm\)0.71     & \multicolumn{1}{l|}{55.61\tiny\(\pm\)3.37}     & 61.06\tiny\(\pm\)5.79     & \multicolumn{1}{l|}{21.06\tiny\(\pm\)7.89}     & 36.15\tiny\(\pm\)2.28     & \multicolumn{1}{l|}{21.55\tiny\(\pm\)1.34}     & 36.43\tiny\(\pm\)3.89     & \multicolumn{1}{l|}{3.36\tiny\(\pm\)2.87}      & 35.87\tiny\(\pm\)4.23         & \multicolumn{1}{l|}{4.65\tiny\(\pm\)3.72} & - & - \\
\multicolumn{1}{l|}{\textbf{MLP+NA}}                & 88.83\tiny\(\pm\)1.40      & \multicolumn{1}{l|}{62.09\tiny\(\pm\)10.27}     & 56.38\tiny\(\pm\)15.89     & \multicolumn{1}{l|}{21.34\tiny\(\pm\)8.52}    & 38.87\tiny\(\pm\)13.04       & \multicolumn{1}{l|}{29.42\tiny\(\pm\)23.83}     & 28.89\tiny\(\pm\)5.96     & \multicolumn{1}{l|}{1.98\tiny\(\pm\)0.98}      & 38.33\tiny\(\pm\)6.74          & \multicolumn{1}{l|}{1.55\tiny\(\pm\)0.01}  & -1.71 & +2.03  \\

\multicolumn{1}{l|}{\textbf{MLP+SIA}}                & 90.38\tiny\(\pm\)1.05     & \multicolumn{1}{l|}{62.79\tiny\(\pm\)7.55}     & 60.85\tiny\(\pm\)7.83     & \multicolumn{1}{l|}{21.79\tiny\(\pm\)15.07}      & 40.68\tiny\(\pm\)3.56     & \multicolumn{1}{l|}{33.57\tiny\(\pm\)11.87}    & 35.42\tiny\(\pm\)3.83     & \multicolumn{1}{l|}{3.26\tiny\(\pm\)2.55}      & 39.20\tiny\(\pm\)2.96          & \multicolumn{1}{l|}{12.2\tiny\(\pm\)5.05}    & \makecolor+1.33  & \makecolor+5.48 \\ \midrule
\multicolumn{1}{l|}{\textbf{GCN}}                    & 91.37\tiny\(\pm\)0.59     & \multicolumn{1}{l|}{68.59\tiny\(\pm\)7.47}     & 63.68\tiny\(\pm\)6.63     & \multicolumn{1}{l|}{28.72\tiny\(\pm\)14.26}     & 72.35\tiny\(\pm\)2.58     & \multicolumn{1}{l|}{40.57\tiny\(\pm\)7.67}     & 54.91\tiny\(\pm\)2.37     & \multicolumn{1}{l|}{28.80\tiny\(\pm\)7.57}     & 60.83\tiny\(\pm\)1.92         & \multicolumn{1}{l|}{30.45\tiny\(\pm\)4.34}   & - & - \\
\multicolumn{1}{l|}{\textbf{GCN+NA}}                & 89.45\tiny\(\pm\)1.65     & \multicolumn{1}{l|}{64.38\tiny\(\pm\)7.10}     & 62.03\tiny\(\pm\)1.12     & \multicolumn{1}{l|}{28.19\tiny\(\pm\)16.32}    & 69.95\tiny\(\pm\)3.27     & \multicolumn{1}{l|}{52.41\tiny\(\pm\)8.71}     & 60.37\tiny\(\pm\)1.40     & \multicolumn{1}{l|}{31.52\tiny\(\pm\)4.18}    & 54.97\tiny\(\pm\)2.57         & \multicolumn{1}{l|}{27.60\tiny\(\pm\)10.07} & -1.27 & -0.41  \\

\multicolumn{1}{l|}{\textbf{GCN+SIA}}                & 91.32\tiny\(\pm\)0.73     & \multicolumn{1}{l|}{71.66\tiny\(\pm\)6.99}     & 64.35\tiny\(\pm\)9.76     & \multicolumn{1}{l|}{24.24\tiny\(\pm\)17.03}     & 73.84\tiny\(\pm\)3.65     & \multicolumn{1}{l|}{58.74\tiny\(\pm\)9.49}     & 59.78\tiny\(\pm\)1.65     & \multicolumn{1}{l|}{45.70\tiny\(\pm\)6.70}     & 60.32\tiny\(\pm\)2.90         & \multicolumn{1}{l|}{38.78\tiny\(\pm\)4.55} & \makecolor+1.29 & \makecolor+8.40   \\ \midrule
\multicolumn{1}{l|}{\textbf{GAT}}                    & 91.27\tiny\(\pm\)1.43     & \multicolumn{1}{l|}{68.35\tiny\(\pm\)7.02}     & 69.73\tiny\(\pm\)2.05     & \multicolumn{1}{l|}{42.23\tiny\(\pm\)11.18}    & 72.25\tiny\(\pm\)4.25    & \multicolumn{1}{l|}{43.86\tiny\(\pm\)6.82}    & 58.22\tiny\(\pm\)2.86     & \multicolumn{1}{l|}{49.36\tiny\(\pm\)4.12}     & 64.39\tiny\(\pm\)3.29         & \multicolumn{1}{l|}{38.36\tiny\(\pm\)8.93}   &  - & - \\
\multicolumn{1}{l|}{\textbf{GAT+NA}}                & 91.16\tiny\(\pm\)1.45     & \multicolumn{1}{l|}{70.08\tiny\(\pm\)6.42}     & 63.42\tiny\(\pm\)2.64     & \multicolumn{1}{l|}{34.08\tiny\(\pm\)13.43}    & 72.25\tiny\(\pm\)2.24     & \multicolumn{1}{l|}{53.25\tiny\(\pm\)9.07}      & 60.98\tiny\(\pm\)0.93     & \multicolumn{1}{l|}{29.76\tiny\(\pm\)2.58}     & 60.18\tiny\(\pm\)5.91         & \multicolumn{1}{l|}{35.24\tiny\(\pm\)5.82}  & -1.57 & -3.95  \\
\multicolumn{1}{l|}{\textbf{GAT+SIA}}                & 91.88\tiny\(\pm\)2.12     & \multicolumn{1}{l|}{74.87\tiny\(\pm\)5.62}     & 69.64\tiny\(\pm\)6.79     & \multicolumn{1}{l|}{43.87\tiny\(\pm\)7.98}     & 75.35\tiny\(\pm\)3.28     & \multicolumn{1}{l|}{62.71\tiny\(\pm\)4.98}     & 61.42\tiny\(\pm\)1.07     & \multicolumn{1}{l|}{55.73\tiny\(\pm\)12.98}    & 63.27\tiny\(\pm\)3.15         & \multicolumn{1}{l|}{45.97\tiny\(\pm\)7.74}  & \makecolor+1.14 & \makecolor+8.20  \\ \midrule
\multicolumn{1}{l|}{\textbf{GIN}}                    & 88.28\tiny\(\pm\)2.39     & \multicolumn{1}{l|}{66.67\tiny\(\pm\)5.55}     & 57.02\tiny\(\pm\)6.48     & \multicolumn{1}{l|}{22.97\tiny\(\pm\)10.26}     & 74.55\tiny\(\pm\)4.27     & \multicolumn{1}{l|}{50.20\tiny\(\pm\)5.36}     & 62.17\tiny\(\pm\)3.86     & \multicolumn{1}{l|}{44.26\tiny\(\pm\)7.03}     & 62.42\tiny\(\pm\)2.77         & \multicolumn{1}{l|}{33.23\tiny\(\pm\)6.77}  & - & -  \\
\multicolumn{1}{l|}{\textbf{GIN+NA}}                & 89.59\tiny\(\pm\)2.64     & \multicolumn{1}{l|}{69.89\tiny\(\pm\)12.97}     & 55.46\tiny\(\pm\)10.59     & \multicolumn{1}{l|}{18.92\tiny\(\pm\)15.76}    & 73.59\tiny\(\pm\)4.06     & \multicolumn{1}{l|}{54.14\tiny\(\pm\)8.14}     & 61.15\tiny\(\pm\)4.62     & \multicolumn{1}{l|}{35.73\tiny\(\pm\)9.73}     & 61.62\tiny\(\pm\)3.08         & \multicolumn{1}{l|}{32.73\tiny\(\pm\)8.67} & -0.61 &  -1.18  \\

\multicolumn{1}{l|}{\textbf{GIN+SIA}}                & 92.70\tiny\(\pm\)0.45     & \multicolumn{1}{l|}{75.99\tiny\(\pm\)4.74}     & 61.30\tiny\(\pm\)6.77     & \multicolumn{1}{l|}{24.42\tiny\(\pm\)16.37}    & 74.88\tiny\(\pm\)4.24     & \multicolumn{1}{l|}{51.36\tiny\(\pm\)7.76}     & 62.83\tiny\(\pm\)1.07     & \multicolumn{1}{l|}{42.82\tiny\(\pm\)8.92}     & 63.00\tiny\(\pm\)4.24          & \multicolumn{1}{l|}{41.65\tiny\(\pm\)4.19}  & \makecolor+2.05  & \makecolor+3.78  \\ \midrule
\multicolumn{1}{l|}{\textbf{FAGCN}}                  & 90.58\tiny\(\pm\)1.72     & \multicolumn{1}{l|}{64.93\tiny\(\pm\)7.62}     & 62.96\tiny\(\pm\)2.12     & \multicolumn{1}{l|}{24.65\tiny\(\pm\)11.71}    & 70.03\tiny\(\pm\)5.20      & \multicolumn{1}{l|}{42.34\tiny\(\pm\)6.61}     & 43.51\tiny\(\pm\)4.29     & \multicolumn{1}{l|}{10.16\tiny\(\pm\)7.80}     & 55.78\tiny\(\pm\)3.5          & \multicolumn{1}{l|}{22.65\tiny\(\pm\)12.87}  & - & - \\
\multicolumn{1}{l|}{\textbf{FAGCN+NA}}              & 91.97\tiny\(\pm\)1.67     & \multicolumn{1}{l|}{72.30\tiny\(\pm\)10.54}     & 58.39\tiny\(\pm\)4.25     & \multicolumn{1}{l|}{10.71\tiny\(\pm\)9.0}    & 70.47\tiny\(\pm\)4.73     & \multicolumn{1}{l|}{51.61\tiny\(\pm\)12.47}     & 50.27\tiny\(\pm\)4.45     & \multicolumn{1}{l|}{12.11\tiny\(\pm\)5.39}      & 50.75\tiny\(\pm\)10.91         & \multicolumn{1}{l|}{10.26\tiny\(\pm\)16.24}  & -0.20 & -1.55   \\

\multicolumn{1}{l|}{\textbf{FAGCN+SIA}}              & 90.17\tiny\(\pm\)2.83     & \multicolumn{1}{l|}{74.65\tiny\(\pm\)9.13}     & 62.40\tiny\(\pm\)3.36     & \multicolumn{1}{l|}{30.35\tiny\(\pm\)13.48}    & 71.30\tiny\(\pm\)5.79     & \multicolumn{1}{l|}{48.94\tiny\(\pm\)10.62}    & 46.95\tiny\(\pm\)5.71     & \multicolumn{1}{l|}{10.99\tiny\(\pm\)7.50}       & 52.82\tiny\(\pm\)6.28         & \multicolumn{1}{l|}{19.08\tiny\(\pm\)5.32} & \makecolor+0.16 & \makecolor+3.86   \\ \midrule
\multicolumn{1}{l|}{\textbf{GNNML3}}                 & 92.01\tiny\(\pm\)1.56     & \multicolumn{1}{l|}{64.18\tiny\(\pm\)6.99}     & 62.31\tiny\(\pm\)4.90     & \multicolumn{1}{l|}{32.94\tiny\(\pm\)12.86}    & 71.59\tiny\(\pm\)3.5      & \multicolumn{1}{l|}{40.74\tiny\(\pm\)15.0}     & 63.73\tiny\(\pm\)4.67     & \multicolumn{1}{l|}{51.75\tiny\(\pm\)9.05}     & 59.39\tiny\(\pm\)3.76         & \multicolumn{1}{l|}{33.80\tiny\(\pm\)11.19} & - & -  \\
\multicolumn{1}{l|}{\textbf{GNNML3+NA}}             & 87.85\tiny\(\pm\)1.96     & \multicolumn{1}{l|}{66.42\tiny\(\pm\)1.93}     & 65.15\tiny\(\pm\)7.50     & \multicolumn{1}{l|}{42.05\tiny\(\pm\)13.30}     & 66.11\tiny\(\pm\)10.07     & \multicolumn{1}{l|}{55.39\tiny\(\pm\)13.23}     & 59.58\tiny\(\pm\)5.06     & \multicolumn{1}{l|}{32.48\tiny\(\pm\)8.90}    & 60.76\tiny\(\pm\)4.95         & \multicolumn{1}{l|}{27.18\tiny\(\pm\)8.65}   & -1.92 & +0.02 \\

\multicolumn{1}{l|}{\textbf{GNNML3+SIA}}             & 92.70\tiny\(\pm\)0.81     & \multicolumn{1}{l|}{70.43\tiny\(\pm\)6.36}     & 64.57\tiny\(\pm\)2.72     & \multicolumn{1}{l|}{37.73\tiny\(\pm\)7.68}     & 69.32\tiny\(\pm\)3.79     & \multicolumn{1}{l|}{48.94\tiny\(\pm\)10.62}    & 63.91\tiny\(\pm\)5.81     & \multicolumn{1}{l|}{48.85\tiny\(\pm\)12.11}    & 61.58\tiny\(\pm\)3.98         & \multicolumn{1}{l|}{49.70\tiny\(\pm\)17.85}  & \makecolor+0.61 & \makecolor+6.45  \\

\bottomrule
\end{tabular}
}
\end{table*}

\reb{We use \na to denote the use of naive attention without using cycle features. The results in \Cref{tab:NA_table} demonstrate that \na does not consistently enhance graph classification for large graphs. This suggests that relying solely on the attention mechanism, without considering cycle information, does not improve the size generalizability of GNNs. In contrast, the incorporation of cycle information in \sia leads to significant improvements on large graphs.}

%% file: sections/appen_cycle_align_break.tex
\section{Details and Discussions on Breaking Cycles and Aligning Cycle Lengths}
{In this section, we begin by elaborating on the procedures for breaking cycles and aligning cycle lengths, which we briefly introduced in \Cref{sec: analysis}. Then, we present the results of an ablation study on adding nodes randomly (instead of based on the underlying cycles).}

\subsection{Implementation Details}
\label{appen:break_align_cycles}
{Our \textbf{cycle-breaking algorithm} is designed to remove all cycles in the graph with minimal edge removal while maintaining the current number of disconnected components. This approach minimally perturbs the graph while focusing on cycles. To achieve this, we begin by obtaining the cycle basis $\mathcal{C}$ (\Cref{sec:prelim}) of the input graph $\mathcal{G}$.
Breaking all cycles can be accomplished by removing just one edge from each cycle in the basis. To ensure that this edge removal process does not increase the number of disconnected components,}
we propose a backtracking algorithm (~\Cref{alg:cycle_break}). 
{We note that while it is possible for \Cref{alg:cycle_break} to have failure cases, luckily, in all the datasets we have examined, we have not observed any instances where cycle breaking fails (i.e., Alg.~\ref{alg:cycle_break} returning False). }

\begin{algorithm}
  \caption{Cycle Breaking}
  \label{alg:cycle_break}
   
  \begin{algorithmic}[1]
    \Procedure{\texttt{CycleBacktracking}}{$i \colon \text{index integer}$, $\mathcal{G} \colon \text{input graph}$, $\mathcal{C} \colon \text{a list of cycle basis}$}

    \If{ $i == \text{len}(\mathcal{C})$ }
    \State \Return \text{True} 
    \EndIf

    \State $N_{\text{before}} \leftarrow \texttt{NumberOfConnectedComponent}(\mathcal{G})$
    \For { \text{edge pair} $\varepsilon$ in $\mathcal{C}_i$ }

    \State {  $\mathcal{G} \leftarrow \text{Remove} (\varepsilon)$    }
    \State {$N_{\text{after}} \leftarrow \texttt{NumberOfConnectedComponent}(\mathcal{G})$}
    \If {$N_{\text{before}} == N_{\text{after}}$}
    \If {\texttt{CycleBacktracking}($i+1, \mathcal{G}, \mathcal{C}$)} 
    \State \Return \text{True}
    \EndIf
    \EndIf
    \State{$ \mathcal{G} \leftarrow \text{add}(\varepsilon)$  }
    \EndFor 
    \State \Return \text{False}
    \EndProcedure
  \end{algorithmic}
  
\end{algorithm}

{Before delving into the cycle length alignment algorithm, we first present the average statistics for both the small and large graphs, as outlined in \Cref{tab:cycle_statistics}. To derive these statistics, we obtain the cycle basis  (\Cref{sec:prelim}) for each input graph and compute the average length across the cycles in that basis. Then, we calculate the average cycle length by further averaging it over all the graphs within the small and large datasets, respectively. Note that small and large graphs differ in both the mean and the standard deviation of cycle lengths, indicating that large graphs may contain notably long cycles that are not present in the small graphs.}

\begin{table}[H]
\caption{Average cycle lengths for small and large graphs before and after cycle length alignment. Small and large graphs in real biological datasets originally differ in cycle length and such difference is reduced after cycle length alignment.}

\label{tab:cycle_statistics}

\resizebox{\linewidth}{!}{
\begin{tabular}{@{}l|cc|cc|cc|cc|cc@{}}
\toprule
              & \multicolumn{2}{c|}{BBBP} & \multicolumn{2}{c|}{BACE} & \multicolumn{2}{c|}{NCI1} & \multicolumn{2}{c|}{NCI109} & \multicolumn{2}{c}{PROTEINS} \\ \midrule
              & small       & large       & small       & large       & small       & large       & small        & large        & small         & large        \\
before alignment  & 5.56\tiny\(\pm\)1.39           & 6.54\tiny\(\pm\)3.52           & 5.70\tiny\(\pm\)0.56   & 6.50\tiny\(\pm\)2.01   & 5.64\tiny\(\pm\)1.19   & 7.04\tiny\(\pm\)3.05   & 5.59\tiny\(\pm\)1.25    & 7.06\tiny\(\pm\)2.88    & 3.20\tiny\(\pm\)0.34    & 6.24\tiny\(\pm\)2.75    \\
after alignment  & 6.50\tiny\(\pm\)3.12           & 6.54\tiny\(\pm\)3.52           & 6.48\tiny\(\pm\)2.08   & 6.50\tiny\(\pm\)2.01   & 6.97\tiny\(\pm\)2.97   & 7.04\tiny\(\pm\)3.05   & 6.86\tiny\(\pm\)3.02    & 7.06\tiny\(\pm\)2.88    & 5.97\tiny\(\pm\)3.03    & 6.24\tiny\(\pm\)2.75    \\
\bottomrule
\end{tabular}
}
\end{table}

{Our \textbf{cycle length alignment algorithm} aligns both the \emph{mean} and \emph{std} of average cycle lengths between small and large graphs. This is achieved by selectively increasing the cycle length of a subset of small graphs.} 
Formally, we present our algorithm in~\Cref{alg:align_cycle_length}. It has two hyperparameters: $n$, which controls the {maximum cycle length increments} one graph can get, and $\mathcal{R}$, which controls the portion of small graphs whose cycle lengths will be increased. We tune these hyperparameters in order to best align the cycle length of small graphs with that of large graphs. We tune $n$ over $[2,3,4,5,6,7]$ and $\mathcal{R}$ over $[1,2,3,4,5,6,7,8]$, and identify 
 $n=7, \mathcal{R}=8$ for BBBP, $n=6, \mathcal{R}=8$ for BACE, $n=6, \mathcal{R}=5$ for NCI1, $n=6, \mathcal{R}=5$ for NCI109 and $n=5, \mathcal{R}=3$ for PROTEINS. We present the average cycle length after alignment in \Cref{tab:cycle_statistics} and the discrepancy between small and large graphs is significantly reduced.

\begin{algorithm}
  \caption{Add One Cycle Length}
  \label{alg:helper_cycle_length}
   
  \begin{algorithmic}[1]
    \Procedure{\texttt{Add1\_Cycle\_Length}}{$\mathcal{G} \colon \text{input graph}$, $\mathcal{C} \colon \text{list of cycle basis}$}

    \For{ \text{each} $ \text{cycle } \mathcal{C}_i  $ }

    \State {$\text{edge } (v_1, v_2) \leftarrow \texttt{RandomChoice}(\mathcal{C}_i)$}
    \State {$\mathcal{G} \leftarrow \text{RemoveEdge}(v_1, v_2)$}
    \State {$v_{\text{new}} \leftarrow \text{ ReplicateNode}(\argmin_{\text{degree}}(\mathcal{C}_i))$}
    \State {$\mathcal{G} \leftarrow \text{AddNode}(v_{\text{new}})$}
    \State {$\mathcal{G} \leftarrow \text{AddEdge}(v_1, v_{\text{new}}) \ \text{AddEdge}(v_2, v_{\text{new}})$}
    \EndFor

    \State {\Return $\mathcal{G}$}
   
    \EndProcedure

  \end{algorithmic}

\end{algorithm}

\begin{algorithm}
  \caption{Align Cycle Length}
  \label{alg:align_cycle_length}
   
  \begin{algorithmic}[1]
    \Procedure{\texttt{AlignCycleLength}}{$\{\mathcal{G}_i\}_{i=1}^N \colon \text{a set of graphs}$, $\mathcal{R} \colon \text{skipping ratio}$, $n \colon \text{increased cycle length}, n \geq 1$}
    \State $\text{ResultList } \mathcal{M} \leftarrow []$  
    \For{ $ i = 0, 1,\dots, N-1 $ }
    \State{$\mathcal{G'} \leftarrow \mathcal{G}_i$}
    \If {$i \mod \mathcal{R} == 0$} 
    
    \For{$j = 0, 1, \dots, n-1$}
    \State{$\text{obtain Cycle Basis }\mathcal{C} \leftarrow \mathcal{G'}$} 
    
    \State{$\mathcal{G'} \leftarrow \texttt{Add1\_Cycle\_Length}(\mathcal{G'},\mathcal{C})$}
    
    \EndFor 
    \EndIf
    \State $\mathcal{M} \leftarrow \text{append }\mathcal{G'}$ 
    \EndFor
    \State {\Return $\mathcal{M}$}
    \EndProcedure
  \end{algorithmic}
\end{algorithm}

\subsection{Ablation Study on Randomly Adding Nodes}
\label{app: random_nodes}
One natural question arising from \Cref{subsection: breaking_cycles} is whether the cycle length alignment is necessary or it suffices to add nodes / edges randomly in order to reduce the spectrum discrepancy between small and large graphs.
{To answer this question, we conduct an ablation study on adding nodes and edges randomly to the graphs in the small dataset, \textit{matching the quantity} introduced during the cycle length alignment process. We then compare the changes in the relative spectral difference with those observed during the cycle length alignment procedure.}

{We present our results in \Cref{tab:eigen_diff_random}. We find that aligning cycle lengths consistently leads to smaller spectral differences between graphs of varying sizes compared to randomly adding nodes. These results further highlight the importance of cycle lengths for GNNs to achieve size generalizability.}

\begin{table*}[h]
\centering
\caption{Comparison between randomly adding nodes and aligning cycle lengths. Average Wasserstein distance of eigenvalue distributions between graphs of similar size and different sizes are computed. Relative difference is computed as in \Cref{tab:eigen_dist}. We use $\uparrow$ ($\downarrow$) to denote the increase (decrease) in the relative difference compared to not taking the corresponding action. Aligning cycle lengths results in a greater reduction of the relative spectral difference compared to randomly adding a node.}
\label{tab:eigen_diff_random}
\small 
\resizebox{\linewidth}{!}{
\centering
\begin{tabular}{@{}lccc|ccc@{}}
\toprule

                                       & \multicolumn{3}{|c}{\textbf{Randonly adding nodes}}                                                                  & \multicolumn{3}{|c}{\textbf{Aligning cycle lengths}}                                                                \\ \midrule
\multicolumn{1}{l|}{\textbf{Datasets}} & \multicolumn{1}{l}{\textbf{Different sizes}} & \multicolumn{1}{l}{\textbf{Similar size}} & \multicolumn{1}{l|}{\textbf{$\Delta$ relative difference}} & \multicolumn{1}{l}{\textbf{Different sizes}} & \multicolumn{1}{l}{\textbf{Similar size}} & \multicolumn{1}{l}{\textbf{$\Delta$ relative difference}} \\ \midrule
\multicolumn{1}{l|}{\textbf{BBBP}}     & 0.00557                                      & 0.00203                                    & $\downarrow$ 33\%                                             & 0.00565                                      & 0.00211                                    & $\downarrow$ 41\%                                            \\
\multicolumn{1}{l|}{\textbf{BACE}}     & 0.00406                                      & 0.00162                                    &  $\downarrow$ 26\%                                             & 0.00417                                      & 0.00176                                    & $\downarrow$ 41\%                                            \\
\multicolumn{1}{l|}{\textbf{NCI1}}     & 0.00559                                      & 0.00230                                    & $\downarrow$ 20\%                                             & 0.00566                                      & 0.00242                                    & $\downarrow$ 31\%                                            \\
\multicolumn{1}{l|}{\textbf{NCI109}}   & 0.00558                                      & 0.00232                                    & $\downarrow$ 22\%                                             & 0.00568                                      & 0.00245                                    & $\downarrow$ 31\%                                            \\
\multicolumn{1}{l|}{\textbf{PROTEINS}} & 0.00756                                      & 0.00280                                    & $\downarrow$ 24\%                                             & 0.00763                                      & 0.00302                                    & $\downarrow$ 41\%                                            \\ \bottomrule
\end{tabular}
}
\end{table*}

%% file: sections/appen_additional_plots.tex
\section{Additional Plots on Cycle Breaking \& Cycle Length Alignment}

Following the same convention in \Cref{fig:cycle_purturbation_eigen}, 
\Cref{fig:additional_break_align} presents the pairwise graph distance measured by eigenvalue distributions after breaking cycles
and aligning cycles on other datasets.

%% file: sections/appen_additional_scaled_eigenvector_plot.tex
\section{Additional Plots on Scaled Eigenvector}
\label{appen:scaled_eigenvector}
Similar to \Cref{fig:eigenvalue_dsitribution}, \Cref{fig:scaled_eigen} presents the pairwise graph distance measured using scaled eigenvectors across the five datasets. We observe that the dependence of the scaled eigenvectors on graph size is mild.

\begin{figure*}[!t]
  \centering
  \begin{subfigure}{0.8\textwidth}
    \centering
    \includegraphics[width=\textwidth]{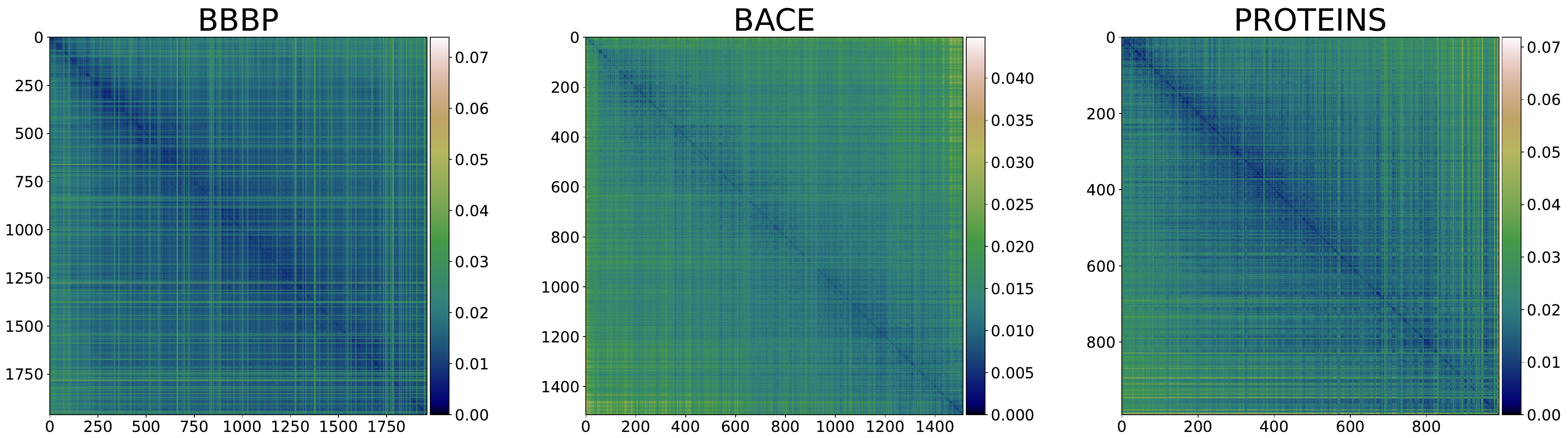}
  \end{subfigure}\\[5pt]
  
  \begin{subfigure}{0.7\textwidth}
    \centering
    \includegraphics[width=0.9\textwidth]{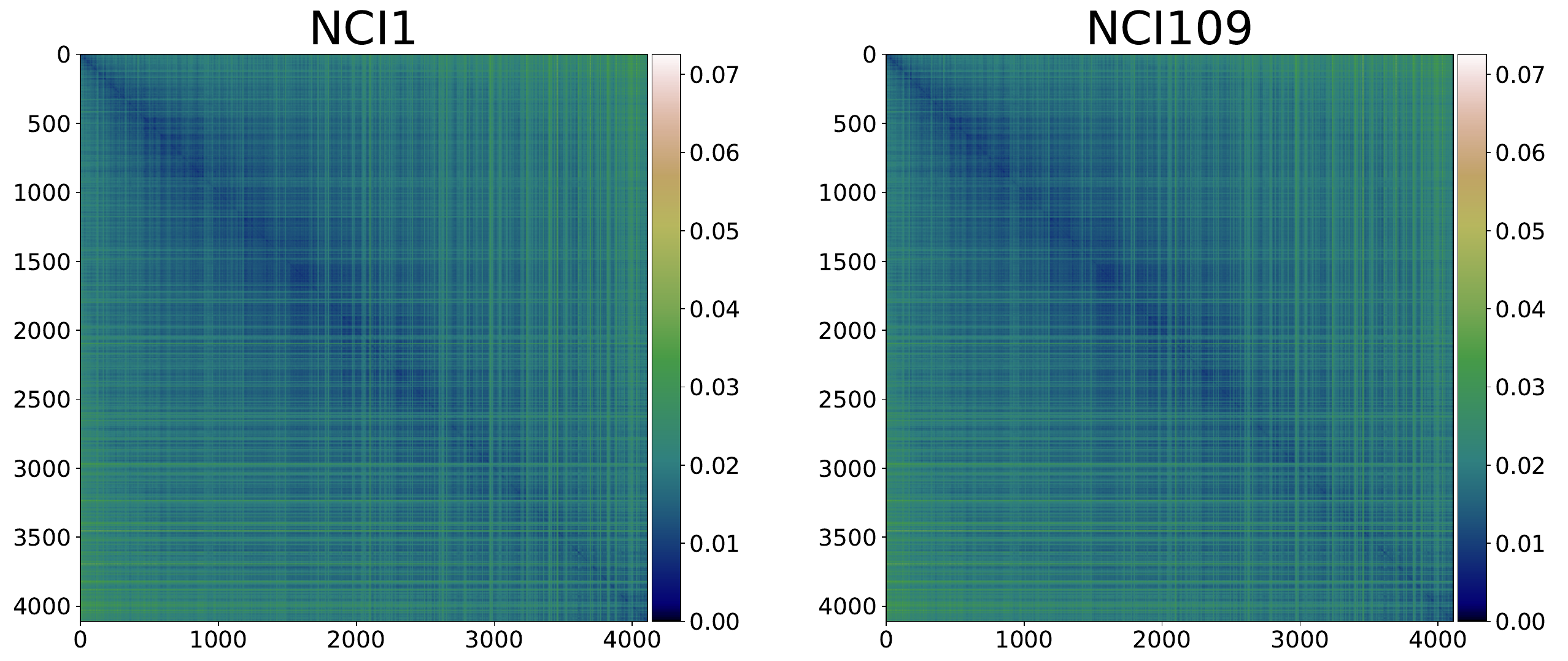}
  \end{subfigure}

  \caption{Pairwise graph distance measured by the Wasserstein distance of the scaled eigenvector distributions, evaluated using the same metrics as in~\Cref{fig:eigenvalue_dsitribution}. \textbf{The scaled eigenvector does not show a clear correlation with graph size}.}
  \label{fig:scaled_eigen}
  \vspace{-0.3cm}
\end{figure*}

%% file: sections/appen_additional_plot_degree.tex
\section{Additional Plots on Degree Distributions}
\label{appen:deg_distri_fromLocal}

\begin{figure*}[!t]
  \centering
  \begin{subfigure}{0.8\textwidth}
    \centering
    \includegraphics[width=\textwidth]{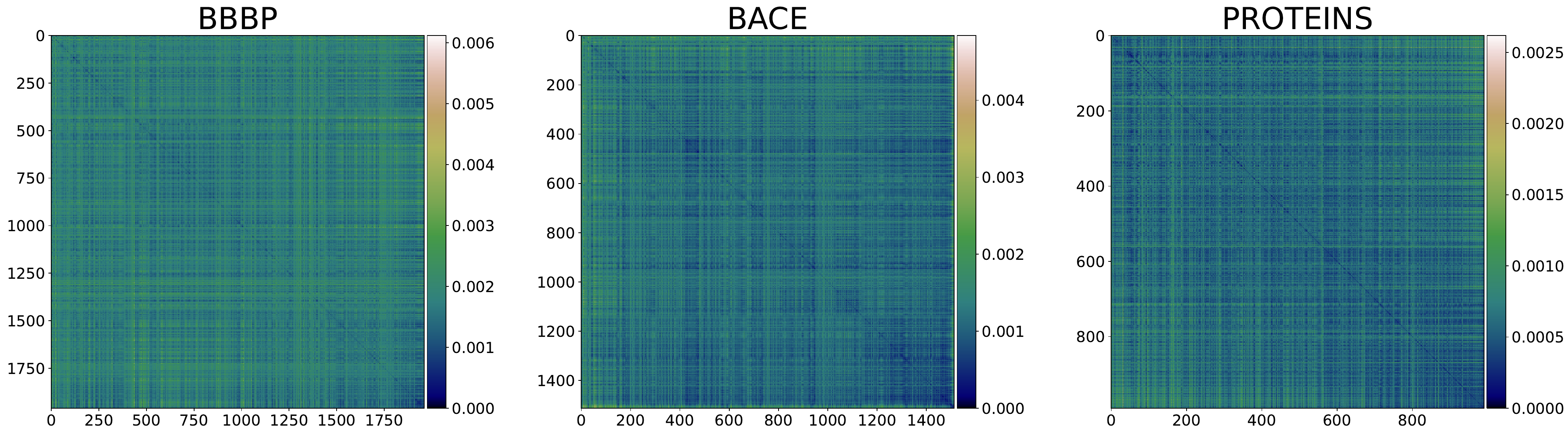}
    \label{fig:degree_distribution_first_three}
  \end{subfigure}\\[5pt]
  
  \begin{subfigure}{0.7\textwidth}
    \centering
    \includegraphics[width=0.9\textwidth]{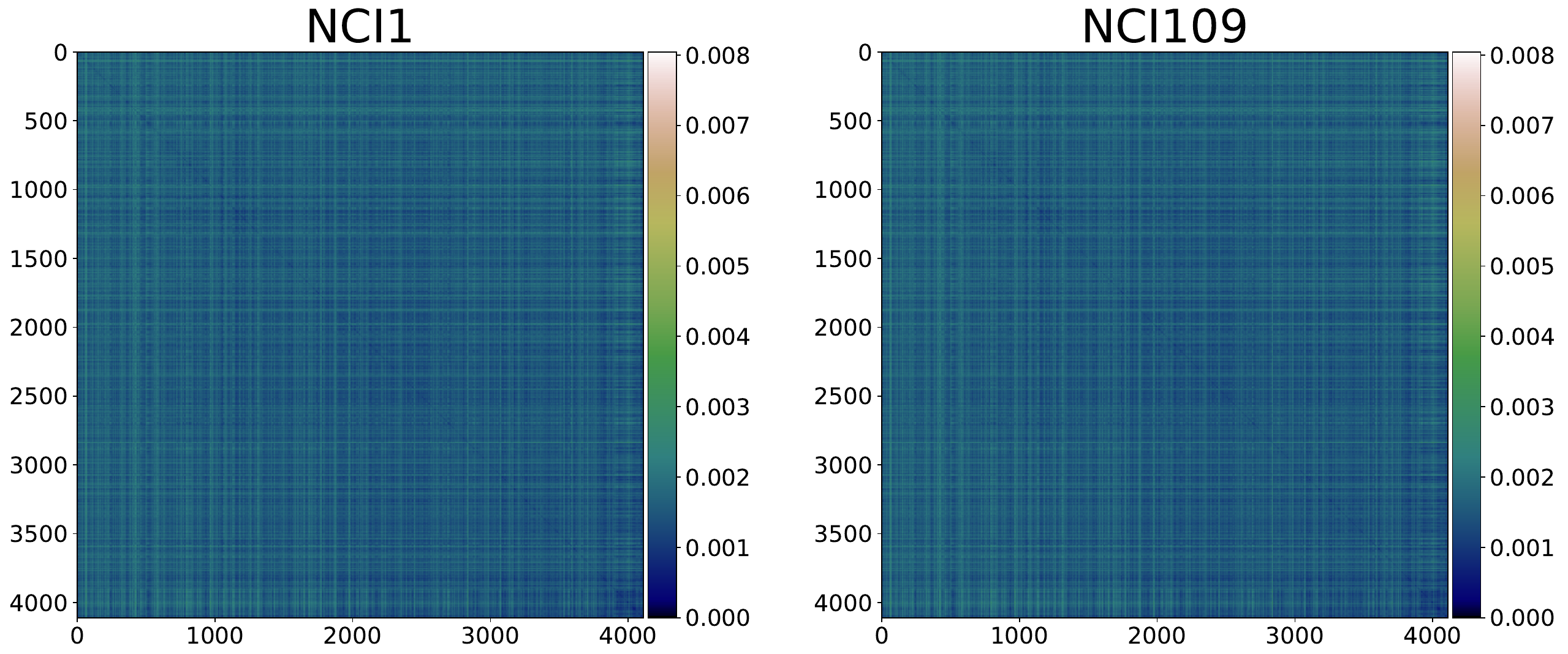}

  \label{fig:cycle_purturbation_eigen_second_three}
  \end{subfigure}

  \caption{\reb{Pairwise graph distance measured by the Wasserstein distance of degree distributions, following the same convention in~\Cref{fig:eigenvalue_dsitribution}. \textbf{Degree distribution does not show a clear correlation with graph size}.} }
  \label{fig:deg_distribution}
  \vspace{-0.3cm}
\end{figure*}

\Cref{fig:deg_distribution} 
reveal that the degree distribution does not exhibit a clear correlation with the graph size. These visualizations and tables adhere to the same convention as presented in Section 3.2. Unlike the assumption made in~\cite{yehudai2021local} that attributes changes in local degree patterns as the primary covariate shifts induced by size, our investigation on biological datasets suggests a different perspective.

%% file: sections/appen_d_pattern.tex
\section{Limitations of d-Patterns in Capturing Cycle Information}
\label{appen:d_pattern}

\begin{figure*}[t]
    \centering
    \includegraphics[width=0.7\textwidth]{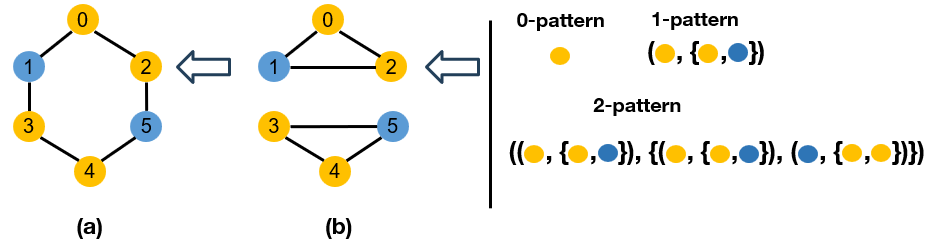}
    
  \caption{A simple counter-example of d-pattern. The d-pattern for each pair of nodes with the same id (number) in figures (a) and (b) is identical (proved in \Cref{claim:counter_example}). The figure follows the same convention as Figure 2 of \citet{yehudai2021local}, where the left side represents the two non-isomorphic graphs with 7 nodes and each color represents a different feature. The right side illustrates the 0, 1, and 2-pattern of the node with id 2.}
  
  \label{fig:d_pattern}
  \vspace{-0.3cm}
\end{figure*}

\begin{figure*}[t]
    \centering
    \includegraphics[width=0.35\textwidth]{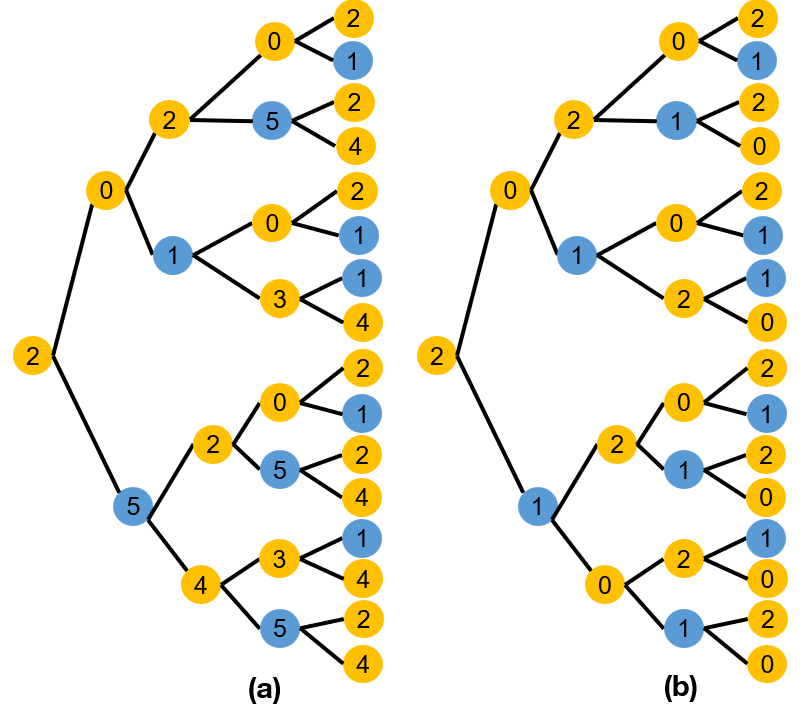}
    
  \caption{Illustration of recursive updates of node 2 in (a) and (b) from \Cref{fig:d_pattern}. In each hierarchy, the d-pattern is the same.}
  
  \label{fig:node_2_proof}
  \vspace{-0.3cm}
\end{figure*}

In this section, we show that the challenges identified in our findings cannot be addressed by the previous work \cite{yehudai2021local}.
\citet{yehudai2021local} introduce the concept of ``d-patterns" to characterize the information learned by GNNs for each node. It demonstrates that the presence of unseen d-patterns in large test graphs
results in significant test errors for GNNs. We first begin with their definition: 

\begin{definition}[d-pattern~\cite{yehudai2021local}]
    Let $C$ be a finite set of node
features, and let $G = (V, E)$ be a graph with node feature $c_v \in C$ for every node $v \in V$. We define the \textbf{d-pattern} of a node $v \in V$ for $d \geq 0$ recursively: For $d = 0$, the 0-pattern is $c_v$. For $d > 0$, the $d$-pattern of $v$ is $p=(p_v, \{(p_{i_1}, m_{i_1}), \ldots, (p_{i_l}, m_{i_l})\})$ iff node $v$ has $(d-1)$-pattern $p_v$ and for every $j \in \{1, \ldots, l\}$ the number of neighbors of $v$ with $(d-1)$-patterns $p_{i_j}$ is exactly $m_{p_{i_j}}$. Here, $i$ is the number of distinct neighboring $d-1$ patterns
of $v$.
\end{definition}

The notion of d-patterns highly resembles the idea of the 1-Weisfeiler-Lehman algorithm and suffers from the same limitations in differentiating cycles. In the following lemma, we show that the d-patterns for every node in the featureless cycle graphs of varying sizes are identical. 

\begin{lemma}
\label{lemma:d_pattern_cycle_graphs}
Consider a set of distinct cycle graphs $\{\mathcal{G}_n = (V_n, E_n)\}_{n=1}^k$, where each graph $\mathcal{G}_n$ in the set has a unique number of vertices $n$ with $n \geq 3$ and has no node features, then the d-pattern of each node is identical for all graphs. i.e. $\text{d-pattern}(v_i) = \text{d-pattern} (v_j) \: \forall \: v_i, v_j \in V_n, \forall \: n$.
\end{lemma}

\begin{proof}
    We can prove this statement through induction on $d$. The 1-pattern corresponds to the node's degree. In any cycle graph, the degree of any node is the same, which equals 2.
    For induction, we assume the statement holds for $k$-pattern, \emph{i.e.} the $k$-pattern of every node is $\mathcal{K}$. Then we aim to prove the statement for the $(k+1)$-pattern. Note that every node has two neighboring nodes. By the recursive update rule of $d$-pattern, the $(k+1)$-pattern of any node is $(\mathcal{K}, \{(\mathcal{K}, \mathcal{K})\})$ and thus the $(k+1)$-pattern of every node is identical.
\end{proof}

The above lemma points out that d-patterns cannot distinguish nodes in cycles of varying sizes, suggesting that our study of cycles is orthogonal to the study in \cite{yehudai2021local}. 

Furthermore, suffering from the same limitation as 1-WL, d-patterns would produce identical embeddings for some pairs of non-isomorphic graphs containing cycles with node features, as can be seen in \Cref{fig:d_pattern} and proved in \Cref{claim:counter_example}. Moreover, d-pattern can fail to distinguish graphs that 1-WL cannot distinguish, with abundant existing counter-examples in existing literature. This essentially implies that d-patterns that are seen during training time, may comprise different graphs during test times and thus contradict previous conclusions. Those corner cases usually involve cycles. Thus, cycle information should be treated independently in the problem of GNN size generalization.

\begin{claim}
    \label{claim:counter_example}
    The d-pattern of each node with the same id in \Cref{fig:d_pattern} is identical.
\end{claim}

\begin{proof}
    We can prove this statement through induction on $d$. The statement can be easily verified for $d=0$ or $1$. Suppose the statement holds for $d\leq k$ and we want to show that it holds for the $(k+1)$-pattern. Note that the statement naturally holds for node $0, 4$ since their $(k+1)$-pattern is dependent on the $k$-pattern of the same neighbor nodes. WLOG it suffices to consider the $k+1$-pattern node 2 for the rest of the cases, where we can enumerate all the recursive update cases from \Cref{fig:node_2_proof}. It is clear to observe that in each hierarchy the dependent neighbor patterns are the same and thus the d-pattern is identical. Hence we prove this claim.  
\end{proof}

\begin{figure*}[t]
  \centering
  \begin{subfigure}{0.82\textwidth}
    \centering
    \includegraphics[width=\textwidth]{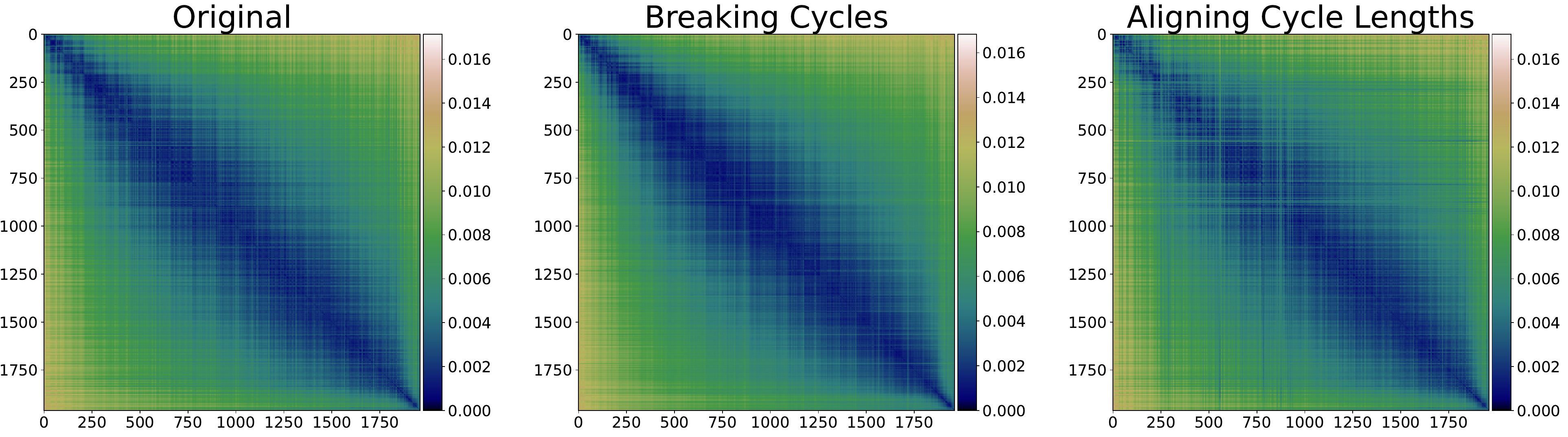}
    \caption*{BBBP}
  \end{subfigure}\\[5pt]
  
  \begin{subfigure}{0.82\textwidth}
    \centering
    \includegraphics[width=\textwidth]{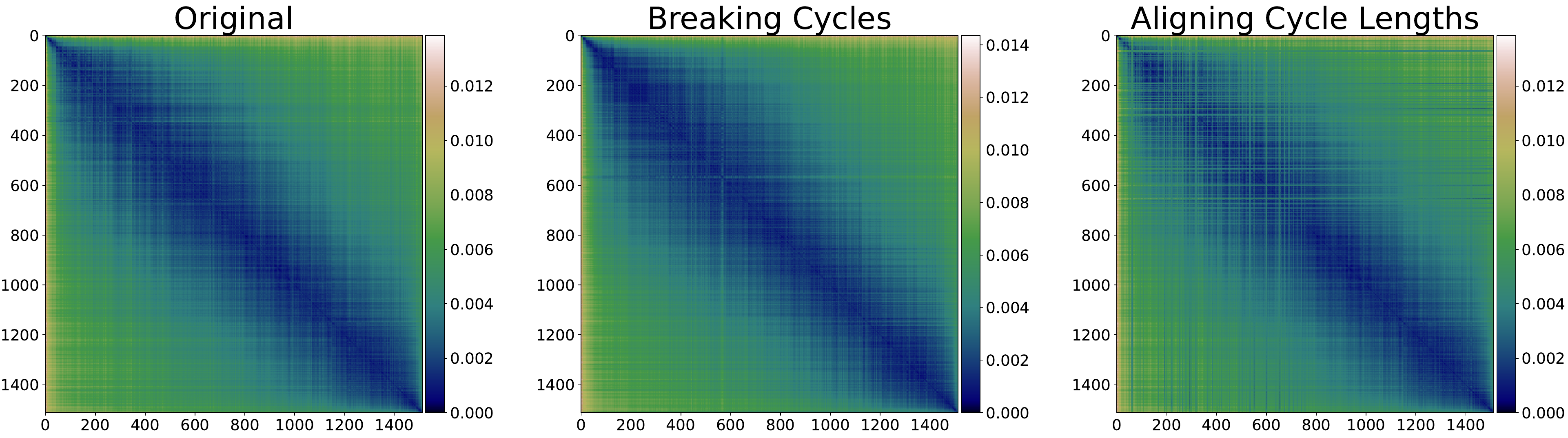}
    \caption*{BACE}
    \end{subfigure}\\[5pt]

    \begin{subfigure}{0.82\textwidth}
    \centering
    \includegraphics[width=\textwidth]{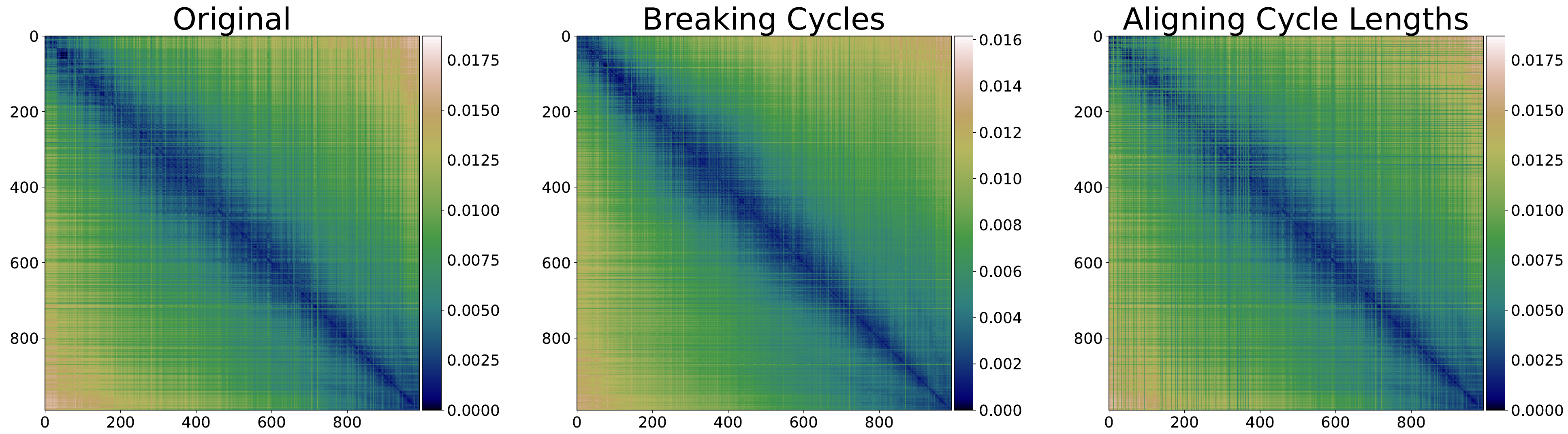}
    \caption*{PROTEINS}
    \end{subfigure}\\[5pt]
\begin{subfigure}{0.82\textwidth}
    \centering
    \includegraphics[width=\textwidth]{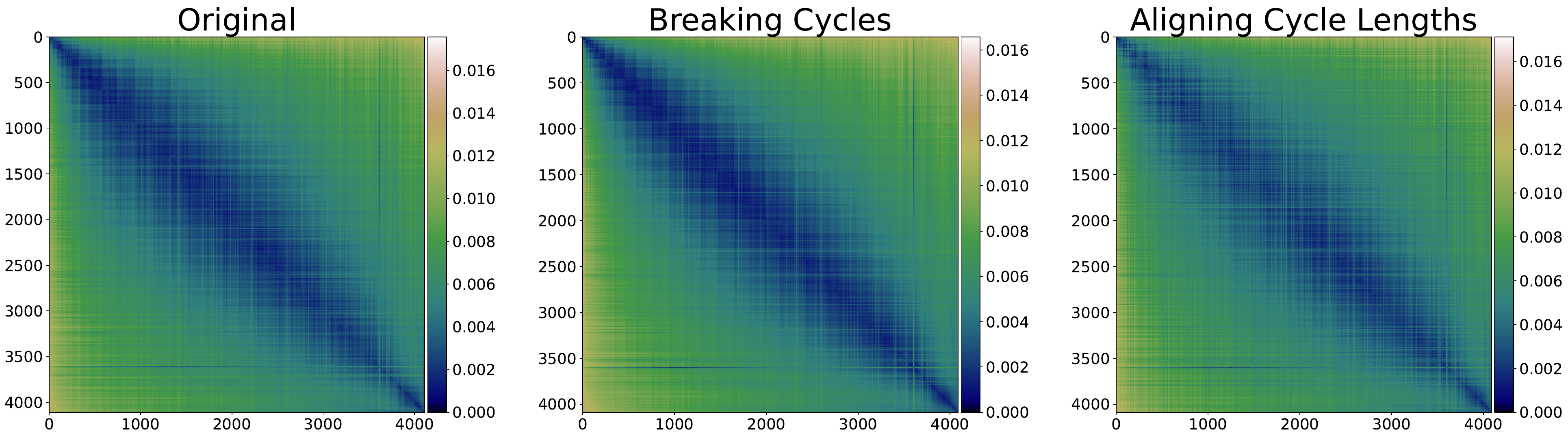}
    \caption*{NCI1}
    \end{subfigure}\\[5pt]
    \vspace{-0.6cm}
    \caption{Pairwise graph distance measured by the Wasserstein distance of eigenvalue distributions after breaking cycles and aligning cycle lengths on different datasets. Breaking cycles \textbf{amplifies the correlation} between eigenvalue distribution and graph size, while aligning cycle lengths  \textbf{reduces the correlation}.}
  \label{fig:additional_break_align}
  \vspace{-0.3cm}
\end{figure*}

%% file: main.bbl

\begin{thebibliography}{53}


\ifx \showCODEN    \undefined \def \showCODEN     #1{\unskip}     \fi
\ifx \showISBNx    \undefined \def \showISBNx     #1{\unskip}     \fi
\ifx \showISBNxiii \undefined \def \showISBNxiii  #1{\unskip}     \fi
\ifx \showISSN     \undefined \def \showISSN      #1{\unskip}     \fi
\ifx \showLCCN     \undefined \def \showLCCN      #1{\unskip}     \fi
\ifx \shownote     \undefined \def \shownote      #1{#1}          \fi
\ifx \showarticletitle \undefined \def \showarticletitle #1{#1}   \fi
\ifx \showURL      \undefined \def \showURL       {\relax}        \fi
\providecommand\bibfield[2]{#2}
\providecommand\bibinfo[2]{#2}
\providecommand\natexlab[1]{#1}
\providecommand\showeprint[2][]{arXiv:#2}

\bibitem[Balcilar et~al\mbox{.}(2021a)]%
        {balcilar2021analyzing}
\bibfield{author}{\bibinfo{person}{Muhammet Balcilar}, \bibinfo{person}{Renton Guillaume}, \bibinfo{person}{Pierre H{\'e}roux}, \bibinfo{person}{Benoit Ga{\"u}z{\`e}re}, \bibinfo{person}{S{\'e}bastien Adam}, {and} \bibinfo{person}{Paul Honeine}.} \bibinfo{year}{2021}\natexlab{a}.
\newblock \showarticletitle{Analyzing the expressive power of graph neural networks in a spectral perspective}. In \bibinfo{booktitle}{\emph{Proceedings of the International Conference on Learning Representations (ICLR)}}.
\newblock


\bibitem[Balcilar et~al\mbox{.}(2021b)]%
        {balcilar2021breaking}
\bibfield{author}{\bibinfo{person}{Muhammet Balcilar}, \bibinfo{person}{Pierre H{\'e}roux}, \bibinfo{person}{Benoit Gauzere}, \bibinfo{person}{Pascal Vasseur}, \bibinfo{person}{S{\'e}bastien Adam}, {and} \bibinfo{person}{Paul Honeine}.} \bibinfo{year}{2021}\natexlab{b}.
\newblock \showarticletitle{Breaking the limits of message passing graph neural networks}. In \bibinfo{booktitle}{\emph{International Conference on Machine Learning}}. PMLR, \bibinfo{pages}{599--608}.
\newblock


\bibitem[Berlingerio et~al\mbox{.}(2012)]%
        {berlingerio2012scalable}
\bibfield{author}{\bibinfo{person}{Michele Berlingerio}, \bibinfo{person}{Danai Koutra}, \bibinfo{person}{Tina Eliassi-Rad}, {and} \bibinfo{person}{Christos Faloutsos}.} \bibinfo{year}{2012}\natexlab{}.
\newblock \showarticletitle{A scalable approach to size-independent network similarity}.
\newblock \bibinfo{journal}{\emph{ArXiv Preprint ArXiv: 12092684}} (\bibinfo{year}{2012}).
\newblock


\bibitem[Bevilacqua et~al\mbox{.}(2021)]%
        {bevilacqua2021size}
\bibfield{author}{\bibinfo{person}{Beatrice Bevilacqua}, \bibinfo{person}{Yangze Zhou}, {and} \bibinfo{person}{Bruno Ribeiro}.} \bibinfo{year}{2021}\natexlab{}.
\newblock \showarticletitle{Size-invariant graph representations for graph classification extrapolations}. In \bibinfo{booktitle}{\emph{International Conference on Machine Learning}}. PMLR, \bibinfo{pages}{837--851}.
\newblock


\bibitem[Bo et~al\mbox{.}(2021)]%
        {bo2021beyond}
\bibfield{author}{\bibinfo{person}{Deyu Bo}, \bibinfo{person}{Xiao Wang}, \bibinfo{person}{Chuan Shi}, {and} \bibinfo{person}{Huawei Shen}.} \bibinfo{year}{2021}\natexlab{}.
\newblock \showarticletitle{Beyond Low-frequency Information in Graph Convolutional Networks}. In \bibinfo{booktitle}{\emph{{AAAI}}}. \bibinfo{publisher}{{AAAI} Press}.
\newblock


\bibitem[Buffelli et~al\mbox{.}(2022)]%
        {buffellisizeshiftreg}
\bibfield{author}{\bibinfo{person}{Davide Buffelli}, \bibinfo{person}{Pietro Lio}, {and} \bibinfo{person}{Fabio Vandin}.} \bibinfo{year}{2022}\natexlab{}.
\newblock \showarticletitle{SizeShiftReg: a Regularization Method for Improving Size-Generalization in Graph Neural Networks}. In \bibinfo{booktitle}{\emph{Advances in Neural Information Processing Systems}}.
\newblock


\bibitem[Chen et~al\mbox{.}(2022)]%
        {chen2022learning}
\bibfield{author}{\bibinfo{person}{Yongqiang Chen}, \bibinfo{person}{Yonggang Zhang}, \bibinfo{person}{Yatao Bian}, \bibinfo{person}{Han Yang}, \bibinfo{person}{MA Kaili}, \bibinfo{person}{Binghui Xie}, \bibinfo{person}{Tongliang Liu}, \bibinfo{person}{Bo Han}, {and} \bibinfo{person}{James Cheng}.} \bibinfo{year}{2022}\natexlab{}.
\newblock \showarticletitle{Learning causally invariant representations for out-of-distribution generalization on graphs}.
\newblock \bibinfo{journal}{\emph{Advances in Neural Information Processing Systems}}  \bibinfo{volume}{35} (\bibinfo{year}{2022}), \bibinfo{pages}{22131--22148}.
\newblock


\bibitem[Chen et~al\mbox{.}(2023)]%
        {balcilar2020bridging}
\bibfield{author}{\bibinfo{person}{Zhiqian Chen}, \bibinfo{person}{Fanglan Chen}, \bibinfo{person}{Lei Zhang}, \bibinfo{person}{Taoran Ji}, \bibinfo{person}{Kaiqun Fu}, \bibinfo{person}{Liang Zhao}, \bibinfo{person}{Feng Chen}, \bibinfo{person}{Lingfei Wu}, \bibinfo{person}{Charu Aggarwal}, {and} \bibinfo{person}{Chang-Tien Lu}.} \bibinfo{year}{2023}\natexlab{}.
\newblock \showarticletitle{Bridging the Gap between Spatial and Spectral Domains: A Unified Framework for Graph Neural Networks}.
\newblock \bibinfo{journal}{\emph{ACM Comput. Surv.}} \bibinfo{volume}{56}, \bibinfo{number}{5}, Article \bibinfo{articleno}{126} (\bibinfo{date}{Dec.} \bibinfo{year}{2023}), \bibinfo{numpages}{42}~pages.
\newblock
\showISSN{0360-0300}
\href{https://doi.org/10.1145/3627816}{doi:\nolinkurl{10.1145/3627816}}


\bibitem[Chen et~al\mbox{.}(2020)]%
        {chen2020can}
\bibfield{author}{\bibinfo{person}{Zhengdao Chen}, \bibinfo{person}{Lei Chen}, \bibinfo{person}{Soledad Villar}, {and} \bibinfo{person}{Joan Bruna}.} \bibinfo{year}{2020}\natexlab{}.
\newblock \showarticletitle{Can graph neural networks count substructures?}
\newblock \bibinfo{journal}{\emph{Advances in neural information processing systems}}  \bibinfo{volume}{33} (\bibinfo{year}{2020}), \bibinfo{pages}{10383--10395}.
\newblock


\bibitem[Chu et~al\mbox{.}(2023)]%
        {chu2023wasserstein}
\bibfield{author}{\bibinfo{person}{Xu Chu}, \bibinfo{person}{Yujie Jin}, \bibinfo{person}{Xin Wang}, \bibinfo{person}{Shanghang Zhang}, \bibinfo{person}{Yasha Wang}, \bibinfo{person}{Wenwu Zhu}, {and} \bibinfo{person}{Hong Mei}.} \bibinfo{year}{2023}\natexlab{}.
\newblock \showarticletitle{Wasserstein Barycenter matching for graph size generalization of message passing neural networks}. In \bibinfo{booktitle}{\emph{International Conference on Machine Learning}}. PMLR, \bibinfo{pages}{6158--6184}.
\newblock


\bibitem[Defferrard et~al\mbox{.}(2016)]%
        {defferrard2016convolutional}
\bibfield{author}{\bibinfo{person}{Micha{\"e}l Defferrard}, \bibinfo{person}{Xavier Bresson}, {and} \bibinfo{person}{Pierre Vandergheynst}.} \bibinfo{year}{2016}\natexlab{}.
\newblock \showarticletitle{Convolutional neural networks on graphs with fast localized spectral filtering}. In \bibinfo{booktitle}{\emph{NeurIPS}}.
\newblock


\bibitem[Dodziuk(2006)]%
        {dodziuk2006cyclodextrins}
\bibfield{author}{\bibinfo{person}{Helena Dodziuk}.} \bibinfo{year}{2006}\natexlab{}.
\newblock \bibinfo{booktitle}{\emph{Cyclodextrins and their complexes: chemistry, analytical methods, applications}}.
\newblock \bibinfo{publisher}{John Wiley \& Sons}.
\newblock


\bibitem[Gilmer et~al\mbox{.}(2017)]%
        {gilmer2017neural}
\bibfield{author}{\bibinfo{person}{Justin Gilmer}, \bibinfo{person}{Samuel~S Schoenholz}, \bibinfo{person}{Patrick~F Riley}, \bibinfo{person}{Oriol Vinyals}, {and} \bibinfo{person}{George~E Dahl}.} \bibinfo{year}{2017}\natexlab{}.
\newblock \showarticletitle{Neural message passing for quantum chemistry}. In \bibinfo{booktitle}{\emph{International conference on machine learning}}. PMLR, \bibinfo{pages}{1263--1272}.
\newblock


\bibitem[Gokel(2016)]%
        {gokel2016crown}
\bibfield{author}{\bibinfo{person}{George~W Gokel}.} \bibinfo{year}{2016}\natexlab{}.
\newblock \bibinfo{booktitle}{\emph{Crown ethers and cryptands}}.
\newblock \bibinfo{publisher}{Royal Society of Chemistry}.
\newblock


\bibitem[Hamilton et~al\mbox{.}(2017)]%
        {graphsage}
\bibfield{author}{\bibinfo{person}{Will Hamilton}, \bibinfo{person}{Zhitao Ying}, {and} \bibinfo{person}{Jure Leskovec}.} \bibinfo{year}{2017}\natexlab{}.
\newblock \showarticletitle{Inductive representation learning on large graphs}. In \bibinfo{booktitle}{\emph{NeurIPS}}.
\newblock


\bibitem[Hu et~al\mbox{.}(2020)]%
        {hu2020open}
\bibfield{author}{\bibinfo{person}{Weihua Hu}, \bibinfo{person}{Matthias Fey}, \bibinfo{person}{Marinka Zitnik}, \bibinfo{person}{Yuxiao Dong}, \bibinfo{person}{Hongyu Ren}, \bibinfo{person}{Bowen Liu}, \bibinfo{person}{Michele Catasta}, {and} \bibinfo{person}{Jure Leskovec}.} \bibinfo{year}{2020}\natexlab{}.
\newblock \showarticletitle{Open graph benchmark: Datasets for machine learning on graphs}.
\newblock \bibinfo{journal}{\emph{Advances in neural information processing systems}}  \bibinfo{volume}{33} (\bibinfo{year}{2020}), \bibinfo{pages}{22118--22133}.
\newblock


\bibitem[Huang et~al\mbox{.}(2024)]%
        {huang2024enhancing}
\bibfield{author}{\bibinfo{person}{Zheng Huang}, \bibinfo{person}{Qihui Yang}, \bibinfo{person}{Dawei Zhou}, {and} \bibinfo{person}{Yujun Yan}.} \bibinfo{year}{2024}\natexlab{}.
\newblock \showarticletitle{Enhancing size generalization in graph neural networks through disentangled representation learning}. In \bibinfo{booktitle}{\emph{Proceedings of the 41st International Conference on Machine Learning}} (Vienna, Austria) \emph{(\bibinfo{series}{ICML'24})}. \bibinfo{publisher}{JMLR.org}, Article \bibinfo{articleno}{818}, \bibinfo{numpages}{17}~pages.
\newblock


\bibitem[Ji et~al\mbox{.}(2023)]%
        {ji2022drugood}
\bibfield{author}{\bibinfo{person}{Yuanfeng Ji}, \bibinfo{person}{Lu Zhang}, \bibinfo{person}{Jiaxiang Wu}, \bibinfo{person}{Bingzhe Wu}, \bibinfo{person}{Lanqing Li}, \bibinfo{person}{Long-Kai Huang}, \bibinfo{person}{Tingyang Xu}, \bibinfo{person}{Yu Rong}, \bibinfo{person}{Jie Ren}, \bibinfo{person}{Ding Xue}, \bibinfo{person}{Houtim Lai}, \bibinfo{person}{Wei Liu}, \bibinfo{person}{Junzhou Huang}, \bibinfo{person}{Shuigeng Zhou}, \bibinfo{person}{Ping Luo}, \bibinfo{person}{Peilin Zhao}, {and} \bibinfo{person}{Yatao Bian}.} \bibinfo{year}{2023}\natexlab{}.
\newblock \showarticletitle{DrugOOD: out-of-distribution dataset curator and benchmark for AI-aided drug discovery - a focus on affinity prediction problems with noise annotations}. In \bibinfo{booktitle}{\emph{Proceedings of the Thirty-Seventh AAAI Conference on Artificial Intelligence and Thirty-Fifth Conference on Innovative Applications of Artificial Intelligence and Thirteenth Symposium on Educational Advances in Artificial Intelligence}} \emph{(\bibinfo{series}{AAAI'23/IAAI'23/EAAI'23})}. \bibinfo{publisher}{AAAI Press}, Article \bibinfo{articleno}{901}, \bibinfo{numpages}{9}~pages.
\newblock
\showISBNx{978-1-57735-880-0}
\href{https://doi.org/10.1609/aaai.v37i7.25970}{doi:\nolinkurl{10.1609/aaai.v37i7.25970}}


\bibitem[Kipf and Welling(2017)]%
        {kipf2016semi}
\bibfield{author}{\bibinfo{person}{Thomas~N. Kipf} {and} \bibinfo{person}{Max Welling}.} \bibinfo{year}{2017}\natexlab{}.
\newblock \showarticletitle{Semi-Supervised Classification with Graph Convolutional Networks}. In \bibinfo{booktitle}{\emph{International Conference on Learning Representations (ICLR)}}.
\newblock


\bibitem[Knyazev et~al\mbox{.}(2019)]%
        {knyazev2019understanding}
\bibfield{author}{\bibinfo{person}{Boris Knyazev}, \bibinfo{person}{Graham~W Taylor}, {and} \bibinfo{person}{Mohamed Amer}.} \bibinfo{year}{2019}\natexlab{}.
\newblock \showarticletitle{Understanding attention and generalization in graph neural networks}.
\newblock \bibinfo{journal}{\emph{Advances in neural information processing systems}}  \bibinfo{volume}{32} (\bibinfo{year}{2019}).
\newblock


\bibitem[Lee et~al\mbox{.}(2018)]%
        {lee2018graph}
\bibfield{author}{\bibinfo{person}{John~Boaz Lee}, \bibinfo{person}{Ryan Rossi}, {and} \bibinfo{person}{Xiangnan Kong}.} \bibinfo{year}{2018}\natexlab{}.
\newblock \showarticletitle{Graph classification using structural attention}. In \bibinfo{booktitle}{\emph{Proceedings of the 24th ACM SIGKDD International Conference on Knowledge Discovery \& Data Mining}}. \bibinfo{pages}{1666--1674}.
\newblock


\bibitem[Levie et~al\mbox{.}(2021)]%
        {levie2021transferability}
\bibfield{author}{\bibinfo{person}{Ron Levie}, \bibinfo{person}{Wei Huang}, \bibinfo{person}{Lorenzo Bucci}, \bibinfo{person}{Michael Bronstein}, {and} \bibinfo{person}{Gitta Kutyniok}.} \bibinfo{year}{2021}\natexlab{}.
\newblock \showarticletitle{Transferability of spectral graph convolutional neural networks}.
\newblock \bibinfo{journal}{\emph{The Journal of Machine Learning Research}} \bibinfo{volume}{22}, \bibinfo{number}{1} (\bibinfo{year}{2021}), \bibinfo{pages}{12462--12520}.
\newblock


\bibitem[Levie et~al\mbox{.}(2018)]%
        {levie2018cayleynets}
\bibfield{author}{\bibinfo{person}{Ron Levie}, \bibinfo{person}{Federico Monti}, \bibinfo{person}{Xavier Bresson}, {and} \bibinfo{person}{Michael~M Bronstein}.} \bibinfo{year}{2018}\natexlab{}.
\newblock \showarticletitle{Cayleynets: Graph convolutional neural networks with complex rational spectral filters}.
\newblock \bibinfo{journal}{\emph{IEEE Transactions on Signal Processing}} \bibinfo{volume}{67}, \bibinfo{number}{1} (\bibinfo{year}{2018}), \bibinfo{pages}{97--109}.
\newblock


\bibitem[Li et~al\mbox{.}({[n.\,d.]})]%
        {LiDZKY23_sparsification}
\bibfield{author}{\bibinfo{person}{Gaotang Li}, \bibinfo{person}{Marlena Duda}, \bibinfo{person}{Xiang Zhang}, \bibinfo{person}{Danai Koutra}, {and} \bibinfo{person}{Yujun Yan}.} \bibinfo{year}{[n.\,d.]}\natexlab{}.
\newblock \showarticletitle{Interpretable Sparsification of Brain Graphs: Better Practices and Effective Designs for Graph Neural Networks}. In \bibinfo{booktitle}{\emph{Proceedings of the 29th {ACM} {SIGKDD} Conference on Knowledge Discovery and Data Mining, pages = {1223--1234}, publisher = {{ACM}}, year = {2023},}}.
\newblock


\bibitem[Liu et~al\mbox{.}(2024)]%
        {liuexploring}
\bibfield{author}{\bibinfo{person}{Xuyuan Liu}, \bibinfo{person}{Yinghao Cai}, \bibinfo{person}{Qihui Yang}, {and} \bibinfo{person}{Yujun Yan}.} \bibinfo{year}{2024}\natexlab{}.
\newblock \showarticletitle{Exploring Consistency in Graph Representations: from Graph Kernels to Graph Neural Networks}.
\newblock \bibinfo{journal}{\emph{The Thirty-eighth Annual Conference on Neural Information Processing Systems}}  \bibinfo{volume}{37} (\bibinfo{year}{2024}).
\newblock


\bibitem[Morris et~al\mbox{.}(2020)]%
        {Morris+2020}
\bibfield{author}{\bibinfo{person}{Christopher Morris}, \bibinfo{person}{Nils~M. Kriege}, \bibinfo{person}{Franka Bause}, \bibinfo{person}{Kristian Kersting}, \bibinfo{person}{Petra Mutzel}, {and} \bibinfo{person}{Marion Neumann}.} \bibinfo{year}{2020}\natexlab{}.
\newblock \showarticletitle{TUDataset: A collection of benchmark datasets for learning with graphs}. In \bibinfo{booktitle}{\emph{ICML 2020 Workshop on Graph Representation Learning and Beyond (GRL+ 2020)}}.
\newblock
\showeprint[arxiv]{2007.08663}
\urldef\tempurl%
\url{www.graphlearning.io}
\showURL{%
\tempurl}


\bibitem[Muhammet et~al\mbox{.}(2020)]%
        {muhammet2020spectral}
\bibfield{author}{\bibinfo{person}{Balcilar Muhammet}, \bibinfo{person}{Renton Guillaume}, \bibinfo{person}{H{\'e}roux Pierre}, \bibinfo{person}{Ga{\"u}z{\`e}re Benoit}, \bibinfo{person}{Adam S{\'e}bastien}, {and} \bibinfo{person}{Paul Honeine}.} \bibinfo{year}{2020}\natexlab{}.
\newblock \showarticletitle{When spectral domain meets spatial domain in graph neural networks}. In \bibinfo{booktitle}{\emph{Thirty-seventh International Conference on Machine Learning (ICML 2020)-Workshop on Graph Representation Learning and Beyond (GRL+ 2020)}}.
\newblock


\bibitem[Murphy et~al\mbox{.}(2019)]%
        {murphy2019relational}
\bibfield{author}{\bibinfo{person}{Ryan Murphy}, \bibinfo{person}{Balasubramaniam Srinivasan}, \bibinfo{person}{Vinayak Rao}, {and} \bibinfo{person}{Bruno Ribeiro}.} \bibinfo{year}{2019}\natexlab{}.
\newblock \showarticletitle{Relational pooling for graph representations}. In \bibinfo{booktitle}{\emph{International Conference on Machine Learning}}. PMLR, \bibinfo{pages}{4663--4673}.
\newblock


\bibitem[Ning et~al\mbox{.}(2024)]%
        {ning2024information}
\bibfield{author}{\bibinfo{person}{Xuying Ning}, \bibinfo{person}{Wujiang Xu}, \bibinfo{person}{Xiaolei Liu}, \bibinfo{person}{Mingming Ha}, \bibinfo{person}{Qiongxu Ma}, \bibinfo{person}{Youru Li}, \bibinfo{person}{Linxun Chen}, {and} \bibinfo{person}{Yongfeng Zhang}.} \bibinfo{year}{2024}\natexlab{}.
\newblock \showarticletitle{Information maximization via variational autoencoders for cross-domain recommendation}.
\newblock \bibinfo{journal}{\emph{arXiv preprint arXiv:2405.20710}} (\bibinfo{year}{2024}).
\newblock


\bibitem[Paton(1969)]%
        {paton1969algorithm}
\bibfield{author}{\bibinfo{person}{Keith Paton}.} \bibinfo{year}{1969}\natexlab{}.
\newblock \showarticletitle{An algorithm for finding a fundamental set of cycles of a graph}.
\newblock \bibinfo{journal}{\emph{Commun. ACM}} \bibinfo{volume}{12}, \bibinfo{number}{9} (\bibinfo{year}{1969}), \bibinfo{pages}{514--518}.
\newblock


\bibitem[Pliska et~al\mbox{.}(1996)]%
        {pliska1996lipophilicity}
\bibfield{author}{\bibinfo{person}{Vladimir Pliska}, \bibinfo{person}{Bernard Testa}, \bibinfo{person}{Han van~de Waterbeemd}, \bibinfo{person}{R Mannhold}, \bibinfo{person}{H Kubinyi}, {and} \bibinfo{person}{H Timmerman}.} \bibinfo{year}{1996}\natexlab{}.
\newblock \bibinfo{booktitle}{\emph{Lipophilicity in drug action and toxicology}}.
\newblock \bibinfo{publisher}{VCH Weinheim}.
\newblock


\bibitem[Sanchez-Gonzalez et~al\mbox{.}(2020)]%
        {sanchez2020learning}
\bibfield{author}{\bibinfo{person}{Alvaro Sanchez-Gonzalez}, \bibinfo{person}{Jonathan Godwin}, \bibinfo{person}{Tobias Pfaff}, \bibinfo{person}{Rex Ying}, \bibinfo{person}{Jure Leskovec}, {and} \bibinfo{person}{Peter Battaglia}.} \bibinfo{year}{2020}\natexlab{}.
\newblock \showarticletitle{Learning to simulate complex physics with graph networks}. In \bibinfo{booktitle}{\emph{International conference on machine learning}}. PMLR, \bibinfo{pages}{8459--8468}.
\newblock


\bibitem[Singh and Singh(2022)]%
        {singh2022feature}
\bibfield{author}{\bibinfo{person}{Dalwinder Singh} {and} \bibinfo{person}{Birmohan Singh}.} \bibinfo{year}{2022}\natexlab{}.
\newblock \showarticletitle{Feature wise normalization: An effective way of normalizing data}.
\newblock \bibinfo{journal}{\emph{Pattern Recognition}}  \bibinfo{volume}{122} (\bibinfo{year}{2022}), \bibinfo{pages}{108307}.
\newblock


\bibitem[Sys{\l}o(1979)]%
        {syslo1979cycle}
\bibfield{author}{\bibinfo{person}{Maciej~Marek Sys{\l}o}.} \bibinfo{year}{1979}\natexlab{}.
\newblock \showarticletitle{On cycle bases of a graph}.
\newblock \bibinfo{journal}{\emph{Networks}} \bibinfo{volume}{9}, \bibinfo{number}{2} (\bibinfo{year}{1979}), \bibinfo{pages}{123--132}.
\newblock


\bibitem[Trivedi et~al\mbox{.}(2024)]%
        {trivedi2024accurate}
\bibfield{author}{\bibinfo{person}{Puja Trivedi}, \bibinfo{person}{Mark Heimann}, \bibinfo{person}{Rushil Anirudh}, \bibinfo{person}{Danai Koutra}, {and} \bibinfo{person}{Jayaraman~J. Thiagarajan}.} \bibinfo{year}{2024}\natexlab{}.
\newblock \showarticletitle{Accurate and Scalable Estimation of Epistemic Uncertainty for Graph Neural Networks}. In \bibinfo{booktitle}{\emph{The Twelfth International Conference on Learning Representations}}.
\newblock


\bibitem[Veli{\v{c}}kovi{\'{c}} et~al\mbox{.}(2018)]%
        {velickovic2018graph}
\bibfield{author}{\bibinfo{person}{Petar Veli{\v{c}}kovi{\'{c}}}, \bibinfo{person}{Guillem Cucurull}, \bibinfo{person}{Arantxa Casanova}, \bibinfo{person}{Adriana Romero}, \bibinfo{person}{Pietro Li{\`{o}}}, {and} \bibinfo{person}{Yoshua Bengio}.} \bibinfo{year}{2018}\natexlab{}.
\newblock \showarticletitle{{Graph Attention Networks}}.
\newblock \bibinfo{journal}{\emph{International Conference on Learning Representations (ICLR)}} (\bibinfo{year}{2018}).
\newblock
\urldef\tempurl%
\url{https://openreview.net/forum?id=rJXMpikCZ}
\showURL{%
\tempurl}


\bibitem[Veli{\v{c}}kovi{\'c} et~al\mbox{.}(2020)]%
        {velivckovicneural}
\bibfield{author}{\bibinfo{person}{Petar Veli{\v{c}}kovi{\'c}}, \bibinfo{person}{Rex Ying}, \bibinfo{person}{Matilde Padovano}, \bibinfo{person}{Raia Hadsell}, {and} \bibinfo{person}{Charles Blundell}.} \bibinfo{year}{2020}\natexlab{}.
\newblock \showarticletitle{Neural Execution of Graph Algorithms}. In \bibinfo{booktitle}{\emph{International Conference on Learning Representations}}.
\newblock


\bibitem[Vignac et~al\mbox{.}(2020)]%
        {vignac2020building}
\bibfield{author}{\bibinfo{person}{Clement Vignac}, \bibinfo{person}{Andreas Loukas}, {and} \bibinfo{person}{Pascal Frossard}.} \bibinfo{year}{2020}\natexlab{}.
\newblock \showarticletitle{Building powerful and equivariant graph neural networks with structural message-passing}.
\newblock \bibinfo{journal}{\emph{Advances in neural information processing systems}}  \bibinfo{volume}{33} (\bibinfo{year}{2020}), \bibinfo{pages}{14143--14155}.
\newblock


\bibitem[Villani et~al\mbox{.}(2009)]%
        {villani2009optimal}
\bibfield{author}{\bibinfo{person}{C{\'e}dric Villani} {et~al\mbox{.}}} \bibinfo{year}{2009}\natexlab{}.
\newblock \bibinfo{booktitle}{\emph{Optimal transport: old and new}}. Vol.~\bibinfo{volume}{338}.
\newblock \bibinfo{publisher}{Springer}.
\newblock


\bibitem[Wale and Karypis(2006)]%
        {4053093}
\bibfield{author}{\bibinfo{person}{Nikil Wale} {and} \bibinfo{person}{George Karypis}.} \bibinfo{year}{2006}\natexlab{}.
\newblock \showarticletitle{Comparison of Descriptor Spaces for Chemical Compound Retrieval and Classification}. In \bibinfo{booktitle}{\emph{Sixth International Conference on Data Mining (ICDM'06)}}. \bibinfo{pages}{678--689}.
\newblock
\href{https://doi.org/10.1109/ICDM.2006.39}{doi:\nolinkurl{10.1109/ICDM.2006.39}}


\bibitem[Wu et~al\mbox{.}(2018)]%
        {wu2018moleculenet}
\bibfield{author}{\bibinfo{person}{Zhenqin Wu}, \bibinfo{person}{Bharath Ramsundar}, \bibinfo{person}{Evan~N Feinberg}, \bibinfo{person}{Joseph Gomes}, \bibinfo{person}{Caleb Geniesse}, \bibinfo{person}{Aneesh~S Pappu}, \bibinfo{person}{Karl Leswing}, {and} \bibinfo{person}{Vijay Pande}.} \bibinfo{year}{2018}\natexlab{}.
\newblock \showarticletitle{MoleculeNet: a benchmark for molecular machine learning}.
\newblock \bibinfo{journal}{\emph{Chemical science}} \bibinfo{volume}{9}, \bibinfo{number}{2} (\bibinfo{year}{2018}), \bibinfo{pages}{513--530}.
\newblock


\bibitem[Xu et~al\mbox{.}(2018)]%
        {xu2018powerful}
\bibfield{author}{\bibinfo{person}{Keyulu Xu}, \bibinfo{person}{Weihua Hu}, \bibinfo{person}{Jure Leskovec}, {and} \bibinfo{person}{Stefanie Jegelka}.} \bibinfo{year}{2018}\natexlab{}.
\newblock \showarticletitle{How Powerful are Graph Neural Networks?}
\newblock \bibinfo{journal}{\emph{International Conference on Learning Representations}} (\bibinfo{year}{2018}).
\newblock


\bibitem[Xu et~al\mbox{.}(2021)]%
        {xu2021neural}
\bibfield{author}{\bibinfo{person}{Keyulu Xu}, \bibinfo{person}{Mozhi Zhang}, \bibinfo{person}{Jingling Li}, \bibinfo{person}{Simon~S Du}, \bibinfo{person}{Ken-ichi Kawarabayashi}, {and} \bibinfo{person}{Stefanie Jegelka}.} \bibinfo{year}{2021}\natexlab{}.
\newblock \showarticletitle{How neural networks extrapolate: from feedforward to graph neural networks}. In \bibinfo{booktitle}{\emph{International Conference on Learning Representations (ICLR)}}.
\newblock


\bibitem[Xu et~al\mbox{.}(2024)]%
        {xu2024towards}
\bibfield{author}{\bibinfo{person}{Wujiang Xu}, \bibinfo{person}{Xuying Ning}, \bibinfo{person}{Wenfang Lin}, \bibinfo{person}{Mingming Ha}, \bibinfo{person}{Qiongxu Ma}, \bibinfo{person}{Qianqiao Liang}, \bibinfo{person}{Xuewen Tao}, \bibinfo{person}{Linxun Chen}, \bibinfo{person}{Bing Han}, {and} \bibinfo{person}{Minnan Luo}.} \bibinfo{year}{2024}\natexlab{}.
\newblock \showarticletitle{Towards open-world cross-domain sequential recommendation: A model-agnostic contrastive denoising approach}. In \bibinfo{booktitle}{\emph{Joint European Conference on Machine Learning and Knowledge Discovery in Databases}}. Springer, \bibinfo{pages}{161--179}.
\newblock


\bibitem[Xu et~al\mbox{.}(2025a)]%
        {xu2025iagent}
\bibfield{author}{\bibinfo{person}{Wujiang Xu}, \bibinfo{person}{Yunxiao Shi}, \bibinfo{person}{Zujie Liang}, \bibinfo{person}{Xuying Ning}, \bibinfo{person}{Kai Mei}, \bibinfo{person}{Kun Wang}, \bibinfo{person}{Xi Zhu}, \bibinfo{person}{Min Xu}, {and} \bibinfo{person}{Yongfeng Zhang}.} \bibinfo{year}{2025}\natexlab{a}.
\newblock \showarticletitle{iAgent: LLM Agent as a Shield between User and Recommender Systems}.
\newblock \bibinfo{journal}{\emph{arXiv preprint arXiv:2502.14662}} (\bibinfo{year}{2025}).
\newblock


\bibitem[Xu et~al\mbox{.}(2025b)]%
        {xu2024slmrec}
\bibfield{author}{\bibinfo{person}{Wujiang Xu}, \bibinfo{person}{Qitian Wu}, \bibinfo{person}{Zujie Liang}, \bibinfo{person}{Jiaojiao Han}, \bibinfo{person}{Xuying Ning}, \bibinfo{person}{Yunxiao Shi}, \bibinfo{person}{Wenfang Lin}, {and} \bibinfo{person}{Yongfeng Zhang}.} \bibinfo{year}{2025}\natexlab{b}.
\newblock \showarticletitle{{SLMR}ec: Distilling Large Language Models into Small for Sequential Recommendation}. In \bibinfo{booktitle}{\emph{The Thirteenth International Conference on Learning Representations}}.
\newblock
\urldef\tempurl%
\url{https://openreview.net/forum?id=G4wARwjF8M}
\showURL{%
\tempurl}


\bibitem[Yan et~al\mbox{.}(2020)]%
        {yan2020neural}
\bibfield{author}{\bibinfo{person}{Yujun Yan}, \bibinfo{person}{Kevin Swersky}, \bibinfo{person}{Danai Koutra}, \bibinfo{person}{Parthasarathy Ranganathan}, {and} \bibinfo{person}{Milad Hashemi}.} \bibinfo{year}{2020}\natexlab{}.
\newblock \showarticletitle{Neural execution engines: Learning to execute subroutines}.
\newblock \bibinfo{journal}{\emph{Advances in Neural Information Processing Systems}}  \bibinfo{volume}{33} (\bibinfo{year}{2020}), \bibinfo{pages}{17298--17308}.
\newblock


\bibitem[Yan et~al\mbox{.}(2019)]%
        {yan2019groupinn}
\bibfield{author}{\bibinfo{person}{Yujun Yan}, \bibinfo{person}{Jiong Zhu}, \bibinfo{person}{Marlena Duda}, \bibinfo{person}{Eric Solarz}, \bibinfo{person}{Chandra Sripada}, {and} \bibinfo{person}{Danai Koutra}.} \bibinfo{year}{2019}\natexlab{}.
\newblock \showarticletitle{Groupinn: Grouping-based interpretable neural network for classification of limited, noisy brain data}. In \bibinfo{booktitle}{\emph{proceedings of the 25th ACM SIGKDD international conference on knowledge discovery \& data mining}}. \bibinfo{pages}{772--782}.
\newblock


\bibitem[Yehudai et~al\mbox{.}(2021)]%
        {yehudai2021local}
\bibfield{author}{\bibinfo{person}{Gilad Yehudai}, \bibinfo{person}{Ethan Fetaya}, \bibinfo{person}{Eli Meirom}, \bibinfo{person}{Gal Chechik}, {and} \bibinfo{person}{Haggai Maron}.} \bibinfo{year}{2021}\natexlab{}.
\newblock \showarticletitle{From local structures to size generalization in graph neural networks}. In \bibinfo{booktitle}{\emph{International Conference on Machine Learning}}. PMLR, \bibinfo{pages}{11975--11986}.
\newblock


\bibitem[Ying et~al\mbox{.}(2019)]%
        {ying2019gnnexplainer}
\bibfield{author}{\bibinfo{person}{Zhitao Ying}, \bibinfo{person}{Dylan Bourgeois}, \bibinfo{person}{Jiaxuan You}, \bibinfo{person}{Marinka Zitnik}, {and} \bibinfo{person}{Jure Leskovec}.} \bibinfo{year}{2019}\natexlab{}.
\newblock \showarticletitle{Gnnexplainer: Generating explanations for graph neural networks}.
\newblock \bibinfo{journal}{\emph{Advances in neural information processing systems}}  \bibinfo{volume}{32} (\bibinfo{year}{2019}).
\newblock


\bibitem[Ying et~al\mbox{.}(2018)]%
        {ying2018hierarchical}
\bibfield{author}{\bibinfo{person}{Zhitao Ying}, \bibinfo{person}{Jiaxuan You}, \bibinfo{person}{Christopher Morris}, \bibinfo{person}{Xiang Ren}, \bibinfo{person}{Will Hamilton}, {and} \bibinfo{person}{Jure Leskovec}.} \bibinfo{year}{2018}\natexlab{}.
\newblock \showarticletitle{Hierarchical graph representation learning with differentiable pooling}. In \bibinfo{booktitle}{\emph{NeurIPS}}. \bibinfo{pages}{4800--4810}.
\newblock


\bibitem[Zhang et~al\mbox{.}(2018)]%
        {zhang2018end}
\bibfield{author}{\bibinfo{person}{Muhan Zhang}, \bibinfo{person}{Zhicheng Cui}, \bibinfo{person}{Marion Neumann}, {and} \bibinfo{person}{Yixin Chen}.} \bibinfo{year}{2018}\natexlab{}.
\newblock \showarticletitle{An end-to-end deep learning architecture for graph classification}. In \bibinfo{booktitle}{\emph{Proceedings of the AAAI conference on artificial intelligence}}, Vol.~\bibinfo{volume}{32}.
\newblock


\bibitem[Zhou et~al\mbox{.}(2022)]%
        {zhou2022ood}
\bibfield{author}{\bibinfo{person}{Yangze Zhou}, \bibinfo{person}{Gitta Kutyniok}, {and} \bibinfo{person}{Bruno Ribeiro}.} \bibinfo{year}{2022}\natexlab{}.
\newblock \showarticletitle{OOD link prediction generalization capabilities of message-passing GNNs in larger test graphs}.
\newblock \bibinfo{journal}{\emph{Advances in Neural Information Processing Systems}}  \bibinfo{volume}{35} (\bibinfo{year}{2022}), \bibinfo{pages}{20257--20272}.
\newblock


\end{thebibliography}
